\documentclass[11pt]{amsart}

\usepackage[margin=1.1in]{geometry}
\usepackage[T1]{fontenc}
\usepackage{lmodern}
\usepackage{amssymb,mathtools}
\usepackage{bm}
\usepackage{booktabs}
\usepackage{array,longtable}
\usepackage{caption}
\usepackage{enumitem}
\usepackage{microtype}
\usepackage{etoolbox}
\makeatletter
\newcommand{\keepheadingwithtext}{\if@nobreak\else
    \par
    \begingroup
    \@tempdima=8\baselineskip
    \vskip 0pt plus\@tempdima
    \penalty-100
    \vskip 0pt plus-\@tempdima
    \vskip\@tempdima
    \penalty9999
    \vskip-\@tempdima
    \vskip0pt
    \endgroup
  \fi}
\AtBeginDocument{\pretocmd{\@startsection}{\keepheadingwithtext}{}{\PackageError{pldrllm}{Section heading hook failed}{}}}
\makeatother
\usepackage{xcolor}
\usepackage{tikz}
\PassOptionsToPackage{hyphens}{url}
\usepackage[colorlinks=true,linkcolor=blue!60!black,citecolor=blue!60!black,urlcolor=blue!60!black]{hyperref}
\numberwithin{equation}{section}

\newtheorem{theorem}{Theorem}[section]
\newtheorem{proposition}[theorem]{Proposition}
\newtheorem{lemma}[theorem]{Lemma}
\newtheorem{corollary}[theorem]{Corollary}
\newtheorem{conjecture}[theorem]{Conjecture}
\theoremstyle{definition}
\newtheorem{definition}[theorem]{Definition}
\theoremstyle{definition}
\newtheorem{hypothesis}[theorem]{Hypothesis}
\theoremstyle{remark}
\newtheorem{remark}[theorem]{Remark}

\newcommand{\ket}[1]{\left| #1 \right\rangle}
\newcommand{\bra}[1]{\left\langle #1 \right|}
\newcommand{\R}{\mathbb{R}}
\newcommand{\Z}{\mathbb{Z}}

\newcommand{\E}{\mathbb{E}}

\newcommand{\tr}{\operatorname{tr}}

\newcommand{\rank}{\operatorname{rank}}
\newcommand{\softmax}{\operatorname{softmax}}
\newcommand{\LN}{\operatorname{LN}}
\newcommand{\RoPE}{\operatorname{RoPE}}
\newcommand{\iswiglu}{\operatorname{iSwiGLU}}
\newcommand{\Alm}{A_{\mathrm{LM}}}
\newcommand{\Glm}{G_{\mathrm{LM}}}
\newcommand{\Elm}{E_{\mathrm{LM}}}
\newcommand{\Vlm}{V_{\mathrm{LM}}}
\newcommand{\Ap}{A_{P}}
\newcommand{\dk}{d_k}
\newcommand{\dmodel}{d_{\mathrm{model}}}
\newcommand{\Adff}{A_{d\!f\!f}}
\newcommand{\Had}{\odot}
\newcommand{\norm}[1]{\left\lVert #1 \right\rVert}
\newcommand{\abs}[1]{\left\lvert #1 \right\rvert}
\newcommand{\ip}[2]{\left\langle #1, #2 \right\rangle}

\title[Power law graph attention]{Power law graph attention:\\
exact generalization of scaled dot-product attention,\\
empirical collapse at inference}

\author{Burc Gokden}
\address{Fromthesky Research Labs LLC, Oregon, USA}
\email{burc@fromtheskyresearchlabs.com}

\makeatletter
\def\@setauthors{\begingroup
  \trivlist
  \centering\footnotesize \@topsep30\p@\relax
  \advance\@topsep by -\baselineskip
  \item\relax
  {\scshape Burc Gokden}\\[3pt]
  {\itshape Fromthesky Research Labs LLC, Oregon, USA}\\[1pt]
  {\upshape\ttfamily burc@fromtheskyresearchlabs.com}\endtrivlist
  \endgroup
}
\def\enddoc@text{\ifx\@empty\@translators\else\@settranslators\fi}
\def\paragraph{\@startsection{paragraph}{4}\z@{.5\baselineskip\@plus.2\baselineskip\@minus.1\baselineskip}{-\fontdimen2\font}\normalfont}
\makeatother

\date{August 10, 2026}

\subjclass[2020]{Primary 68T07; Secondary 15B48, 05C50, 82C27}

\keywords{Large language models, attention mechanisms, power law graph
attention, learned bilinear operators, Perron--Frobenius theory,
self-organized criticality}

\begin{document}

\begin{abstract}
The Large Language Model from Power Law Decoder Representations
(PLDR-LLM) and its attention, Power Law Graph Attention (PLGA),
replace the fixed bilinear form of scaled dot-product attention (SDPA)
with a \emph{learned, input-generated bilinear operator} $\Glm$, built
from a positive tensor $\Alm$ by elementwise power laws. The
architecture is fully specified, verified against pinned reference
releases; claims are labeled theorem, conditional theorem,
measurement, or conjecture. Unconditionally: PLGA contains SDPA
exactly at $\Glm=I$; $\Alm$ and $\Ap$ are strictly entrywise positive,
with Perron--Frobenius structure on $\Alm$; the DAG regularizer has
the NOTEARS walk-counting form and positivity obstructs exact
acyclicity; and, under nonresonance (satisfied by standard rotary
frequencies), a commutant criterion identifies which operators
preserve relative-position dependence. An inference-collapse theorem:
exact input invariance of the head score operators $\Glm$ collapses
inference to generalized SDPA with a constant operator per head. Measured invariance:
relative fluctuations of $10^{-6}$ and below; perturbation bounds
quantify but do not certify cached inference; the assembled proxy
misses the decoding margin. A conditional three-stage mechanism
(rotary twirl, concentration, row-map contraction) is measured on a
released checkpoint. Blockwise training and scoring under the global
Gram are stated with explicit target exposure; on tested samples,
block and sequential scoring select identical answers and agree on the
published TruthfulQA probability-mass metric within $5\times10^{-5}$
per item. Self-organized criticality enters as a phenomenological
framework with an intrinsic order parameter; open claims become
falsifiable conjectures. Selected proof cores are machine-checked in
Lean~4.
\end{abstract}

\maketitle

\vfill\newpage

\tableofcontents

\vfill\newpage

\section{Introduction}
\label{sec:intro}

Modern large language models are overwhelmingly built on the transformer
\cite{vaswani2017} with scaled dot-product attention (SDPA): attention
scores are Euclidean inner products of linearly projected token
representations. The Power Law Graph Attention (PLGA) mechanism
\cite{gokden2019,gokden2021} and the decoder-only architecture built from it,
the Large Language Model from Power Law Decoder Representations (PLDR-LLM)
\cite{gokden2024,gokden2025,gokden2026}, generalize this design in a specific
and mathematically meaningful way: \emph{the bilinear form used to compare
queries and keys is itself learned, nonlinearly, from the input}, through a
chain
\begin{gather*}
\text{input} \;\longrightarrow\; \text{query Gram (``density'') operator}
\;\longrightarrow\; \text{positive tensor } \Alm \\
\;\longrightarrow\; \text{potential tensor } \Ap
\;\longrightarrow\; \text{bilinear score operator } \Glm
\;\longrightarrow\; \text{attention } \Elm ,
\end{gather*}
in which the step $\Alm \mapsto \Ap = \Alm^{\Had P}$ imposes an element-wise
\emph{power law} with learned exponents, and the step $\Ap \mapsto \Glm =
a\,\Ap + b_a$ superposes the resulting interaction profiles. The
intermediate tensors are exposed as \emph{deductive outputs} of the model,
alongside the usual next-token probabilities (the \emph{inductive output}).
The title's two claims are glossed here once. ``Exact
generalization'' means precisely this algebraic relationship: SDPA is
contained in the PLGA family exactly, as its $\Glm = I$ point
(Theorem~\ref{thm:collapse}(i), machine-checked); containment does not
assert strict function-class separation at matched resources, which is
explicitly not proved (Remark~\ref{rem:asymmetry-scope}), nor
performance superiority. ``Empirical collapse at inference'' names the
phenomenon of Section~\ref{sec:invariance} at its actual epistemic
level: the collapse theorem is proved conditionally on exact operator
invariance, and the invariance itself is a measured hypothesis, not a
theorem.
Three observations on released checkpoints motivate the mathematical
questions; their scope is the tested checkpoints and workloads.

\begin{enumerate}[leftmargin=2em]
\item \textbf{Operator invariance.} On the tested workloads of selected
pretrained checkpoints, the deductive outputs are invariant under change
of input up to a
minute perturbation (relative fluctuations of order $10^{-6}$--$10^{-11}$,
and $0$ at floating-point resolution for the best models), so the entire
deep nonlinear PLGA subnetwork can be replaced at inference by a single
cached tensor operator $\Glm$; the published cached-versus-uncached
benchmark evaluations are unchanged \cite{gokden2025} (scored by
one-pass block scoring, a path in which no generation-time cache
participates by construction, see Appendix~\ref{app:repos}; the
substantive fidelity evidence is the measured operator and logit
agreement of Appendix~\ref{app:numerical}).
\item \textbf{A learned singularity condition.} On the primary audited
checkpoint, the generator $A$ is numerically singular in the sharpest sense: rank one
under an explicit singular-value tolerance, with identical rows
(identical, moreover, \emph{across attention heads} within a layer).
For the identical-row generator, the only possible nonzero eigenvalue
is the common row sum, whose absolute value is its spectral radius.
If the row sum vanishes, there is no nonzero eigenvalue, including in
the nonzero rank-one case of Proposition~\ref{prop:rankone}. The interaction tensor $\Alm$ built from it remains
strictly entry-wise positive with Perron--Frobenius spectral structure
\cite{gokden2025}, and is generically of near-full numerical rank: the
float-zero determinants reported for both tensors \cite{gokden2025}
are confirmed as rank statements for $A$ by its singular-value
spectrum, and traced for $\Alm$ to determinant underflow with no rank
implication (Remark~\ref{rem:rankone}, Appendix~\ref{app:numerical}).
\item \textbf{Critical-like training phenomenology.} In the published
experiments, whether the model generalizes (produces coherent language and
transfers to reasoning benchmarks) tracks whether pretraining is carried
out in a critical-like regime of its driving/dissipation dynamics, with
maximum learning rate and warm-up steps as control parameters; an
intrinsic mean-normalized deductive-fluctuation diagnostic
separates the observed phases in that sample \cite{gokden2026}. The source
papers read this as self-organized criticality; in this paper that reading
is treated as a phenomenological framework and hypothesis, not an
established result (Section~\ref{subsec:soc}).
\end{enumerate}

The goal of this paper is a full analytical description of PLDR-LLM and
PLGA. We (a) give formal definitions of every component and the end-to-end
model map, each verified against the reference implementations; (b) prove
the structural properties underlying the three observations above, citing
existing proofs where they exist and supplying proofs where they do not,
including a conditional mechanistic analysis of how invariance of the
generator $A$ can arise inside the metric learner, with hypotheses
stated and their measurable diagnostics computed directly
(Section~\ref{subsec:origin}, Appendix~\ref{app:numerical}); (c)
consolidate these results into a comparison of PLDR-LLM with its SDPA
base point in which each advantage is tied to a result at its actual
epistemic strength, stating the costs with equal explicitness
(Section~\ref{sec:advantages}); and (d) collect the program's open
claims as precise, falsifiable conjectures (rigidity of the invariant
operator, a spectral form of the order parameter, and operator transfer
across domains; Section~\ref{sec:conjectures}).

\paragraph{Epistemic conventions.}
Every claim in this paper carries one of four labels, and the categories
are never blurred. (1) \emph{Algebraic theorems} (theorem, proposition,
lemma, corollary environments with no unverified hypotheses) are proved in
place or cited. (2) \emph{Conditional theorems} are proved under
hypotheses, stated inside the formal statement itself, that are
\emph{not} verified for trained models (e.g.\ Lipschitz or stationarity
assumptions); their conclusions inherit that conditionality wherever they
are used. (3) \emph{Empirical observations} are cited to the specific
source experiment and never restated as theorems; floating-point zeros
are reported as finite-precision observations, not exact identities.
(4) \emph{Analogies and conjectures} are flagged as such in their
statements (conjectures in a dedicated environment) and are never used
as premises of proved results. Conjecture~\ref{conj:modeltransfer}
explicitly assumes Conjecture~\ref{conj:functorial}; those two open
claims are logically dependent. A correct elementary identity is not treated
as evidence for a stronger interpretive claim in whose direction it
points. Where a hypothesis is measurable, we measure it on a released
checkpoint (Appendix~\ref{app:numerical}) instead of assuming it.

\paragraph{Related work.}
The ingredients of PLGA have distinct lineages, and priority is claimed
narrowly. Learned and input-generated attention patterns predate this
program: Synthesizer \cite{tay2020} learns or generates score matrices
directly, and talking-heads attention \cite{shazeerTH2020} learns linear
maps across heads; SDPA itself carries a learned constant bilinear form
$W_Q W_K^\top$ in pre-projection coordinates. The broader
dynamic-parameter lineage is also necessary context: hypernetworks
\cite{ha2016} generate the weights of one network with another; bilinear
attention networks \cite{kim2018} learn explicit bilinear attention
structure; the fast-weight-programmer line \cite{schlag2021,irie2021}
develops the view of attention as input-updated operator/associative
memory, with linear attention \cite{katharopoulos2020} as a
fast-weight update rule. That line supplies a construction
worth contrasting directly: its \emph{cumulative} outer-product state
$S_t = \sum_{n \le t} \phi(k_n) v_n^\top$ is updated causally token by
token, so every row of one parallel pass is its own sequential
conditional (historical-row prefix consistency,
Definition~\ref{def:prefixconsistent}, holds by construction), whereas
PLGA's \emph{global} query Gram deliberately spends that property to
let all known tokens shape a nonlinear learned operator
(Section~\ref{subsec:online}), with the masked cumulative Gram
\eqref{eq:cumgram} as the interpolating option; no equivalence between
the two mechanisms is claimed; and dynamic
bilinear low-rank attention (DBA) \cite{qin2022} is the closest in
name, generating \emph{input-sensitive low-rank projection
matrices} that compress sequence length for efficient attention,
whereas PLGA generates a $\dk \times \dk$ \emph{head-space score
operator} from an entrywise positive base by a learned entrywise power
transformation followed by an unrestricted affine coupling. Its intermediate
tensors and final operator are exposed for regularization, measurement,
and caching: the
shared idea is input-conditioned bilinear structure, the object and
purpose differ. Closest at the level of the score bilinear form,
PaTH attention \cite{yang2025} also places a data-dependent matrix
inside the query--key logit, writing $q_i^\top H_{ij} k_j$ with
$H_{ij}$ an accumulated product of data-dependent
identity-plus-rank-one (Householder-like) transitions along the
positional path from $j$ to $i$: a content-conditioned position
encoding with one transition per query--key pair, built causally
link by link, so that one parallel pass keeps the cumulative prefix
semantics of the fast-weight line. PLGA instead generates a
\emph{single} sequence-global $\dk \times \dk$ operator per layer
and head from the query Gram through a deep row-wise metric learner
and an entrywise positive power-law base followed by an unrestricted
affine coupling, shares it
across every row of the call, and exposes it as a tensor for
regularization, spectral measurement, transfer, and caching; the
global Gram spends the historical-row prefix consistency
(Section~\ref{subsec:online}) that the accumulated construction
keeps, and neither construction reproduces the other.
Closest in the operator's generating statistic,
Curvature-Conditioned Query (CCQ) \cite{le2026ccq} likewise inserts
a context-generated $\dk \times \dk$ operator into the bilinear
key--query score, deriving it from a second-moment statistic of the
context: the centered causal running key covariance $\Sigma_t$ of
the prefix enters the prescribed affine contraction
$I - \lambda_t \Sigma_t$, with one learned scalar gate $\lambda_t$
per token, so that the cleaned read is
$q_t^\top (I - \lambda_t \Sigma_t)\, k_j$, a per-token causal
correction to the read of a linear-attention memory that keeps the
prefix semantics of its recurrent state. PLGA also generates its
operator from a second-moment statistic, the global rotated-query
Gram, but maps it through a deep learned generator, an entrywise positive
power-law base and an unrestricted affine coupling, shares the resulting
operator across every row
of the call, and exposes it as a deductive tensor for
regularization, spectral measurement, transfer, and caching, at the
cost of the historical-row prefix consistency
(Section~\ref{subsec:online}) that the running-covariance state
keeps; a prescribed contraction with a scalar gate differs from a
learned deep generator with a positive power-law base and unrestricted
final coupling, a per-token causal state
from one sequence-global row-shared operator, and a
linear-attention read correction from an operator inside ordinary
softmax scores, so neither construction reproduces the other.
Closest in feature-space mechanics,
cross-covariance attention (XCiT) \cite{elnouby2021} also contracts
the token axis into a $\dk \times \dk$ feature-space matrix (the
key--query cross-covariance), but uses that matrix, normalized and
softmax-ed, \emph{as the attention map itself}, mixing feature
channels in place of token--token attention for linear-in-tokens
efficiency; PLGA instead feeds its query Gram to a deep nonlinear
row-wise metric learner and power/coupling stages whose final output
$\Glm$ is a bilinear operator
applied \emph{inside} ordinary token--token attention scores, which
remain the mixing mechanism, and the operator (not the mixing) is
what is regularized, measured, and cached. The two lines developed
in parallel: XCiT and the encoder--decoder PLGA of \cite{gokden2021}
appeared within weeks of each other in mid-2021, the feature-space
lineage here descending from the graph-attention construction of
\cite{gokden2019}. Closest in interpretive frame, \emph{Attention as a
Hypernetwork} \cite{schug2025} reads standard multi-head attention
itself as an implicit hypernetwork producing key--query-conditioned
operations, sharpening the operator-generation view of attention;
PLGA generates an explicit head-space score operator from an entrywise
positive power-law base through an unrestricted affine coupling. The
operator is exposed for regularization, measurement and empirical caching.
Closest in the operational use of a fixed head-space operator,
Autonomy-of-Heads (AoH) \cite{yang2026aoh} extracts the frozen
bilinear form that each trained SDPA head carries in pre-projection
coordinates and uses its effective rank as a data-free spectral
diagnostic, classifying retrieval heads (concentrated spectra)
against streaming heads (diffuse spectra) and driving
sparse-attention and KV-cache policy from the classification alone;
the fixed pre-projection form thus supports direct, operationally
consequential spectral diagnostics without any input. PLGA's
operator differs in being generated anew from each input's
rotated-query Gram in head space through a deep generator, an entrywise
positive power-law base and an unrestricted affine coupling, shared
across every row of the call, and
exposed as a deductive tensor for fluctuation, spectral, transfer,
and collapse analysis; AoH has no input-conditioned generator and
no invariant-operator removal theorem, and neither construction
reproduces the other. None of
these anticipates the specific PLGA construction, but a claim centered
on input-conditioned operator generation and removable fast/deductive
weights belongs inside that lineage. What is distinctive in PLGA
is the \emph{input-conditioned generation of the head-space operator}
through an entrywise positive power-law base followed by unrestricted
coupling (Section~\ref{subsec:bilinear}), together with the
removable/cacheable deductive network of Section~\ref{sec:invariance}.
Here positivity describes the base $\Alm$ and powered intermediate $\Ap$,
not the final score operator $\Glm$; Definition~\ref{def:plga}
and Proposition~\ref{prop:positivity} give the exact stages
and their properties.
On the training side, a contemporaneous gradient-flow analysis of
standard attention \cite{vashisht2026} shows that factorizing the
query--key and output--value circuits implicitly rescales their
relative learning rates, with faster query--key movement sharpening
attention at comparable loss; the analysis concerns fixed SDPA
parameterizations, with no input-conditioned operator, and it is
independent precedent for the reading adopted in
Theorem~\ref{thm:collapse}(iii) and
Remark~\ref{rem:asymmetry-scope}: a parameterization changes the
gradient geometry of training without, by itself, changing the
realized function class. Softmax attention as
an averaging (Markov) operator and the rank-collapse/oversmoothing
phenomenology have an established literature \cite{dong2021,sander2021}.
The rotary-embedding structure we analyze originates with RoFormer
\cite{su2021}. Its design line generalizes the rotations through
commuting/group-structured generators \cite{ostmeier2024,yu2025}; its
analysis line has characterized the surviving symmetries: the
gauge-symmetry characterization of \cite{wangwang2025} identifies the
invertible query/key reparameterizations of rotary attention exactly
with the commutant of the position rotations (typically
$(\mathrm{GL}(1,\mathbb{C}))^{\dk/2}$ for standard RoPE), and the
functional-equivalence analysis of \cite{tran2026} ties the
symmetries of rotary attention to the same commutant. The rotation
commutant is thus an established organizing principle of this
literature, and its
mathematical content (the commutant of a semisimple matrix with
distinct eigenvalues) is classical linear algebra;
Proposition~\ref{prop:rope-commutant} contributes the PLGA-specific
instance: an \emph{arbitrary inserted operator} $G$ between
rotated queries and keys preserves offset-only score dependence
exactly when it lies in that commutant, under an explicit
nonresonance hypothesis. From this instance follow the absorption
step of Theorem~\ref{thm:collapse}(ii), the head-space codimension
count of Corollary~\ref{cor:posgap}, and the commutant-residual
diagnostic of Appendix~\ref{app:numerical}. The DAG regularizer is NOTEARS
\cite{zheng2018}. The remaining mathematical tools (Perron--Frobenius
theory, multiplicative Cauchy equations, geometric sums, Lipschitz
composition, rank-one algebra) are classical; this paper's mathematical
contribution is application and synthesis, not new general theorems.

\paragraph{How to read this paper.}
Section~\ref{sec:prelim} fixes notation. Sections~\ref{sec:plga} and
\ref{sec:pldr} are the formal core: the PLGA operator and the full PLDR-LLM
architecture, with basic propositions. Section~\ref{sec:invariance} contains
the operator-invariance theory (inference collapse, perturbation bounds,
order parameter), culminating in the new mechanism analysis of
Section~\ref{subsec:origin}. Section~\ref{sec:criticality} formalizes scale
invariance and the SOC training picture.
Section~\ref{sec:advantages} consolidates, with formal backing, the
advantages of PLDR-LLM over SDPA-LLM and their costs.
Section~\ref{sec:conjectures} collects the program's conjectures.
Section~\ref{sec:discussion} collects open problems. Proofs are
short and given in place. A notation table is provided in
Appendix~\ref{app:notation}; code, model, and verification resources in
Appendix~\ref{app:repos}.

\clearpage
\paragraph{Dependency structure of the main claims.}
The load-bearing chain of the paper, with the epistemic status of each
link made explicit, is:
\enlargethispage{\baselineskip}
\begingroup
\small
\begin{longtable}{@{}l>{\raggedright\arraybackslash}p{0.36\textwidth}>{\raggedright\arraybackslash}p{0.24\textwidth}@{}}
\caption{Dependency structure of the main claims. No result in a lower row is used as a premise for a claim in a higher row.}\label{tab:claim-dependencies}\\
\toprule
Claim & Rests on & Status of the link \\
\midrule\endfirsthead
\caption[]{Dependency structure of the main claims (continued).}\\
\toprule
Claim & Rests on & Status of the link \\
\midrule\endhead
\midrule
\multicolumn{3}{r}{\footnotesize Continued on the next page.}\\
\endfoot
\bottomrule\endlastfoot
SDPA $=$ PLGA at $\Glm = I$ & Thm.~\ref{thm:collapse}(i) &
proved, machine-checked \\
Exact cache under invariance & Prop.~\ref{prop:cacheable}(i),
Thm.~\ref{thm:collapse}(ii) & proved (sufficiency only) \\
Observed invariance & \cite{gokden2025,gokden2026} +
App.~\ref{app:numerical} & empirical \\
Cache fidelity bound & Prop.~\ref{prop:perturb},
Cor.~\ref{cor:budget} & proved; constants evaluated, do not close
margins \\
Invariance mechanism & Lem.~\ref{lem:twirl} (constants large),
Prop.~\ref{prop:conc} (idealized), Lem.~\ref{lem:ln} (scale error
quantified), Prop.~\ref{prop:contract} (uniform $L_j$ assumed; composite diagnostics sampled) &
conditional analysis \\
Selection at criticality & Hypothesis~\ref{hyp:selection} &
hypothesis \\
Positional codimension & Prop.~\ref{prop:rope-commutant},
Cor.~\ref{cor:posgap} & conditional theorem (head-level) \\
SOC reading & Def.~\ref{def:soc}, \S\ref{subsec:soc} &
phenomenological framework \\
Spectral dictionary & Prop.~\ref{prop:gap} +
Conj.~\ref{conj:spectral} & conditional lemma + conjecture \\
\end{longtable}
\endgroup

\paragraph{Machine-checked proofs.}
Selected elementary algebraic and analytic proof cores in this paper
have been formalized in Lean~4 over mathlib and verified by the Lean
proof checker; the library builds with no unproved obligations, and a
continuous-integration axiom audit confirms that every exported
theorem depends only on mathlib's standard classical principles
(Appendix~\ref{app:lean}). The Lean
source is available at
\url{https://github.com/burcgokden/PLDR-LLM-Math-Foundations}
(Appendix~\ref{app:repos}). Appendix~\ref{app:lean} gives the exact
claim-by-claim coverage table: which clauses are kernel-checked and
which remain written obligations. Formalization of a proof core does
not establish adjacent analytical or empirical claims.

\section{Preliminaries and Notation}
\label{sec:prelim}

\subsection{Tokens, graphs, and the quantization/manifold duality}

Fix a finite vocabulary $\mathcal{V}$ with $\abs{\mathcal{V}} = V$ (in the
reference implementations, a SentencePiece unigram vocabulary with
$V = 32{,}000$ \cite{gokden2024}). A \emph{context} is a finite sequence
$x = (x_1,\dots,x_S) \in \mathcal{V}^S$ with $S \le S_{\max}$ (context length
$S_{\max}=1024$ in \cite{gokden2024,gokden2025,gokden2026}). An embedding map
\[
\iota : \mathcal{V} \to \R^{\dmodel}
\]
assigns each token a dense feature vector; a context is identified with the
matrix $X \in \R^{S\times\dmodel}$ whose $i$-th row is $\iota(x_i)^\top$.

Following \cite{gokden2021,gokden2024}, a context is regarded as a weighted
graph $G=(\mathcal{V}_x, E_x)$ whose nodes are the tokens with feature
vectors $\iota(x_i)$. Two dual descriptions coexist:
\begin{itemize}[leftmargin=2em]
\item the \emph{quantization set}: the discrete vocabulary $\mathcal{V}$,
  from which contexts are sampled as graph instances (a ``local'' object);
\item the \emph{language-model manifold}: a $\dmodel$-dimensional continuous
  feature space whose interaction structure is learned globally
  from the ensemble of all instances.
\end{itemize}
PLGA is designed so that dataset-level (\emph{global}) structure is carried
by learned parameters $(P, a, b_a, W, b_W)$, while instance-level
(\emph{local}) structure is carried by the inferred tensors
$A$, $\Alm$, $\Ap$, $\Glm$, $\Elm$ \cite{gokden2021}. A central empirical result of
\cite{gokden2025}, formalized in Section~\ref{sec:invariance}, is that after
pretraining at criticality the ``local'' tensors become global: they are
(numerically) independent of the instance.

\subsection{Matrix conventions}

For matrices $M, N$ of equal shape, $M \Had N$ is the Hadamard (element-wise)
product, and for $M$ with positive entries and any real matrix $P$ of the
same shape, the \emph{element-wise power} is
\begin{equation}
\label{eq:hadpower}
\bigl(M^{\Had P}\bigr)_{ij} \;=\; M_{ij}^{\,P_{ij}}
\;=\; \exp\!\bigl(P_{ij}\,\log M_{ij}\bigr).
\end{equation}
Plain juxtaposition $aM$ of two matrices always denotes the ordinary matrix
product. $\mathbf{1}$ is the all-ones column vector (its dimension clear
from context), so $\mathbf{1}\mathbf{1}^\top$ is the all-ones matrix; $I$
is the identity. Rows of a matrix are written $M_{i,:}$ or, where
subscripts would crowd, $M[i,\cdot]$; likewise $V[j]$ is row $j$ of $V$.
$\softmax$ acts row-wise: $\softmax(z)_i = e^{z_i}/\sum_j e^{z_j}$. For a
nonempty index set $\mathcal{J} \subseteq \{1,\dots,S\}$ the \emph{ideal
masked softmax} is
\[
\softmax_{\mathcal{J}}(z)_i \;=\;
\begin{cases}
e^{z_i}\big/\sum_{j\in\mathcal{J}} e^{z_j}, & i \in \mathcal{J},\\
0, & i \notin \mathcal{J},
\end{cases}
\]
equivalently the ordinary softmax with masked entries set to $-\infty$ in
the extended reals. All causal-mask statements in this paper are proved
for the ideal masked softmax; the implementations realize it by adding a
finite negative constant to disallowed scores ($-10^9$ in the native
TensorFlow and PyTorch repositories; the dtype minimum
\texttt{torch.finfo(dtype).min}, via the Transformers causal-mask
utility, in the Hugging Face port), whose exact real-valued softmax is
strictly positive everywhere and which reproduces the ideal operator
only through floating-point underflow of the masked entries under
ordinary score ranges. We keep the two semantics explicitly separate
(Remark~\ref{rem:finitemask}).
$\LN$ denotes LayerNorm \cite{ba2016} acting on the last (feature) axis;
throughout the reference implementations LayerNorm uses
$\varepsilon_{\LN} = 10^{-6}$ and learned affine parameters; the
$\varepsilon_{\LN}$-dependence of its invariances is made explicit in
Lemma~\ref{lem:ln}.
$\norm{\cdot}_2$ is the spectral norm for
matrices and the Euclidean norm for vectors; $\norm{\cdot}_F$ the Frobenius
norm; $\norm{\cdot}_\infty$ the operator norm induced by the sup-norm
(maximum absolute row sum). The Swish/SiLU function is $\sigma_s(u) = u\,\varsigma(u)$ with
$\varsigma$ the logistic sigmoid. The source papers write
$D_Q = Q^\top Q$ in Dirac notation as $\ket{Q^\top}\!\bra{Q^\top}$
\cite{gokden2021}; in this paper we use plain matrix notation throughout
and record the correspondence here once for cross-reference.

\section{The Power Law Graph Attention Operator}
\label{sec:plga}

Throughout this section we work inside a single attention head of width
$\dk = \dmodel/h$, where $h$ is the number of heads. All statements extend
head-wise; the multi-head structure is discussed in
Section~\ref{subsec:multihead}.

\subsection{Formal definition}

\begin{definition}[iSwiGLU]
\label{def:iswiglu}
The \emph{identity-SwiGLU} activation is the map
$\iswiglu : \R \to \R_{\ge 0}$,
\begin{equation}
\iswiglu(u) \;=\; \sigma_s(u)\cdot u \;=\; u^2\,\varsigma(u) \;\ge\; 0,
\end{equation}
applied element-wise; it is SwiGLU \cite{shazeer2020,dauphin2017} with both
weight matrices set to the identity and no bias \cite{gokden2024}. It is
smooth, non-negative, and vanishes only at $u=0$.
\end{definition}

\begin{definition}[Power Law Graph Attention
\cite{gokden2021,gokden2024,gokden2026}]
\label{def:plga}
Let $Q, K, V \in \R^{S \times \dk}$ be query, key, and value matrices for a
context of length $S$. Let $\Phi_{\mathrm{res}} : \R^{\dk\times\dk} \to
\R^{\dk\times\dk}$ be a deep residual network, \emph{shared by all heads of
the decoder layer}, consisting in the reference design of
$N_{\mathrm{res}} = 8$ residual units, each applying $n_A = 2$ SwiGLU blocks
in succession, where each block is a complete gated map
$\R^{\dk} \to \R^{\dk}$ with hidden width $\Adff$ (two linear maps
$\dk \to \Adff$ whose outputs are multiplied elementwise, followed by a
linear map $\Adff \to \dk$, acting on the last axis), followed by a
residual sum and LayerNorm; it therefore acts \emph{row-wise}
on its matrix argument (Proposition~\ref{prop:rowfact}). Let
$W, b_W, P, a, b_a \in \R^{\dk\times\dk}$ be five full parameter matrices
per head, and $\epsilon > 0$ a small constant ($\epsilon = 10^{-9}$). PLGA
is the composite map defined by
\begin{align}
D_Q &\;=\; Q^\top Q
  && \text{(``density operator'': query Gram matrix)} \label{eq:density}\\
A &\;=\; \Phi_{\mathrm{res}}\bigl(\LN(D_Q)\bigr)
  && \text{(generator)} \label{eq:metricgen}\\
\Alm &\;=\; \iswiglu\bigl(W A + b_W\bigr) + \epsilon
  && \text{(``metric tensor'': positive interaction tensor)} \label{eq:metric}\\
\Ap &\;=\; \Alm^{\Had P}
  && \text{(potential tensor)} \label{eq:potential}\\
\Glm &\;=\; a\, \Ap + b_a
  && \text{(``energy--curvature'': bilinear score operator)} \label{eq:ec}\\
E &\;=\; \frac{Q\, \Glm\, K^\top}{\sqrt{\dk}}
  && \text{(scores)} \label{eq:scores}\\
\Elm &\;=\; \softmax\bigl[\operatorname{mask}(E)\bigr]
  && \text{(attention operator)} \label{eq:attn}\\
\Vlm &\;=\; \Elm V
  && \text{(inductive head output).} \label{eq:out}
\end{align}
In \eqref{eq:metric} and \eqref{eq:ec}, $WA$ and $a\,\Ap$ are ordinary
matrix products ($W$ and $a$ act on the left), and the biases are full
matrices added entry-wise; component-wise,
\begin{equation}
\label{eq:components}
(\Alm)_{ij} = \iswiglu\Bigl(\textstyle\sum_k W_{ik} A_{kj} + (b_W)_{ij}\Bigr)
+ \epsilon,
\qquad
(\Glm)_{ij} = \textstyle\sum_k a_{ik}\, (\Ap)_{kj} + (b_a)_{ij}.
\end{equation}
Row $i$ of $\Glm$ is thus a coupling-weighted \emph{superposition of the
interaction profiles} (rows) of the potential tensor; this is the precise
sense of ``superposition of potentials'' in \cite{gokden2021,gokden2026}. The tuple
$(A, \Alm, \Ap, \Glm, \Elm)$ constitutes the \emph{deductive outputs};
$\Vlm$ is the \emph{inductive output} of the head. $\operatorname{mask}$
restricts the \emph{score support} to the causal index sets
$\mathcal{J}_i = \{j : j \le i\}$: in this paper's theorems it is the
ideal masked softmax of Section~\ref{sec:prelim}; the implementations
add a large finite negative constant to disallowed scores before the
softmax, with the provenance-specific values as recorded in
Section~\ref{sec:prelim} (Remark~\ref{rem:finitemask}). The mask
constrains which keys each row may attend to; what score-support
masking does and does not imply for the row-wise dependence structure
of the full model is stated precisely in
Section~\ref{subsec:online}.
\end{definition}

Positivity applies to $\Alm$ and $\Ap$, not to the final score operator:
the real matrices $a,b_a$ in \eqref{eq:ec} are unrestricted. For example,
$a=\bm0_{\dk\times\dk}$ and $b_a=-I_{\dk}$ give $\Glm=-I_{\dk}$.
Thus $\Glm$ need not be entrywise positive, symmetric, or positive
semidefinite. The uniform entry floor of $\Alm$ also does not pass
unchanged through arbitrary learned powers to $\Ap$.

\begin{remark}[Implementation provenance]
\label{rem:provenance}
Equations \eqref{eq:density}--\eqref{eq:out} are the canonical form
implemented identically in the TensorFlow v500 code accompanying
\cite{gokden2024}, the PyTorch v510 code accompanying
\cite{gokden2025,gokden2026}, and the Hugging Face
\texttt{PldrllmForCausalLM} port (Appendix~\ref{app:repos}); for this
article we have verified them against all three. In the LLM the query and key are
rotary-rotated \emph{before both} the density operator \eqref{eq:density}
and the scores \eqref{eq:scores}; see Section~\ref{subsec:rope} and
Definition~\ref{def:pldr}. The original encoder--decoder formulation
\cite{gokden2021} additionally applied a LeakyReLU before the softmax and
used ReLU in place of iSwiGLU/SwiGLU; \cite{gokden2024} removed the former
and introduced the latter. Initialization depends on the implementation lineage. In the legacy
PyTorch v510 row layer, $W,P,a$ use Glorot normal; the Hugging Face port
uses Glorot uniform for these tensors. Both initialize the PLGA biases
to zero. Dense weights use the documented Glorot-uniform rules,
subject to the native $W_V$ exception noted with
Definition~\ref{def:pldr}; embeddings and LayerNorm follow their own
rules. A loaded checkpoint is not a fresh draw from the constructor
initializer. Appendix~\ref{app:numerical} distinguishes the
initialization controls used in the audit.
\end{remark}

\begin{remark}[Genealogy: the screened-Coulomb origin]
PLGA descends from the CoulGAT mechanism \cite{gokden2019}, where attention
on molecular graphs was computed from a hand-engineered inverse-square
distance adjacency $A_{ij} = d(i,j)^{-2}$ raised to a \emph{learnable}
element-wise power and row-softmax normalized,
$A_p = \softmax\bigl(A^{\Had P}\bigr)$, in analogy with a screened Coulomb
(Yukawa) potential $C e^{-Mr}/r$ \cite{yukawa1935}: the power matrix $P$
learns the \emph{range} and the coupling matrix the \emph{strength} of
pairwise interactions. PLGA replaces the hand-engineered $A$ by the learned
metric \eqref{eq:metric}, making the potential fully data-driven. The
interpretation of $\Glm$ as ``energy--curvature'' is by analogy with the
derivation of curvature tensors from a metric in Riemannian geometry and of
stress--energy sourcing curvature in general relativity \cite{misner1973};
the analogy is structural (nonlinear functional of a metric-like object),
not an identification, and the source papers are explicit on this point
\cite{gokden2026}.
\end{remark}

\begin{remark}[Nomenclature]
\label{rem:nomenclature}
The names \emph{density operator}, \emph{metric tensor}, and
\emph{energy--curvature tensor} are the established nomenclature of the
source-paper series \cite{gokden2019,gokden2021,gokden2024}; we retain the
symbols $\Alm, \Glm$ and quote those names for cross-reference, but no
formal statement in this paper uses ``metric,'' ``information metric,'' or
``curvature'' as a mathematical predicate, for a concrete reason:
entry-wise positivity implies neither symmetry nor positive definiteness.
For example
\[
M = \begin{pmatrix}1 & 10\\ 1 & 1\end{pmatrix} > 0 \ \text{entry-wise},
\qquad
\tfrac12(M + M^\top) = \begin{pmatrix}1 & 5.5\\ 5.5 & 1\end{pmatrix}
\ \text{has eigenvalues } 6.5,\ -4.5:
\]
neither $M$ nor its symmetric part defines a Riemannian or information
metric, and the same applies to $\Alm$ and $\Glm$, which are not
constrained symmetric. Operationally, $D_Q$ is an (unnormalized,
uncentered) query Gram/second-moment matrix, $\Alm$ a \emph{positive
interaction tensor}, and $\Glm$ a \emph{bilinear score operator}. A
genuine learned metric would require a symmetric positive-definite
parameterization (e.g.\ $BB^\top + \epsilon I$) and a specified
transformation law; the architecture imposes neither.
\end{remark}

\subsection{Basic structural properties}

\begin{proposition}[Density operator]
\label{prop:density}
$D_Q = Q^\top Q$ is symmetric positive semi-definite with
$\rank D_Q \le \min(S,\dk)$, and $\tfrac1S D_Q$ is the (uncentered) second
moment of the query token vectors. In particular, for short contexts
($S<\dk$) the raw instance geometry seen by the metric learner is
rank-deficient, and any full-rank structure in $\Alm$ is
necessarily \emph{generated} by $\Phi_{\mathrm{res}}$ and the downstream
maps rather than present as rank in the single-instance Gram matrix.
(Generation downstream does not by itself attribute the structure to
training data: biases, the nonlinearity, or random initialization can
produce full-rank $\Alm$ from a rank-deficient input. Attribution
requires a trained-versus-initialized comparison, which
Appendix~\ref{app:numerical} carries out: the same spectral battery run
at random initialization.)
\end{proposition}

\begin{proof}
$v^\top Q^\top Q v = \norm{Qv}_2^2 \ge 0$ and
$\rank(Q^\top Q) = \rank Q \le \min(S,\dk)$. The moment statement is the
definition of the empirical second moment.
\end{proof}

\begin{proposition}[Strict positivity and well-posedness of the power law]
\label{prop:positivity}
$\Alm \ge \epsilon > 0$ entry-wise. Consequently the element-wise power
\eqref{eq:hadpower} in \eqref{eq:potential} is well defined and jointly
smooth in $(\Alm, P)$, and for fixed $\Alm$ the family
$\{\Alm^{\Had tP}\}_{t\in\R}$ is a one-parameter multiplicative group:
\[
\Alm^{\Had (s+t)P} = \Alm^{\Had sP} \Had \Alm^{\Had tP},
\qquad \Alm^{\Had 0} = \mathbf{1}\mathbf{1}^\top .
\]
Equivalently, $L(t) = \log \Alm^{\Had tP} = t\, (P \Had \log\Alm)$ solves the
linear flow $\dot L = P \Had \log \Alm$: the potential tensor is the
time-$1$ point of a linear dynamical system in log-space.
\end{proposition}

\begin{proof}
$\iswiglu \ge 0$ by Definition~\ref{def:iswiglu}, so
$\Alm \geq \epsilon$ entry-wise by \eqref{eq:metric}. The remaining
statements follow from \eqref{eq:hadpower} and elementary properties of
$\exp$ and $\log$ applied entry-wise.
\end{proof}

\begin{theorem}[Perron--Frobenius structure of the positive interaction
tensor]
\label{thm:pf}
Each head's tensor $\Alm \in \R^{\dk\times\dk}$ is entry-wise
positive; hence its spectral radius $\rho(\Alm)$ is a simple, positive
eigenvalue (the Perron root) with entry-wise positive left and right
eigenvectors, and every other eigenvalue $\lambda$ satisfies
$\abs{\lambda} < \rho(\Alm)$.
\end{theorem}

\begin{proof}
This is the classical Perron--Frobenius theorem for positive matrices; see
\cite[Ch.~8]{hornjohnson} or \cite{seneta2006}. Positivity of $\Alm$ is
Proposition~\ref{prop:positivity}.
\end{proof}

The empirical finding of \cite{gokden2025} is sharper: at convergence, the
generator $A$ has (to high precision) \emph{identical rows}, with the same
row profile appearing on every head of a layer, and the determinants of
both $A$ and $\Alm$ evaluate to zero at floating-point resolution on
every head of every layer. (What these float-zero determinants do and do
not imply is quantified in Remark~\ref{rem:rankone} and
Appendix~\ref{app:numerical}: for $A$, the singular-value spectrum
confirms numerical rank one; for $\Alm$, it does not.) The
following elementary proposition is the exact statement of that limit
configuration for the generator, which \cite{gokden2025} calls the
\emph{learned singularity condition}.

\begin{proposition}[Rank-one singularity condition]
\label{prop:rankone}
Let $d=\dk\ge1$ and $\mathbf{1}_d=(1,\ldots,1)^{\mathsf T}$.
Suppose $A = \mathbf{1}_d\alpha^\top$ for some row vector
$\alpha^\top \in \R^{1\times d}$ (identical rows), and write
$s = \alpha^\top\mathbf{1} = \sum_j \alpha_j$ for the common row sum.
Then:
\begin{enumerate}[label=(\roman*),leftmargin=2.2em]
\item $\rank A \le 1$, with equality if and only if $\alpha \ne \bm0$, and
$\det A = 0$ whenever $\dk \ge 2$;
\item $A\mathbf{1} = s\mathbf{1}$ and the characteristic polynomial of
$A$ is $\lambda^{\dk-1}(\lambda - s)$, so the spectrum is
$\{s\} \cup \{0\}$ for $\dk \ge 2$ and $\{s\}$ for $\dk = 1$; if
$s \ne 0$, then $s$ is the unique nonzero
eigenvalue, $\abs{s} = \rho(A)$, and $A/s$ is idempotent, a rank-one
(oblique) projection; if $s = 0$ and $\alpha \ne \bm0$, then $A$ is nonzero
and nilpotent of index $2$ ($A^2 = \bm0_{\mathrm{mat}}$), with spectral radius $0$;
\item $A^k = s^{k-1} A$ for all $k\ge 1$.
\item (Perturbation, singular-value form.) If
$A' = \mathbf{1}\alpha^\top + \Delta$ with $\norm{\Delta}_2 \le \delta$,
then $\sigma_j(A') \le \delta$ for every $j \ge 2$, and hence
\[
\abs{\det A'} \;\le\; \bigl(\norm{\mathbf{1}\alpha^\top}_2 +
\delta\bigr)\,\delta^{\dk-1}
\;=\; \bigl(\sqrt{\dk}\,\norm{\alpha}_2 + \delta\bigr)\,\delta^{\dk-1}.
\]
No conclusion about the locations of individual \emph{eigenvalues} of
$A'$ is drawn: $A'$ is in general nonnormal, and its eigenvalues can move
at rate $\sqrt{\delta}$ under a perturbation of size $\delta$ (take
$\mathbf{1}\alpha^\top = \begin{psmallmatrix}1&-1\\1&-1\end{psmallmatrix}$
and $\Delta = \begin{psmallmatrix}0&0\\\delta&0\end{psmallmatrix}$:
the eigenvalues of $A'$ are $\pm i\sqrt{\delta}$ and
$\det A' = \delta$).\footnote{Weyl's eigenvalue inequality is valid
only for Hermitian (more generally normal) matrices and does not apply
to $A'$; the counterexample displayed here refutes the eigenvalue and
determinant bounds it would otherwise suggest. Eigenvalue-location
statements for the nonnormal $A'$ would require normality or explicit
pseudospectral/eigenbasis-conditioning hypotheses, which we do not
impose.}
\end{enumerate}
\end{proposition}

\begin{proof}
(i)--(ii): For every $v\in\R^d$,
$Av=\mathbf{1}_d(\alpha^\top v)$, so
$\operatorname{ran}A\subseteq\operatorname{span}\{\mathbf{1}_d\}$.
If $\alpha=\bm0_d$, then $A=\bm0_{d\times d}$ and
$\operatorname{ran}A=\{\bm0_d\}$. If $\alpha\ne\bm0_d$, choose
$j$ with $\alpha_j\ne0$. For every $t\in\R$, the vector
$v=(t/\alpha_j)e_j$, where $e_j$ is the $j$th standard basis vector,
satisfies $Av=t\mathbf{1}_d$. Hence
\[
\operatorname{ran}A=
\begin{cases}
\operatorname{span}\{\mathbf{1}_d\},&\alpha\ne\bm0_d,\\
\{\bm0_d\},&\alpha=\bm0_d.
\end{cases}
\]
Because $d\ge1$, the all-ones vector is nonzero, giving the asserted
rank dichotomy. At $d=1$ a nonzero profile gives the full ambient space;
at $d\ge2$ the range is a proper subspace of $\R^d$, so $\det A=0$.
Also $A\mathbf{1}_d=s\mathbf{1}_d$; since $\rank A\le1$, the eigenvalue $0$
has geometric (hence algebraic) multiplicity at least $\dk - 1$, and the
trace $s$ accounts for the remaining root, giving
$\lambda^{\dk-1}(\lambda-s)$. If $s = 0$ and $\alpha \ne \bm0$ then
$A \ne \bm0_{\mathrm{mat}}$ while
$A^2 = \mathbf{1}(\alpha^\top\mathbf{1})\alpha^\top = \bm0_{\mathrm{mat}}$.
Its range still equals $\operatorname{span}\{\mathbf{1}_d\}$, as the
profile $\alpha=(1,-1)^{\mathsf T}$ at $d=2$ illustrates.
(iii): $A^2 = sA$ and induction. (iv): $\sigma_j(A') \le
\sigma_j(\mathbf{1}\alpha^\top) + \norm{\Delta}_2 \le 0 + \delta$ for
$j \ge 2$ by Weyl's inequality \emph{for singular values}
\cite[Cor.~7.3.5(a) and eq.~(7.3.13)]{hornjohnson}, since $\mathbf{1}\alpha^\top$ has rank
$\le 1$; then $\abs{\det A'} = \prod_j \sigma_j(A') \le \sigma_1(A')\,
\delta^{\dk-1}$ and $\sigma_1(A') \le \norm{\mathbf{1}\alpha^\top}_2 +
\delta$, with $\norm{\mathbf{1}\alpha^\top}_2 =
\norm{\mathbf{1}}_2\norm{\alpha}_2 = \sqrt{\dk}\,\norm{\alpha}_2$.
\end{proof}

\begin{remark}
\label{rem:rankone}
Proposition~\ref{prop:rankone} is consistent with the numerical
observations of \cite{gokden2025}: determinants of $A$ and $\Alm$
reported as $0$ at floating-point resolution on all heads, one dominant
real eigenvalue equal to the row sum, and row-wise repetition of values.
(Float-zero determinants of $64\times64$ matrices are underflow-level
evidence; determinants are ill-conditioned, and
Appendix~\ref{app:numerical} reports singular-value spectra and numerical
ranks under stated tolerances instead.) One exact consequence propagates
down the chain, and one does not. First, $WA = (W\mathbf{1})\,\alpha^\top$
remains of rank $\le 1$: the left matrix action of $W$ rescales the
identical rows by the components of $W\mathbf{1}$, so the argument of
$\iswiglu$ in \eqref{eq:metric} is a rank-one matrix plus the learned
bias. Second, however, an entry-wise nonlinear image of a
rank-one-plus-bias matrix is \emph{generically of full rank}: singularity
of $A$ does not imply singularity of $\Alm$. The singular-value audit of
Appendix~\ref{app:numerical} bears this out on the audited checkpoint:
$\Alm$ is generically of near-full numerical rank (median numerical rank
$62.5$ of $64$, range $1$--$64$) even though its float determinant
vanishes on every head: with singular values spanning many scales,
the determinant, a product of $\dk$ of them, underflows while the matrix
is far from rank one. The tolerance-based singularity statement
therefore attaches to the generator $A$, where
Appendix~\ref{app:numerical} confirms it in every audited instance;
this paper asserts no numerical low-rank property of $\Alm$, and the
vanishing float determinants of $\Alm$ reported in \cite{gokden2025}
carry no rank information. Additionally, the
\emph{same} $\alpha$ appearing on every head of a layer is explained by the
architecture: the metric learner $\Phi_{\mathrm{res}}$ is shared across the
heads of a layer and acts row-wise (Proposition~\ref{prop:rowfact}), so a
collapse of its row map to a constant produces one common profile for all
heads; this is the mechanism developed in Section~\ref{subsec:origin}. The
interpretation is important: at convergence the metric learner has collapsed
the input dependence of $A$ onto a single learned direction, which is the
algebraic mechanism behind operator invariance
(Section~\ref{sec:invariance}).
\end{remark}

To state the next result precisely, we make the internal structure of the
metric learner explicit. Each of its $N_{\mathrm{res}}$ residual units acts
on a single row $r \in \R^{\dk}$ of its matrix argument as
\begin{equation}
\label{eq:resunit}
u_j(r) \;=\; \LN\bigl(r + g_{j,2}(g_{j,1}(r))\bigr),
\qquad j = 1, \dots, N_{\mathrm{res}},
\end{equation}
where $g_{j,1}, g_{j,2} : \R^{\dk} \to \R^{\dk}$ are the two successive
complete SwiGLU blocks of unit $j$, each a gated map with hidden width
$\Adff$ (Definition~\ref{def:plga}); in the reference code each block is a
\texttt{GLUVariant} module with two $\dk \to \Adff$ linear maps multiplied
elementwise and a $\dk$-dimensional linear output map.\footnote{The
typing matters: the natural-looking alternative
$g_{j,1} : \R^{\dk} \to \R^{\Adff}$, $g_{j,2} : \R^{\Adff} \to \R^{\dk}$
would describe the two internal halves of a single feed-forward block,
whereas \texttt{ResLayerA} in the reference implementations applies two
full gated blocks in succession; the typing above matches the code.} When the dependence on the input sequence $x$
matters, we write $Q(x)$ for the query matrix produced on input $x$ and,
correspondingly, $D(x) = Q(x)^\top Q(x)$ and $A(x)$ for the resulting
density operator and metric generator.

\begin{proposition}[Row-factorization and permutation equivariance of the
metric learner]
\label{prop:rowfact}
In the reference implementations, every layer of $\Phi_{\mathrm{res}}$
(SwiGLU blocks acting on the last axis, residual sums, and LayerNorm over
the last axis) acts on each row of its $\dk\times\dk$ argument
independently and identically. Consequently there is a single map
$\varphi = u_{N_{\mathrm{res}}} \circ \cdots \circ u_1 :
\R^{\dk} \to \R^{\dk}$ (the \emph{row map}, the composition of the
residual units \eqref{eq:resunit}) such that
\[
\Phi_{\mathrm{res}}(M)_{i,:} \;=\; \varphi\bigl(M_{i,:}\bigr),
\qquad i = 1,\dots,\dk,
\]
for every input $M$. Hence:
\begin{enumerate}[label=(\roman*),leftmargin=2.2em]
\item (Equivariance.) $\Phi_{\mathrm{res}}(\Pi M) = \Pi\,
\Phi_{\mathrm{res}}(M)$ for every row permutation $\Pi$.
\item (Collapse criterion.) The generator $A(x)$ has identical rows
with one \emph{common} row value $\alpha^{\ast\top}$ shared across all
inputs $x$ in a set $\mathcal{X}$ if and only if $\varphi$ is
constant on the union of all rows of $\LN(D(x))$, $x \in \mathcal{X}$.
(Identical rows within each single input separately is the weaker
condition that $\varphi$ is constant on each input's row set; constancy
on the union is what the cross-input invariance of
Section~\ref{sec:invariance} requires.)
\item (Cross-head identity.) Since $\varphi$ is shared by all heads of a
layer, constancy of $\varphi$ forces $A^{(i)} = \mathbf{1}\,
\alpha^{\ast\top}$ with the \emph{same} $\alpha^\ast$ for every head $i$,
even though the heads' density operators differ.
\item (Lipschitz transfer.) $\norm{A(x) - A(x')}_F \le
\operatorname{Lip}(\varphi)\,\norm{\LN(D(x)) - \LN(D(x'))}_F$ whenever
$\varphi$ is $\operatorname{Lip}(\varphi)$-Lipschitz on the visited row
set.
\end{enumerate}
\end{proposition}

\begin{proof}
Each SwiGLU block is a map applied to the last axis, i.e.\ to each row
separately with shared weights; LayerNorm over the last axis normalizes
each row separately; residual addition is row-wise. A composition of
row-wise maps with shared parameters is row-wise with a single shared row
map $\varphi$. (i) follows since applying the same function to permuted
rows permutes the outputs. (ii): a common output value across all rows of
all inputs in $\mathcal{X}$ is exactly constancy of $\varphi$ on
the union of the row sets; per-input identical rows alone constrains
$\varphi$ only on each row set separately. (iii): apply (ii) per head and note $\varphi$ is the same map for
all heads (the parameter tensors $W,b_W,P,a,b_a$ downstream are per-head,
but $\Phi_{\mathrm{res}}$ is instantiated once per layer). (iv) is the
definition of a Lipschitz constant applied row-wise and summed in Frobenius
norm.
\end{proof}

Proposition~\ref{prop:rowfact}(iii) turns the observed cross-head identity
of $A$ \cite{gokden2025} from a curiosity into \emph{diagnostic evidence}:
the heads feed different density operators to the same $\varphi$, and a
locally constant $\varphi$ would force exactly the observed agreement.
Observed equality on the visited set does not by itself identify local
constancy (equal or symmetry-related head inputs, or pointwise agreement
without flatness, are alternatives); the direct discriminating
measurements, composite Jacobians and pairwise contraction ratios of the
trained $\varphi$ on visited rows, are reported in
Appendix~\ref{app:numerical} and support the locally-constant reading on
the audited checkpoint. The mechanism is analyzed in
Section~\ref{subsec:origin}.

\begin{proposition}[The ideal attention operator is Markov]
\label{prop:markov}
With the ideal masked softmax of Section~\ref{sec:prelim},
$\Elm$ is row-stochastic with exact causal support: $\Elm \ge 0$,
$\Elm \mathbf{1} = \mathbf{1}$, and $(\Elm)_{ij} = 0$ for $j > i$, with
$(\Elm)_{ij} > 0$ for $j \le i$. Consequently:
\begin{enumerate}[label=(\roman*),leftmargin=2.2em]
\item $\norm{\Elm}_\infty = 1$ and every eigenvalue of $\Elm$ lies in the
closed unit disk, with $1$ always attained
(lower-triangular stochastic structure gives spectrum equal to the diagonal
entries in the causal case);
\item row $i$ of $\Vlm = \Elm V$ is a convex combination (an
expectation) of value vectors of positions $j \le i$:
$\Vlm[i] = \E_{j\sim \Elm[i,\cdot]}\, V[j]$.
\end{enumerate}
\end{proposition}

\begin{proof}
Rows of the ideal masked softmax are probability vectors supported exactly
on the allowed set $\{j \le i\}$. (i) is standard for stochastic matrices
(Gershgorin or the sub-multiplicativity of $\norm{\cdot}_\infty$); a
causal (lower-triangular) matrix has its eigenvalues on the diagonal.
(ii) is the definition of matrix multiplication with stochastic rows.
\end{proof}

Proposition~\ref{prop:markov} is a statement about the \emph{support}
of the attention operator: row $i$ mixes values of positions $j \le i$
only. It does not by itself constrain how the scores and the operator
$\Glm$ entering row $i$ depend on the input, and in PLGA they depend on
\emph{all} supplied rows through the density operator
\eqref{eq:density}. The row-wise dependence property that
lower-triangular support does and does not deliver is stated precisely
in Section~\ref{subsec:online}; the deployed final-row generation
interface (Definition~\ref{def:pldr}) does not require the stronger
property.

\begin{remark}[Finite-mask implementation semantics]
\label{rem:finitemask}
The reference implementations realize the causal mask by adding a large
finite negative constant $-\mu$ to disallowed scores, with the
implementation-specific values as recorded in
Section~\ref{sec:prelim}: $\mu = 10^9$ in the native TensorFlow and
PyTorch code, $\mu = -\texttt{torch.finfo(dtype).min}$ in the Hugging
Face port via the Transformers causal-mask utility. For finite real
logits the exact
softmax assigns strictly positive mass to \emph{every} coordinate, so the
exact real-valued implemented operator is dense, not causally supported;
Proposition~\ref{prop:markov} does not literally hold for it. Under
ordinary score ranges the masked entries
($e^{z - \mu}$ relative to unmasked $e^{z}$) underflow to exact zero in
floating point for either value of $\mu$, so the float operator
coincides with the ideal masked
softmax; this is a numerical-realization statement, conditional on the
unmasked scores staying many orders of magnitude above $-\mu$, and it is
the semantics in which all implementation-level claims of this paper are
to be read. Statements proved for the ideal operator (here;
Proposition~\ref{prop:perturb}) are exact; their transfer to the
implementation carries this underflow caveat.
\end{remark}

Proposition~\ref{prop:markov} is elementary but load-bearing: PLGA's
inductive action is that of an input-adapted \emph{averaging operator}
on the token graph, the same structural role played by transfer
operators in ergodic theory and by graph Laplacian smoothers.

\subsection{PLGA as attention with a learned bilinear form}
\label{subsec:bilinear}

SDPA computes scores $QK^\top/\sqrt{\dk}$: the bilinear form comparing
queries with keys is the Euclidean one, $B(q,k) = q^\top I k$. PLGA computes
$Q \Glm K^\top/\sqrt{\dk}$: the form is $B_G(q,k) = q^\top \Glm k$, with
$\Glm$ produced by the nonlinear chain
\eqref{eq:density}--\eqref{eq:ec} from the input itself. Three regimes must
be distinguished:

\begin{enumerate}[leftmargin=2em]
\item \textbf{$\Glm \equiv I$ (frozen identity).} PLGA is exactly SDPA
(Theorem~\ref{thm:collapse}(i)).
\item \textbf{$\Glm = G^\ast$ frozen but arbitrary.} A generalized SDPA with
a learned constant bilinear operator. Without position-dependent transforms
this is linearly equivalent to SDPA by absorbing $G^\ast$ into the query
projection; with rotary embeddings the equivalence generically fails
(Proposition~\ref{prop:rope-commutant}), so even the frozen head realizes
score functions outside the SDPA-realizable family, in the conditional,
head-level sense counted by Corollary~\ref{cor:posgap}.
\item \textbf{$\Glm = \Glm(x)$ input-generated (training-time PLGA).} The
score map $Q \mapsto Q\,\Glm(Q)\,K^\top$ is generically nonlinear in the
input (degenerate parameter choices linearize it: $a = \bm0_{\mathrm{mat}}$ makes
$\Glm \equiv b_a$ constant and the head falls back to regime 2),
and gradients flow through $\Phi_{\mathrm{res}}, W, P, a$; this regime is
\emph{not} reducible to SDPA, which is the content of the
training/inference asymmetry (Theorem~\ref{thm:collapse}(iii)).
\end{enumerate}

\subsection{Multi-head structure as a local decomposition}
\label{subsec:multihead}

With $h$ heads, the model space $\R^{\dmodel}$ is split as an internal direct
sum $\bigoplus_{i=1}^h \R^{\dk}$ and each head learns its own tuple
$(A^{(i)}, \Alm^{(i)}, \Ap^{(i)}, \Glm^{(i)}, \Elm^{(i)})$; outputs are
concatenated and mixed by a linear map. Formally, the deductive state of one
decoder layer is the block-diagonal operator
\begin{equation}
\label{eq:blockdiag}
\Glm^{\mathrm{layer}} \;=\; \bigoplus_{i=1}^{h} \Glm^{(i)}
\;\in\; \R^{\dmodel\times\dmodel}
\quad\text{(in the head-aligned basis)},
\end{equation}
so the layer's score computation is assembled from mutually
non-interacting local components, each acting on its own subspace
\cite{gokden2021}. The output projection $W_O$ then mixes the head
\emph{outputs} linearly, after the scores are formed; it does not act on,
or conjugate, the per-head score operators themselves. The division
of parameters respects this structure: the metric-learner network
$\Phi_{\mathrm{res}}$ is shared across the heads of a layer
(Proposition~\ref{prop:rowfact}), while the five tensors
$W, b_W, P, a, b_a$ are per-head: heads are differentiated \emph{after}
the common metric generator, not before.

\section{The PLDR-LLM Architecture}
\label{sec:pldr}

\subsection{Rotary position embeddings and a commutant characterization}
\label{subsec:rope}

PLDR-LLM applies rotary position embeddings (RoPE) \cite{su2021} to queries
and keys after head-splitting. RoPE at position $n$ is the orthogonal
block-diagonal rotation
\begin{equation}
R_n \;=\; \bigoplus_{j=1}^{\dk/2} \operatorname{Rot}(n\theta_j),
\qquad
\operatorname{Rot}(\phi) = \begin{pmatrix}\cos\phi & -\sin\phi\\
\sin\phi & \cos\phi\end{pmatrix},
\qquad \theta_j = \Theta^{-2(j-1)/\dk},
\end{equation}
with $\Theta = 10^4$ in the reference configuration, so $n \mapsto R_n$ is a
group homomorphism $\Z \to SO(\dk)$ into a maximal
torus $T \cong (S^1)^{\dk/2}$ (the \emph{ambient RoPE torus}), with
$R_n^\top = R_{-n}$ and
$R_nR_m = R_{n+m}$. Scores become
\begin{equation}
\label{eq:rope-score}
E_{nm} \;=\; \frac{(R_n q_n)^\top\, \Glm\, (R_m k_m)}{\sqrt{\dk}}
\;=\; \frac{q_n^\top\, R_{-n}\, \Glm\, R_m\, k_m}{\sqrt{\dk}} .
\end{equation}
For SDPA ($\Glm = I$) one has $R_{-n} I R_m = R_{m-n}$: scores depend on
positions only through the offset $m-n$ (the celebrated relative-position
property). With a learned metric the situation is characterized exactly:

\begin{proposition}[Commutant characterization of relative-position invariance]
\label{prop:rope-commutant}
Assume the rotation angles $\theta_1,\dots,\theta_{\dk/2}$ satisfy the
\emph{nonresonance conditions}
\[
\theta_a \not\equiv \pm\,\theta_b \pmod{2\pi} \quad (a \ne b),
\qquad
\theta_a \not\equiv 0 \pmod{\pi} \quad (\text{all } a),
\]
i.e.\ the eigenvalues $e^{\pm i\theta_a}$ of $R_1$ are $\dk$ distinct
non-real numbers. The standard frequencies
$\theta_j = \Theta^{-2(j-1)/\dk}$ with $\Theta = 10^4$ satisfy these
conditions: they lie in $(0,1] \subset (0,\pi)$ and are strictly
decreasing, hence pairwise distinct with no pair summing to a multiple of
$2\pi$.\footnote{Nonresonance is the correct hypothesis here. The
natural-looking alternative, rational independence of the $\theta_a$
with a density argument for the cyclic orbit in the torus, is both
insufficient for density (which requires rational independence of
$1, \theta_1/2\pi, \dots$; e.g.\ $\theta = \pi$ is irrational yet has a
finite orbit) and false for the standard frequencies, which obey the
exact rational relations $\theta_{j + \dk/8} = \theta_j/10$ for base
$10^4$. The nonresonance route avoids density altogether.} Then the map
$(n,m) \mapsto R_{-n}\,G\,R_m$ depends only on $m-n$ (for all offsets) if and
only if $G$ commutes with every rotation $R_n$ (equivalently, with
the ambient RoPE torus $T$), which holds if and only if $G$
is block-diagonal with $2\times2$ blocks of the form
$\begin{psmallmatrix} c_j & -s_j\\ s_j & c_j\end{psmallmatrix}$, i.e.\
$G \in \bigoplus_j \{c_j I_2 + s_j J_2\} \cong \mathbb{C}^{\dk/2}$, where
$I_2$ is the $2\times2$ identity and
$J_2 = \begin{psmallmatrix} 0 & -1\\ 1 & 0\end{psmallmatrix}$ is the
rotation by $\pi/2$ (the complex structure of the plane): $G$ acts as the
complex scalar $c_j + i s_j$ on the $j$-th rotation plane. For generic learned $\Glm$ this
fails, and PLGA scores carry \emph{absolute} positional information through
the conjugation orbit $R_{-n}\Glm R_{n}$.
\end{proposition}

\begin{proof}
Offset dependence for all $(n,m)$ is equivalent to
$R_{-n} G R_m = R_{-n'} G R_{m'}$ whenever $m-n = m'-n'$; taking $n'=0$,
$m'=m-n$ gives $R_{-n} G R_n = G$ for all $n$, i.e.\ $G$ commutes with the
cyclic group generated by $R_1$, which holds iff $G$ commutes with the
single matrix $R_1$ (every element is a power of $R_1$). $R_1$ is real
semisimple; after complexification it is diagonal with eigenvalues
$e^{\pm i\theta_a}$, which under the nonresonance hypothesis are $\dk$
\emph{distinct} numbers. The commutant of a diagonalizable matrix with
distinct eigenvalues is the algebra of matrices diagonal in the same
eigenbasis; regrouping the conjugate eigenpairs into their real
$2\times2$ planes, the real matrices in this commutant are exactly the
block-diagonal matrices whose $j$-th block commutes with
$\operatorname{Rot}(\theta_j)$, i.e.\ (for $\theta_j \notin \pi\Z$) lies
in $\{cI_2 + sJ_2\}$, the algebra generated by the rotation itself
(isomorphic to $\mathbb{C}$). Since every $R_t$ is block-diagonal with
blocks in these algebras, any such $G$ commutes with every $R_t$, giving
offset dependence; density of the orbit in the torus is not needed.
\end{proof}

\begin{remark}
\label{rem:rope-twice}
The mathematical content of Proposition~\ref{prop:rope-commutant}
(the commutant of a semisimple matrix with distinct eigenvalues) is
classical, and the commutant's organizing role for rotary attention
is established in the RoPE literature: commuting generators are the
design principle of the generalized-rotation line
\cite{su2021,ostmeier2024,yu2025}; the gauge-symmetry
characterization of \cite{wangwang2025} identifies the invertible
query/key reparameterizations preserving rotary attention with the
rotation commutant; and the functional-equivalence analysis of
\cite{tran2026} ties the symmetries of rotary attention to the same
commutant. What the proposition contributes is the PLGA-specific
formulation: the object constrained here is an \emph{inserted
bilinear operator} $G$ between rotated queries and keys (not a
reparameterization of the projections, and not required to be
invertible), with the
absolutely-position-aware complement quantified per head
(Corollary~\ref{cor:posgap}). It says the
learned operator interpolates between purely relative (commutant-valued $G$) and
absolutely position-aware attention geometries, and quantifies exactly which
degrees of freedom of $G$ break relative invariance. It also suggests a diagnostic: the relative off-commutant Frobenius
residual is a checkpoint-scoped occupancy measurement of the
absolute-position-sensitive operator sector
(Appendix~\ref{subsec:commutant-audit}). Establishing its predictive
contribution requires projecting a trained $\Glm$ onto the commutant and
evaluating the behavioral change, with replication across checkpoints
and scales. A second, verified,
implementation fact makes RoPE doubly relevant to PLGA: rotation is applied
\emph{before the density operator is formed}, so RoPE enters the
architecture twice: through the bilinear scores (this subsection) and
through the input of the metric learner, where the position-dependent
conjugation acts as an ergodic ``twirl'' analyzed in
Lemma~\ref{lem:twirl}.
\end{remark}

\subsection{The full decoder map}

\begin{definition}[PLDR-LLM \cite{gokden2024,gokden2025}]
\label{def:pldr}
Fix depth $L$, heads $h$, width $\dmodel = h\dk$, feed-forward width
$d_{f\!f}$, metric-learner width $\Adff$, and vocabulary $\mathcal{V}$. The
PLDR-LLM is the length-indexed family of maps
$F_\theta : \bigcup_{S \le S_{\max}} \mathcal{V}^{S} \to
\bigcup_{S \le S_{\max}} \Delta(\mathcal{V})^{S}$
(row-wise probability simplices), with one output row per input row:
an input of length $S$ is mapped into $\Delta(\mathcal{V})^{S}$, so
$\abs{F_\theta(x)} = \abs{x}$ always (this length-preservation
statement is carried as an explicit predicate in the Lean
development, Appendix~\ref{app:lean}). The computation is as follows. With
$X^{(0)} = \LN\bigl(\sqrt{\dmodel}\,\cdot\,\iota(x)\bigr)$,
for $\ell = 1,\dots,L$ and heads $i = 1,\dots,h$:
\begin{align}
Q^{(\ell,i)} &= X^{(\ell-1)} W_Q^{(\ell,i)} + \mathbf{1}\,b_Q^{(\ell,i)\top},
\qquad
K^{(\ell,i)} = X^{(\ell-1)} W_K^{(\ell,i)} + \mathbf{1}\,b_K^{(\ell,i)\top},
\label{eq:qkv}\\
V^{(\ell,i)} &= X^{(\ell-1)} W_V^{(\ell,i)} + \mathbf{1}\,b_V^{(\ell,i)\top},
\qquad
\widetilde Q^{(\ell,i)} = \RoPE(Q^{(\ell,i)}), \qquad
\widetilde K^{(\ell,i)} = \RoPE(K^{(\ell,i)}),\\
\Vlm^{(\ell,i)} &= \operatorname{PLGA}^{(\ell,i)}
\bigl(\widetilde Q^{(\ell,i)}, \widetilde K^{(\ell,i)}, V^{(\ell,i)}\bigr),
\label{eq:plga-call}\\
U^{(\ell)} &= \LN\Bigl( X^{(\ell-1)} +
\bigl[\Vlm^{(\ell,1)} \Vert \cdots \Vert \Vlm^{(\ell,h)}\bigr] W_O^{(\ell)}
+ \mathbf{1}\,b_O^{(\ell)\top}\Bigr),\\
X^{(\ell)} &= \LN\Bigl( U^{(\ell)} +
\operatorname{SwiGLU\text{-}FFN}^{(\ell)}\bigl(U^{(\ell)}\bigr)\Bigr),
\end{align}
and finally $F_\theta(x) = \softmax\bigl(X^{(L)} W_{\mathrm{vocab}} +
\mathbf{1}\,b_{\mathrm{vocab}}^\top\bigr)$, with $\mathbf{1} \in \R^S$
the all-ones column, so each bias adds its row to every position.
In \eqref{eq:plga-call} each head applies the PLGA operator per
Definition~\ref{def:plga} with
$\widetilde D^{(\ell,i)} = \widetilde Q^{(\ell,i)\top} \widetilde
Q^{(\ell,i)}$ in place of $D_Q$: the density operator is built from the
\emph{rotated} query $\widetilde Q$, so positional phases enter the metric
learner (Remark~\ref{rem:rope-twice}). The displayed affine terms are
those of the released configuration (biases enabled on all attention
projections and on the unembedding); the unembedding is not tied to
the embedding $\iota$, and the gated feed-forward and metric-learner
blocks carry their own internal biases. The deployed generative interface is \emph{final-row online
generation}: at step $t$ the model is invoked on exactly the known
prefix $x_{1:t}$, so the call has $S = t$ rows (one-based; in
zero-based code indexing the rows are $0, \dots, S{-}1$ and the
selected row is $S{-}1$), and only the final output row is consumed,
\begin{equation}
\label{eq:onlinecontract}
p_\theta(\,\cdot \mid x_{1:t}) \;=\; F_\theta(x_{1:t})_{t},
\end{equation}
the final row of the call on the prefix itself.
Every tensor of the call, including the density operator
$\widetilde D = \widetilde Q^\top \widetilde Q$ and the deductive
family derived from it, is a function of $x_{1:t}$ alone, so the
generator defined by \eqref{eq:onlinecontract} is causal: no token
beyond position $t$ enters the conditional that predicts $x_{t+1}$.
The causal mask in \eqref{eq:attn} constrains the \emph{score support}
of every row; it does not make internal rows $r < S$ of a longer call
functions of their own prefixes, and they are not so interpreted
(Section~\ref{subsec:online}).\footnote{One initialization detail
distinguishes the released model generations: in the v510 code accompanying
\cite{gokden2025} the value projection $W_V$ keeps PyTorch's default uniform
initialization while $W_Q, W_K$ are Glorot-uniform with zero bias (a
duplicated initializer call on $W_K$; the detail is shared by the base
v510 model and its DAG and ablation variants in that repository); the code
accompanying \cite{gokden2026} initializes $W_V$ identically to $W_Q, W_K$.
This shifts the near-critical $(\eta_{\max}, T_w)$ region slightly
\cite{gokden2026}.}
\end{definition}

The deductive output of the full model is the indexed family
\begin{equation}
\mathcal{D}(x;\theta) \;=\;
\bigl\{\, A^{(\ell,i)},\ \Alm^{(\ell,i)},\ \Ap^{(\ell,i)},\ \Glm^{(\ell,i)}
\,\bigr\}_{\ell \le L,\, i \le h},
\end{equation}
a point in $\R^{4Lh\dk^2}$, computed alongside the inductive output.

\subsection{Training objective and the DAG regularizer}
\label{subsec:dag}

Pretraining minimizes a \emph{blockwise (global-context)
cross-entropy}. For a training block $(x_0, \dots, x_{T-1})$ the
reference implementations run \emph{one} forward pass on the shifted
input $(x_0, \dots, x_{T-2})$ and apply cross-entropy at every aligned
output row:
\begin{equation}
\label{eq:blockloss}
\mathcal{L}_{\mathrm{block}}(\theta)
\;=\; -\,\E \sum_{t=0}^{T-2}
\log q_{\theta,t}\bigl(x_{t+1};\, x_{0:T-2}\bigr),
\end{equation}
where $q_{\theta,t}$ denotes the row-$t$ conditional of the single
full-block call (row $t$ of $F_\theta(x_{0:T-2})$, evaluated at
$x_{t+1}$). As implemented in both reference trainers, the objective
is the masked, normalized form
\[
\widehat{\mathcal{L}}_{\mathrm{block}}(\theta)
\;=\; -\,\frac{\displaystyle\sum_{b,t} m_{bt}\,
\log q_{\theta,t}\bigl(x_{b,t+1};\, x_{b,0:T-2}\bigr)}
{\displaystyle\sum_{b,t} m_{bt}},
\qquad m_{bt} = \mathbf{1}\{x_{b,t+1} \ne 0\},
\]
with the sums over the batch $b \le B$ and the aligned rows $t$:
padding labels (token id $0$) are excluded, and the normalizer is
the batch's nonpadding-label count. For fixed-length un-padded
blocks this equals \eqref{eq:blockloss} up to the positive constant
$1/(T{-}1)$, reading the expectation as the batch mean; in padded
batches the two differ, and the mask removes padded \emph{labels}
from the loss while padded \emph{rows} still enter the forward pass
and its query Gram (see the padding paragraph of
Section~\ref{subsec:online}).
The final summand ($t = T{-}2$) is exactly the deployed
online conditional $-\log p_\theta(x_{T-1} \mid x_{0:T-2})$ of
\eqref{eq:onlinecontract}, since the held-out target is absent from
the input; the block's Gram aggregates only tokens that precede it.
The first summand ($t = 0$) is also structurally guaranteed: under
the ideal causal mask row $0$ has a single admissible key, so every
attention layer assigns it the weight vector $(1)$ independently of
the operator $\Glm$ and of every later supplied row, and all other
sublayers act row-wise; by induction the row-$0$ output depends
only on $x_0$, and the summand equals the online conditional
$-\log p_\theta(x_1 \mid x_0)$ (exactly so in the implemented
finite-mask arithmetic whenever the masked exponentials underflow
to zero, as on the audited calls; Appendix~\ref{app:numerical}
reports the structural insensitivity of prefix position $0$ on both
released checkpoints).
For the intermediate rows $1 \le t \le T{-}3$ the summands are
auxiliary historical-row predictions
conditioned on the \emph{entire} supplied block through the deductive
operator, and the conditioning must be stated at its sharpest: for
each such $t$ the label $x_{t+1}$ is \emph{itself among the
supplied rows}: it enters the query Gram, hence the operator
$\Glm$, used to produce the row-$t$ prediction, as do all tokens
later than the target. The intermediate summands are therefore
\emph{target-exposed} auxiliary scores, generically target-dependent
absent operator collapse or another degeneracy of the score, value,
or output path (Remark~\ref{rem:cacheable-converse} exhibits such
degeneracies; Appendix~\ref{app:numerical} measures a collapsed
checkpoint on which historical rows do not move), and they are not
next-token log probabilities as a protocol; for $T \ge 3$ exactly
two of the $T{-}1$ summands
per block, the first and the final, are structurally guaranteed to
equal deployed conditionals (for $T = 2$ the single summand is
both), while the remaining $T{-}3$ need not be (though every
summand updates shared parameters), and a low value of
$\mathcal{L}_{\mathrm{block}}$ cannot be read as a low
autoregressive perplexity without a sequential measurement
(Appendix~\ref{app:numerical} reports one on held-out text).
Historical-row prefix consistency
(Definition~\ref{def:prefixconsistent}) at the visited inputs
guarantees that the summands coincide with the chain-rule factors
$-\log p_\theta(x_{t+1} \mid x_{0:t})$; the property is equivalent
to equality of the full row conditionals there (equality of every
possible next-token factor), and it is strictly stronger than
equality of the realized quantities alone, since probability mass
can move among untargeted tokens without changing a realized
factor, and summed losses can agree through cancellation across
positions. The global Gram does not supply prefix consistency in
general (Section~\ref{subsec:online}; measured on the released
checkpoints in Appendix~\ref{app:numerical}).
$\mathcal{L}_{\mathrm{block}}$ is a
coherent surrogate in the limited sense of a well-defined
differentiable criterion that trains the causal final-row generator
while letting all past tokens interact through the operator being
learned; it is not a proper scoring rule for the chain-rule language
distribution, and a sufficiently input-sensitive metric learner
could in principle exploit the target exposure during training,
precisely when the learner is meant to be input-sensitive
(Section~\ref{sec:discussion} lists the along-training measurement
this motivates). In this paper the terms \emph{autoregressive NLL}
and \emph{perplexity} are reserved for the sequential final-row
quantity
$-\sum_t \log p_\theta(x_{t+1} \mid x_{0:t})$, and every likelihood or
benchmark number is labeled with the protocol that produced it
(Appendices \ref{app:repos} and \ref{app:numerical}).

The objective is optionally augmented by the \emph{DAG regularizer} of
the deductive outputs
\cite{gokden2024}, built on the NOTEARS characterization
\cite{zheng2018}. As implemented in the reference code, the per-tensor DAG
loss averages the log of the \emph{normalized} heat trace over the batch
elements $b \le B$ ($B$ the batch size), the layers $\ell \le L$, and the
heads $i \le h$, with $M^{(b,\ell,i)}$ denoting the instance of the tensor
$M$ inferred for batch element $b$ at layer $\ell$, head $i$:
\begin{equation}
\label{eq:dagloss}
\begin{split}
D_L(M) &= \frac{1}{BLh}\sum_{b,\ell,i}
\,\Bigl\lvert\, \log\Bigl( \tfrac{1}{\dk}\,\tr\,
e^{\,M^{(b,\ell,i)} \Had M^{(b,\ell,i)}} \Bigr)
\Bigr\rvert,\\
\mathcal{L} &= \mathcal{L}_{\mathrm{block}}
+ \lambda_1 D_L(\Alm) + \lambda_2 D_L(\Ap) + \lambda_3 D_L(\Glm).
\end{split}
\end{equation}
By Theorem~\ref{thm:dag} below, $\tr e^{M\Had M} \ge \dk$ always, so the
absolute value in \eqref{eq:dagloss} is analytically redundant (it guards
the numerics) and $D_L(M) \ge 0$, with equality iff the support graph of
each summand is acyclic. For the two entrywise-positive tensors $\Alm$
and $\Ap$ that equality is unattainable
(Remark~\ref{rem:dagobstruction}); the mixed tensor $\Glm = a\Ap + b_a$
is not sign-constrained, and no positive floor is asserted for
$D_L(\Glm)$, whose value is an empirical matter (measured on a released
checkpoint in Appendix~\ref{app:numerical}).

\begin{theorem}[Walk-counting characterization of acyclicity
\cite{zheng2018}]
\label{thm:dag}
For $M \in \R^{d\times d}$, let $h(M) = \tr e^{M\Had M} - d$. Then
$h(M) \ge 0$ always, and $h(M) = 0$ if and only if the weighted directed
graph with adjacency $M$ (edge $i\to j$ iff $M_{ij} \ne 0$) has no directed
cycle.
\end{theorem}

\begin{proof}
Let $N = M \Had M$, which has entries $N_{ij} = M_{ij}^2 \ge 0$. Then
$(N^k)_{ii} = \sum M_{i j_1}^2 M_{j_1 j_2}^2 \cdots M_{j_{k-1} i}^2$ sums
non-negative weights over closed walks of length $k$ through $i$. Hence
\begin{equation}
\label{eq:heattrace}
\tr e^{N} \;=\; d + \sum_{k\ge1} \frac{\tr N^k}{k!}
\;=\; d + \sum_{k \ge 1} \frac{1}{k!}
\bigl(\text{total squared-weight of closed $k$-walks}\bigr) \;\ge\; d,
\end{equation}
with equality iff every term vanishes, i.e.\ iff there is no closed walk of
any length, i.e.\ iff the graph is acyclic. (A directed cycle yields a closed
walk and conversely any closed walk contains a cycle.)
\end{proof}

\begin{remark}[Positivity obstruction: the deductive tensors are never
exactly acyclic]
\label{rem:dagobstruction}
Theorem~\ref{thm:dag} is a generic theorem; its equality case cannot be
attained by the tensors the loss is applied to. By
Proposition~\ref{prop:positivity}, $(\Alm)_{ij} \ge \epsilon = 10^{-9}$
for \emph{all} $(i,j)$, and positive numbers raised to real powers remain
positive, so $\Ap$ is also entry-wise positive: the support graphs of
$\Alm$ and $\Ap$ are complete (in particular every vertex carries a
self-loop) and are never acyclic. Quantitatively,
\[
h(\Alm) \;\ge\; \tr(\Alm \Had \Alm) \;\ge\; \dk\,\epsilon^2 \;>\; 0 ,
\]
since $\tr e^N \ge \tr(I + N) = \dk + \tr N$ for entry-wise nonnegative
$N$. Two floors must be kept apart, one per quantity:
$\dk\,\epsilon^2 \approx 6.4\times10^{-17}$ is the floor of the
NOTEARS quantity $h$, \emph{before} normalization and logarithm. The
implemented per-instance summand of \eqref{eq:dagloss} is the
normalized logarithm, whose own floor follows by monotonicity:
\[
\log\!\Bigl(\frac{\tr e^{N}}{\dk}\Bigr)
\;=\; \log\!\Bigl(1 + \frac{h(\Alm)}{\dk}\Bigr)
\;\ge\; \log\bigl(1+\epsilon^2\bigr) \;\approx\; \epsilon^2
\;=\; 10^{-18} .
\]
A DAG-loss value reported as $0$ is a finite-precision observation
of $\log(\tr e^{M\Had M}/\dk)$, not exact acyclicity. Both floors
lie below float64 resolution of a naive $\log(1+x)$ evaluation at
an argument near $1$. Rounding the normalized trace to one, or
cancellation in an equivalent logarithmic evaluation, can erase the
small residual without underflow of the positive increment itself.
The exact-real floor does not identify the operation responsible for
a computed zero; Appendix~\ref{app:numerical} reports the observed zeros
for $\Alm$. The obstruction concerns $\Alm$ and $\Ap$ only: $\Glm$ is
not entrywise positive by construction, and its measured DAG loss on the
audited checkpoint is reported in Appendix~\ref{app:numerical} without
an asserted floor. The
correct reading of \eqref{eq:dagloss} is as a \emph{cycle-content
penalty}: it applies soft pressure toward small closed-walk weight, and
its minimum over the reachable set is positive for the positive tensors
(the vocabulary is finite and $S \le S_{\max}$, so for a fixed
checkpoint the admissible input set, hence each tensor image, is
finite and the minimum is attained; for $\Alm$ the architectural
floor gives, in addition, a uniform input-independent lower bound).
We accordingly avoid
describing the regularizer as driving the tensors ``toward a causal DAG
structure'': no causal variables are defined over feature coordinates,
and exact acyclicity is unreachable. A regularizer for which
Theorem~\ref{thm:dag}'s equality case is meaningful would have to act on
an explicitly zero-diagonal, sparsifiable adjacency (e.g.\ a thresholded
or gated matrix), an architecture variant not implemented in the
reference code.
\end{remark}

\begin{remark}[Spectral--geometric duality of the DAG loss]
\label{rem:tracedual}
By Schur triangularization, for any square $X$,
$\tr e^{X} = \sum_i e^{\lambda_i(X)}$ with algebraic multiplicities. Thus
each summand of \eqref{eq:dagloss} has two expansions:
\[
\underbrace{\frac{1}{\dk}\sum_i e^{\lambda_i(M\Had M)}}_{\text{spectral
side: mean heat trace}}
\;=\;
\underbrace{\frac{1}{\dk}\sum_{k\ge0} \frac{1}{k!}\,\#\{\text{weighted closed
$k$-walks}\}}_{\text{geometric side: cycle content}} .
\]
This has the combinatorial shape of a trace formula: a spectral sum (a
heat-kernel trace) equated with a geometric sum over closed walks.
Regularizing the DAG loss suppresses the
cycle content of the deductive graphs toward its (positive) floor
(Remark~\ref{rem:dagobstruction}); empirically, regularization recovers
the DAG loss of the deductive outputs from the overflow condition
(unregularized, the DAG losses of $\Ap$ and $\Glm$ diverge beyond
floating-point range) and can improve benchmark scores, under the
published one-pass block protocol
(Appendix~\ref{app:repos}) \cite{gokden2024}.
\end{remark}

\subsection{Inference: KV-cache and G-cache}
\label{subsec:cache}

At inference, PLDR-LLM admits two nested caches, implemented as follows in
the reference code \cite{gokden2025}:
\begin{itemize}[leftmargin=2em]
\item \textbf{KV-cache} $= (K_{\mathrm{cached}}, V_{\mathrm{cached}},
A_{\mathrm{cached}})$: after the prompt pass, the rotated keys, the values,
\emph{and the metric generator $A$} are stored; for each generated token,
the new $k, v$ rows are rotated (with tracked positions) and appended, while
$A$ is \emph{not} recomputed; the prompt-inferred $A$ is reused.
\item \textbf{G-cache} $= (\Alm, \Glm)$: additionally, the outputs of
\eqref{eq:metric}--\eqref{eq:ec} are stored and the parameter maps
$A \mapsto \Alm \mapsto \Ap \mapsto \Glm$ are skipped.
\end{itemize}
New tokens are then scored by
\begin{equation}
\label{eq:gcache}
E = \frac{\widetilde q\;\Glm^{\mathrm{cached}}\;K_{\mathrm{cached}}^\top}
{\sqrt{\dk}}, \qquad
\Elm = \softmax(E), \qquad v_{\mathrm{next}} = \Elm V_{\mathrm{cached}} .
\end{equation}
Three inference semantics must be kept separate, and this paper uses
them under the following fixed names.
\begin{itemize}[leftmargin=2em]
\item \emph{Exact prefix recomputation}: the uncached online loop of
\eqref{eq:onlinecontract}; every step is a fresh $S = t$ call on
the grown prefix, and every tensor (including $A$ and $\Glm$) is
recomputed. This is the reference semantics; it defines the model's
conditionals.
\item \emph{Prompt-frozen conditional generation}: what KV-cache and
G-cache literally compute; $A$ (hence $\Alm, \Glm$) is frozen at the
prompt and reused for every generated token. This is a
\emph{definition} of the cached conditional, not an approximation
claim: given the cached $A$, the maps to $\Alm$ and $\Glm$ are
deterministic functions of $A$ and the fixed parameters, so
\emph{G-cache is exact relative to KV-cache}; it saves compute and
introduces no additional approximation beyond the freeze itself.
\item \emph{Empirical freeze after collapse}: the measured statement
that the first two semantics coincide on a given checkpoint (the
invariance property of Section~\ref{sec:invariance}). On the audited
checkpoint the recomputed $\Glm$ equals the cached value bitwise at
every compared step and the cached-versus-recomputed logit deviation
stays below the decoding margins (Appendix~\ref{app:numerical}),
which is what makes the cached conditional a faithful implementation
of the reference semantics there.
\end{itemize}
Empirically, caching yields a $\sim 3\times$ speedup
\cite{gokden2025}, with the reported aggregate deductive-output
statistics (cross-head RMSE values and maximum determinant magnitudes)
agreeing to $15$ printed decimal digits on the tested prompt and greedy
decoding run \cite{gokden2025}; agreement of these aggregates is
evidence of, but not identical to, uniform element-wise equality of all
tensors on all prompts. The published benchmark evaluations of
\cite{gokden2025} score answer candidates by \emph{one-pass block
scoring} (one full call on context plus candidate, summing candidate
log-probabilities; Appendix~\ref{app:repos} pins the evaluation
wrapper), a protocol in which no generation-time cache participates by
construction; the sequential-versus-block score comparison on the
released checkpoints is reported in Appendix~\ref{app:numerical}.

\subsection{The online generation contract and historical-row prefix
consistency}
\label{subsec:online}

The global density operator makes PLDR-LLM's row-wise dependence
structure different from SDPA's in one specific, easily misread way.
This subsection states the deployed contract and the stronger property
it deliberately does not supply, and fixes the terminology used for
training and evaluation semantics throughout the paper.

The deployed interface is \eqref{eq:onlinecontract}: supply exactly
the known prefix $x_{1:t}$ ($S = t$), read the final row. Two
properties must be distinguished.

\begin{remark}[Online causality]
\label{rem:onlinecausal}
The step-$t$ online conditional \eqref{eq:onlinecontract} is a
function of $x_{1:t}$ alone: if two token sequences agree on their
first $t$ tokens, the step-$t$ calls receive identical inputs and
produce identical conditionals, for \emph{any} decoder map. In
particular every tensor entering the call ($\widetilde D$, the
deductive family, the scores, the mask) is measurable with respect
to the information available at step $t$, so sequential generation is
causal with \emph{no} invariance, collapse, or contraction assumption.
When the emitted token is appended, the next call computes a new
conditional from $x_{1:t+1}$, lawfully using $\Glm(x_{1:t+1}) \ne
\Glm(x_{1:t})$: the new computation cannot alter the earlier
conditional that produced $x_{t+1}$. The statement is formalized at
the wrapper level in the Lean development
(Appendix~\ref{app:lean}).
\end{remark}

\begin{definition}[Historical-row prefix consistency]
\label{def:prefixconsistent}
A decoder map $F_\theta$ is \emph{historical-row prefix consistent} if
for every input $x_{1:S}$ and every row $r$ with $1 \le r \le S$
(one-based, as in \eqref{eq:onlinecontract}),
\begin{equation}
\label{eq:strongprefix}
F_\theta(x_{1:S})_r \;=\; F_\theta(x_{1:r})_r ,
\end{equation}
i.e.\ every row of one parallel call equals the final row of the call
on its own prefix.
\end{definition}

Property \eqref{eq:strongprefix} is what permits a single forward pass
to be read simultaneously as all the sequential conditionals; it is
the bridge between the blockwise objective \eqref{eq:blockloss} and
the chain-rule factorization, and it is exactly what one-pass
likelihood evaluation assumes.

\begin{remark}[Row-local maps have it; the global Gram does not]
\label{rem:prefixmech}
(i) Call a decoder \emph{row-local} if row $r$ of every sublayer
output depends only on input rows $\le r$. Causally masked SDPA is
row-local: attention row $r$ is a convex combination of value rows
$\le r$ with weights computed from query row $r$ and key rows $\le r$,
and LayerNorm, the feed-forward block, and all projections act
row-wise; by induction over sublayers, a row-local decoder satisfies
\eqref{eq:strongprefix} in exact arithmetic. (The implication
row-local $\Rightarrow$ prefix-consistent, and the coincidence of the
block reading of a historical row with the online final-row output for
such maps, are kernel-checked in the Lean development,
Appendix~\ref{app:lean}.)
(ii) PLGA is not row-local: the density operator \eqref{eq:density}
contracts the \emph{entire} supplied sequence axis, so the operator
$\Glm$ shared by all rows depends on all rows. At a single PLGA
layer, perturbing its current-layer input row at position $s > r$
while holding the other current-layer rows fixed leaves query row
$r$ and key rows $\le r$ unchanged and moves the Gram by
$\delta \widetilde D = \delta\widetilde q_s \widetilde q_s^\top +
\widetilde q_s\, \delta\widetilde q_s^\top +
O(\|\delta\widetilde q_s\|^2)$; the deductive chain $\Psi$ is
smooth on a neighborhood of the visited Grams (every LayerNorm
carries the variance floor $\varepsilon > 0$ and every elementwise
power receives a strictly positive base), so $D\Psi$ is locally
Lipschitz (this justifies the quadratic remainder below), and the
direct current-layer contribution is
\begin{equation}
\label{eq:histderiv}
\delta z_{rj} \;=\; \frac{1}{\sqrt{\dk}}\,
\widetilde q_r^\top\, D\Psi[\widetilde D](\delta \widetilde D)\,
\widetilde k_j \;+\; O(\|\delta\widetilde D\|^2),
\qquad j \le r,
\end{equation}
which, for a historical row with at least two allowed keys
($r \ge 2$ in the one-based indexing of
Definition~\ref{def:prefixconsistent}), is generically nonzero and
generically not a common shift of the allowed scores, so such
historical softmax rows and outputs move; row $1$ has the singleton
admissible key set, any perturbation of its single score is a
common shift on a one-point simplex, and its output cannot move
through this channel (measured as prefix position $0$ in
Appendix~\ref{app:numerical}). (A two-token instance violating
\eqref{eq:strongprefix} for a minimal abstract global-aggregation
decoder is kernel-checked in Lean, Appendix~\ref{app:lean}.)
In the full
multilayer decoder the same suffix perturbation also changes
earlier hidden rows through the upstream layers' global Grams, so
the end-to-end derivative additionally contains
$(\delta\widetilde q_r)^\top \Glm \widetilde k_j$ and
$\widetilde q_r^\top \Glm\, \delta\widetilde k_j$ and
induced-state terms: further channels by which historical rows
move, none needed for the conclusion. This is the
\emph{intended} interaction of all currently known tokens through the
learned operator, not an information-flow error: under the online
contract every token entering $\widetilde D$ is already in the
conditioning set. (iii) Whether \eqref{eq:strongprefix} holds
\emph{empirically} is a checkpoint property tied to the collapse
phenomenon of Section~\ref{sec:invariance}: on the audited collapsed
checkpoint, historical-row logits are bitwise invariant under suffix
changes (zero changed entries in $2{,}048{,}000$ randomized
comparisons), while on an earlier released checkpoint of the same
architecture $74.5\%$ of the compared entries move; both measurements
are reported in Appendix~\ref{app:numerical}.
\end{remark}

\paragraph{An optional cumulative construction.}
A masked running Gram restores \eqref{eq:strongprefix} by
construction, at a cost: with validity mask $m_n \in \{0,1\}$,
\begin{equation}
\label{eq:cumgram}
\widetilde D_t^{\,\mathrm{cum}} \;=\; \sum_{n \le t} m_n\,
\widetilde q_n \widetilde q_n^\top,
\qquad
\Glm^{(t)} = \Psi\bigl(\widetilde D_t^{\,\mathrm{cum}}\bigr),
\qquad
z_{tj} = \frac{\widetilde q_t^\top\, \Glm^{(t)}\, \widetilde k_j}
{\sqrt{\dk}} \quad (j \le t),
\end{equation}
computed layerwise. The rank-one updates to the Gram are cheap, but
row $t$ needs $\Psi$ evaluated at its own prefix Gram: a parallel
all-row pass evaluates the deep metric learner at up to $S$ prefixes
per layer and head (or in blocks, trading granularity for cost),
against \emph{one} $O(S\dk^2)$ contraction plus \emph{one} $\Psi$
evaluation per layer and head for the global Gram. Equation
\eqref{eq:cumgram} is an implementation option for callers who want
one-pass all-row autoregressive semantics or padded batching; it is
not required for sequential generation, and the reference
implementations and released checkpoints use the global Gram.

\paragraph{Padding at the Gram boundary.}
The reference implementations form $\widetilde D$ from \emph{every}
supplied query row: rows excluded by the attention mask still enter
the sequence contraction, and the measured final-real-row deviations
under appended attention-masked padding (including dependence on the
padding \emph{content}) are reported in
Appendix~\ref{app:numerical}. The deployed contract is therefore the
\emph{unpadded} $S = t$ interface: sequential generation presents only
genuine prompt tokens, and callers batching with right padding should
strip padding before scoring (stripping restores the unpadded call
bitwise). A Gram-side validity mask as in \eqref{eq:cumgram} is an
implementation option for future releases; the pinned released
checkpoints and their remote code are kept byte-stable instead.

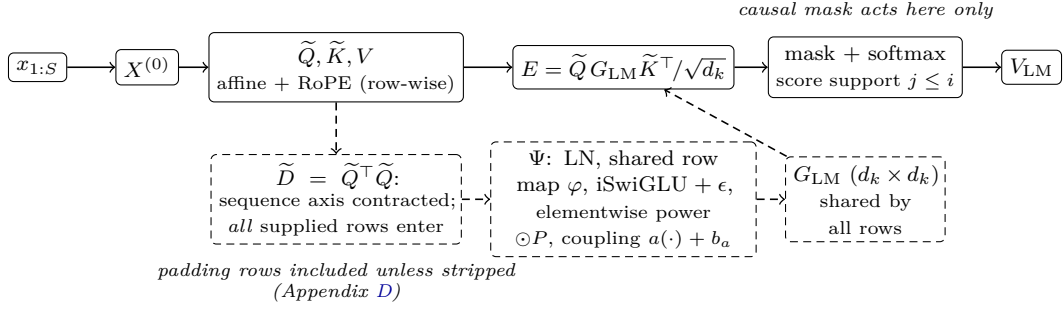
\begin{figure}[t]
\centering
\begin{tikzpicture}[
  every node/.style={font=\scriptsize},
  box/.style={draw, rounded corners=2pt, inner sep=3pt, align=center},
  glob/.style={draw, rounded corners=2pt, inner sep=3pt, align=center,
               densely dashed},
  lab/.style={font=\tiny\itshape, align=center},
  arr/.style={->, semithick},
  garr/.style={->, semithick, densely dashed}
]
\node[box] (tok)  at (-0.15, 2.3) {$x_{1:S}$};
\node[box] (emb)  at (1.3, 2.3)  {$X^{(0)}$};
\node[box] (qkv)  at (3.8, 2.3)  {$\widetilde Q, \widetilde K, V$\\
                                  \tiny affine + RoPE (row-wise)};
\node[box] (sco)  at (7.6, 2.3)  {$E = \widetilde Q\, \Glm
                                  \widetilde K^\top\!/\sqrt{\dk}$};
\node[box] (sm)   at (10.8, 2.3) {mask + softmax\\ \tiny score support
                                  $j \le i$};
\node[box] (out)  at (13.0, 2.3) {$\Vlm$};
\node[glob, text width=3.1cm] (gram) at (3.8, 0.55)
  {$\widetilde D = \widetilde Q^\top \widetilde Q$:\\ \tiny sequence
   axis contracted;\\ \tiny \emph{all} supplied rows enter};
\node[glob, text width=3.3cm] (psi)  at (7.6, 0.55)
  {$\Psi$: LN, shared row map $\varphi$, iSwiGLU${}+\epsilon$,\\
   \tiny elementwise power $\Had P$, coupling $a(\cdot) + b_a$};
\node[glob, text width=1.9cm] (glm)  at (10.8, 0.55)
  {$\Glm$ ($\dk \times \dk$)\\ \tiny shared by all rows};
\draw[arr] (tok) -- (emb);
\draw[arr] (emb) -- (qkv);
\draw[arr] (qkv) -- (sco);
\draw[arr] (sco) -- (sm);
\draw[arr] (sm) -- (out);
\draw[garr] (qkv) -- (gram);
\draw[garr] (gram) -- (psi);
\draw[garr] (psi) -- (glm);
\draw[garr] (glm) -- (sco);
\node[lab] at (3.8, -0.55) {padding rows included unless stripped\\
                            (Appendix~\ref{app:numerical})};
\node[lab] at (10.8, 3.05) {causal mask acts here only};
\end{tikzpicture}
\caption{Information flow in one PLGA head. Solid path: row-local
operations (row $i$ uses rows $\le i$ through the masked softmax).
Dashed path: the global density operator; the sequence axis is
contracted \emph{before} any mask acts, so the learned operator
$\Glm$ depends on every supplied row. Under the online contract
\eqref{eq:onlinecontract} all supplied rows are known context;
historical rows of a longer call are recomputed under the enlarged
context (Remark~\ref{rem:prefixmech}).}
\label{fig:infoflow}
\end{figure}

\paragraph{Terminology for training and evaluation semantics.}
Five distinct semantics appear in this paper, under these fixed
names: \emph{blockwise (global-context) cross-entropy}
\eqref{eq:blockloss}, the implemented training objective (one
full-block pass, cross-entropy at every aligned row);
\emph{sequential autoregressive NLL} (equivalently
\emph{perplexity}), $-\sum_t \log p_\theta(x_{t+1} \mid x_{0:t})$
computed by repeated final-row prefix calls, the chain-rule quantity,
equal to the blockwise value whenever \eqref{eq:strongprefix} holds
at the visited inputs (equality of the two realized scalar totals
alone does not imply \eqref{eq:strongprefix});
\emph{one-pass (whole-candidate) block scoring}, the
published evaluation protocol (one call on context plus candidate,
candidate log-probabilities summed; Appendix~\ref{app:repos});
\emph{sequential generation}, the deployed online loop of
\eqref{eq:onlinecontract}; and \emph{prompt-frozen cache inference},
the KV/G-cache semantics of Section~\ref{subsec:cache}. The measured
gaps between one-pass and sequential scoring on both released
checkpoints are reported in Appendix~\ref{app:numerical}.

\section{Deductive Outputs as Invariant Operators}
\label{sec:invariance}

\subsection{The steady state and the order parameter}

\begin{definition}[$\varepsilon$-invariance and the order parameter]
\label{def:orderparam}
Let $\mu$ be a distribution over admissible inputs (prompt plus stochastic
continuation). A deductive output $T(x) \in \R^{\nu}$ ($\nu$ the number of
entries of $T$) is
\emph{$\varepsilon$-invariant} if there exists a constant tensor $T^\ast$
with $\norm{T(x)-T^\ast}_{\mathrm{rms}} \le \varepsilon$ for $\mu$-almost
all $x$, where $\norm{\cdot}_{\mathrm{rms}} = \norm{\cdot}_F/\sqrt{\nu}$.
Following \cite{gokden2026} in substance, define the \emph{order
parameter} of a trained model $\theta$ from two independent generation
runs $x^{(1)}, x^{(2)}$ (or a run against the cached pass):
\begin{equation}
\label{eq:orderparam}
m(\theta) \;=\;
\frac{\norm{\,\mathcal{D}(x^{(1)};\theta) - \mathcal{D}(x^{(2)};\theta)\,}
_{\mathrm{rms}}}{\operatorname{rms}(\mathcal{D})},
\qquad
\operatorname{rms}(\mathcal{D}) = \tfrac12\bigl(
\norm{\mathcal{D}(x^{(1)};\theta)}_{\mathrm{rms}} +
\norm{\mathcal{D}(x^{(2)};\theta)}_{\mathrm{rms}}\bigr),
\end{equation}
the RMSE between deductive outputs across runs normalized by the RMS
magnitude of their entries; it is computed and reported \emph{per tensor
type}, with $\Glm$ (or $A$, the most sensitive) used as representative.
The statistic is defined \emph{piecewise}: $m = 0$ whenever the
numerator vanishes (equal tensors; this covers the all-zero pair, for
which the denominator also vanishes and the plain quotient would be
$0/0$), and by the displayed quotient otherwise, in which case the RMS
denominator is strictly positive, since it can vanish only if both
tensors, hence the numerator, vanish.\footnote{The source paper
\cite{gokden2026} normalizes by
the absolute signed global mean of the reference tensor: the first
run's $\abs{\mu_1}$ for a run pair, or $\abs{\mu_C}$ against a cached
run. It can be arbitrarily small by cancellation for tensors centered
near zero; a zero denominator with positive numerator gives $+\infty$.
The RMS denominator used here is stable. Appendix~\ref{app:numerical}
reports it alongside the symmetrized signed-mean denominator
$\tfrac12(\abs{\mu_1}+\abs{\mu_2})$, which is an adaptation of the
source convention rather than the same statistic. A reported value
$m = 0$ means agreement at floating-point resolution, not exact
invariance.}
\end{definition}

\begin{proposition}[Invariance suffices for exact cacheability]
\label{prop:cacheable}
\leavevmode
\begin{enumerate}[label=(\roman*),leftmargin=2.2em]
\item (Sufficiency.) If, for every layer and head,
$\Glm^{(\ell,i)}(x) = G^{\ast(\ell,i)}$ for all admissible inputs $x$,
then the G-cache \eqref{eq:gcache} computes the identical function to the
full forward pass for every prompt and continuation. Invariance of the
upstream tensors $A, \Alm, \Ap$ is \emph{not} necessary for this
conclusion.
\item (Definitional equivalence.) For a fixed tensor type $T$,
$m(\theta) = 0$ in exact arithmetic over all input pairs if and only if
$T$ is exactly input-invariant (with the piecewise convention of
Definition~\ref{def:orderparam}, which makes this hold without
exception, including for identically zero tensors).
\item (Two-level characterization.) \emph{Row level (iff):} two finite
allowed score rows induce the same softmax distribution if and only if
they differ by a common additive scalar. \emph{Model level (if):} if on
every reachable decoding state every recomputed allowed-score row
differs from its cached counterpart by a row-wise additive constant,
then cached and full attention rows, head outputs, and decoding
distributions agree for all admissible continuations. Constancy of
$\Glm$ is sufficient but not necessary for the row condition, and
\emph{no converse from full-model output equality to score-row equality
is asserted}: output equality does not imply row equality
(Remark~\ref{rem:cacheable-converse}).
\end{enumerate}
\end{proposition}

\begin{proof}
(i) At every decoding step the cached operator equals the operator the
full network would recompute, so scores, attention rows, and outputs
agree token-by-token. (ii) is immediate from \eqref{eq:orderparam} with an
exact-arithmetic reading of $0$ and the piecewise convention. (iii) Row
level: sufficiency of a common shift is invariance of the softmax under
adding a constant to an allowed row; necessity follows by taking
logarithms of the ratio of the two distributions, which shows the score
difference $z'_i - z_i = \log\sum_{j}e^{z'_j} - \log\sum_{j}e^{z_j}$ is
the same for every allowed $i$ (both directions are machine-checked;
Appendix~\ref{app:lean}). Model level: equal attention rows on every
reachable state give equal head outputs, hence equal states and logits,
by induction over decoding steps; this direction only.
\end{proof}

\begin{remark}[Why the converse directions fail]
\label{rem:cacheable-converse}
A natural strengthening would be the three-way equivalence ``$m = 0$
iff every deductive output invariant iff G-cache exact.'' Only the
directions stated in Proposition~\ref{prop:cacheable} hold; simple
counterexamples break necessity at every level. With coupling
$a = \bm0_{\mathrm{mat}}$ and $b_a = G^\ast$, the tensors $A, \Alm, \Ap$ vary arbitrarily
with the input while $\Glm \equiv G^\ast$ and G-cached inference is
exact. With $W_K = \bm0_{\mathrm{mat}}$ and $b_K = \bm0$, every key vanishes, so every
allowed score is identically zero and each attention row is the
uniform distribution over its allowed positions, independently of
$\Glm$; the query Gram, and hence $\Glm$, still varies with the
input, so cached and recomputed inference coincide while
$m(\Glm) \ne 0$: exact cache equivalence does not force operator
invariance. (The key map must be the one silenced here: zeroing the
\emph{query} map instead would freeze the query Gram and with it
$\Glm$ itself, leaving nothing that varies.) More generally $\Delta\Glm$ may lie in directions annihilated by
all reachable query--key pairs, or may shift each allowed score row by a
scalar, which the softmax removes exactly. Nor does equality of
full-model \emph{outputs} imply the row condition of
Proposition~\ref{prop:cacheable}(iii): take cached and recomputed score
rows $z = (0,0)$ and $z' = (0,1)$, which are not related by a common
shift and have different softmax distributions, and let $V = \bm0_{\mathrm{mat}}$; both
attention outputs vanish and all downstream states and tokens agree.
The same failure arises from a zero output projection, from attention
differences lying in a value/output nullspace, or from downstream blocks
identifying the differing head outputs. A full-model ``iff'' would
require injectivity assumptions along the value/output path and the
downstream network that no result of this paper supplies. The order
parameter is
therefore a \emph{diagnostic} whose vanishing (per tensor type, in exact
arithmetic) certifies invariance of that tensor and, for $\Glm$,
cacheability; its non-vanishing does not by itself preclude exact
functional cache equivalence. A separate scope note, distinguishing
three exactness senses. \emph{(a)} G-cache is
\emph{unconditionally} exact relative to the already-frozen KV-cache
semantics: given the cached $A$, the maps to $\Alm$ and $\Glm$ are
deterministic functions of $A$ and the parameters, so storing their
outputs adds no approximation beyond the freeze itself
(Section~\ref{subsec:cache}). \emph{(b)} Exactness relative to
\emph{full prefix recomputation} is a different, conditional
statement: it is what Proposition~\ref{prop:cacheable}(i) delivers
\emph{under operator invariance}, and an empirical matter otherwise.
\emph{(c)} Nothing in
Proposition~\ref{prop:cacheable} asserts historical-row prefix
consistency (Definition~\ref{def:prefixconsistent}) of the uncached
map, which the global Gram does not supply in general
(Remark~\ref{rem:prefixmech}).
\end{remark}

The published measurements of \cite{gokden2026}, using that paper's
mean-based normalization, report fluctuations of order $10^{-6}$ to
$10^{-11}$ in the near-critical group (and $0$ at float resolution for
$\Ap,\Glm$ of the $41$B-token model), against order $1$--$50$ in the
sub-critical group. Those ranges are not values of the RMS-normalized
$m(\theta)$ in \eqref{eq:orderparam}, which lies in $[0,2]$ by the
triangle inequality and its zero convention. Appendix~\ref{app:numerical}
reports the RMS and symmetrized signed-mean normalizations separately.
The source diagnostic separates the labeled groups and agrees with
benchmark rankings at that level, but small within-group reversals and
unquantified uncertainty prevent a precise ranking of reasoning ability.
Establishing (or refuting) a finer-grained relationship requires the
multi-seed protocol of Section~\ref{sec:discussion}.

\subsection{The inference-collapse theorem}

\begin{theorem}[SDPA as a special case; inference collapse;
training--inference asymmetry]
\label{thm:collapse}
Consider a PLDR-LLM $F_\theta$ per Definition~\ref{def:pldr}.
\begin{enumerate}[label=(\roman*),leftmargin=2.2em]
\item If $\Glm^{(\ell,i)} \equiv I$ for all layers and heads, then
$F_\theta$ is exactly a decoder-only transformer with scaled dot-product
attention (with RoPE, SwiGLU FFN, and the stated normalizations): SDPA-LLM is
the point of PLDR-LLM model space with identity energy--curvature tensor.
\item If every head score operator $\Glm^{(\ell,i)}$ is exactly
input-invariant with value $G^{\ast(\ell,i)}$, then the inference map
of $F_\theta$ coincides with the
map of the architecture in which the subnetwork
\eqref{eq:metricgen}--\eqref{eq:ec} is deleted and replaced by the constants
$G^{\ast(\ell,i)}$, a \emph{generalized SDPA} with learned bilinear forms.
If additionally RoPE is absent, this is exactly an SDPA-LLM with query
projections $W_Q^{(\ell,i)}\, G^{\ast(\ell,i)}$. With RoPE:
\emph{(sufficiency, for the model at hand)} if each $G^{\ast(\ell,i)}$
lies in the ambient torus commutant of
Proposition~\ref{prop:rope-commutant},
the constant operator is absorbable into pre-rotation projections and
the model is exactly an SDPA-LLM; \emph{(necessity, at the unrestricted
operator level)} if the head's bilinear score family
$(n,m,q,k) \mapsto q^\top R_{-n} G^\ast R_m k$ is required to be
SDPA-realizable for \emph{all} positions and all
$q, k \in \R^{\dk}$, then $G^\ast$ must lie in the commutant; this is
Corollary~\ref{cor:posgap}, which carries the quantifiers. For a fixed
trained model, whose reachable queries and keys span restricted
subspaces, commutant membership is sufficient but \emph{not} necessary
for SDPA-realizability of the inference map ($W_Q = \bm0_{\mathrm{mat}}$ \emph{and}
$b_Q = \bm0$, so that the affine query map, and with it every score,
vanishes identically regardless of $G^\ast$, is the degenerate
witness); no model-level necessity is claimed.
\item (Asymmetry, structural.) The gradient of the loss in the full
parameterization decomposes as the collapsed-parameterization terms plus
the chain-rule term
$\sum_{\ell,i}\bigl\langle \partial\mathcal{L}/\partial \Glm^{(\ell,i)},\;
\partial \Glm^{(\ell,i)}/\partial\theta \bigr\rangle$ through the operator generator, including its residual network, the
parameters $W,b_W,P,a,b_a$, and the query and Gram-normalization
parameters. The products denote adjoint-Jacobian actions. This
term is absent by construction from the collapsed model, where $\Glm$ is
a constant. \emph{Under the explicit hypothesis that this term is
nonzero at the evaluation point}, the full gradient and the gradient with the operator held fixed differ
in the same parameter coordinates. (Comparing \emph{training dynamics} further
requires fixing an optimizer and a correspondence between the parameter
spaces; the empirical separation of loss curves and benchmark scores
between learned, transferred, random, and identity $\Glm$
\cite{gokden2025} is cited as evidence that the difference is realized
in practice, not as part of the statement.)
\end{enumerate}
\end{theorem}

\begin{proof}
(i) Substituting $\Glm = I$ into \eqref{eq:scores}--\eqref{eq:out} gives
$E = \widetilde Q \widetilde K^\top/\sqrt{\dk}$,
$\Elm = \softmax(\operatorname{mask}(E))$, $\Vlm = \Elm V$: the SDPA
equations \cite{vaswani2017}; all remaining blocks in
Definition~\ref{def:pldr} are shared.
(ii) Under exact invariance, at every decoding step the recomputed
$\Glm(x)$ equals $G^\ast$; replacing the computation by the constant yields
the same scores, hence the same distribution over continuations. Without
RoPE, $Q G^\ast K^\top = (Q G^\ast) K^\top = Q' K^\top$ with
\[
Q' \;=\; Q\,G^\ast \;=\;
X\,\bigl(W_Q G^\ast\bigr) + \mathbf{1}\,\bigl(b_Q^\top G^\ast\bigr),
\]
an SDPA parameterization in which the affine query map transforms as a
whole: $W_Q' = W_Q G^\ast$ \emph{and} $b_Q'^{\,\top} = b_Q^\top
G^\ast$ (the projection of Definition~\ref{def:pldr} is bias-bearing,
so absorbing $G^\ast$ into the weight alone would change the model;
the affine identity is machine-checked, Appendix~\ref{app:lean}).
With RoPE, sufficiency:
if $G^\ast$ commutes with every $R_t$, then
$R_{-n} G^\ast R_m = G^\ast R_{m-n}$ and the score
\eqref{eq:rope-score} is the pure offset form realized by SDPA with the
same absorbed affine projection applied before rotation: a commuting
$G^\ast$ also commutes with each $R_t^{\top} = R_{-t}$, so, position
by position,
\[
\RoPE(Q\, G^\ast) \;=\; \RoPE(Q)\, G^\ast .
\]
Necessity at the unrestricted operator level is
Corollary~\ref{cor:posgap}, whose proof is self-contained (equality of
the score \emph{functions} for all positions and all $q, k$ forces
commutant membership); for a fixed model only the restrictions of
$R_{-n} G^\ast R_m$ to the reachable query/key subspaces are observable,
so no necessity is claimed there.
(iii) Write the loss as
$\mathcal{L}(\theta) = \mathcal{L}\bigl(\text{blocks};\,
\Glm(\cdot;\theta)\bigr)$; the chain rule gives the stated decomposition,
and the extra term is identically absent from the collapsed
parameterization. If the term is nonzero at $\theta$, the two gradients
differ at $\theta$ by exactly that term.
\end{proof}

\begin{remark}[Scope of the asymmetry]
\label{rem:asymmetry-scope}
Theorem~\ref{thm:collapse}(iii) is a statement about \emph{parameterized
training dynamics}: the two parameterizations induce different gradient
flows (a mechanism with contemporaneous precedent in standard attention
itself, where factorization alone implicitly rescales the query--key
and output--value circuits' relative learning rates under gradient
flow \cite{vashisht2026}). It does not assert that no SDPA-LLM of any size realizes the same
input--output function as a trained PLDR-LLM; strict function-class
containment at matched depth, width, and positional scheme would require a
nonrepresentability theorem that we do not have (see
\S\ref{subsec:adv-family}). Similarly, the claim in (ii) that a constant
operator ``loses no inference-time expressivity'' is a statement about a
\emph{particular} trained model already observed to have (approximately)
constant $\Glm$; as a family-wide statement it would be false, since an
input-dependent $\Glm$ can realize score maps unavailable to any
constant-$G$ head.
\end{remark}

\subsection{Quantitative stability of the collapse}

In practice invariance is $\varepsilon$-exact rather than exact. The
following bounds show that the observed $\varepsilon$ (down to $10^{-11}$
relative) propagates to a perturbation of the output distribution that is
bounded by explicit constants. Whether the resulting bound is small
enough to force bit-identical decoding is a separate, quantitative
question: it requires comparing the evaluated end-to-end constant with
the realized logit margins, which we do on a released checkpoint in
Appendix~\ref{app:numerical}. The bit-identical cached-versus-uncached
benchmark scores reported in \cite{gokden2025} were obtained under
one-pass block scoring, whose scoring path performs a single uncached
pass by construction (Appendix~\ref{app:repos}); the informative
empirical facts here are the measured cached-versus-recomputed
operator and logit deviations of Appendix~\ref{app:numerical}, which
are consistent with these bounds but are an empirical observation, not
a corollary of them.

\begin{lemma}[Softmax is $1$-Lipschitz \cite{gao2017}]
\label{lem:softmax}
For $z, z' \in \R^S$,
$\norm{\softmax(z) - \softmax(z')}_2 \le \norm{z - z'}_2$.
\end{lemma}

\begin{proposition}[Perturbation bound for G-caching]
\label{prop:perturb}
Fix a head with rotated inputs $\widetilde Q, \widetilde K$ and values $V$,
and let $\Elm, \Vlm$ and $\Elm', \Vlm'$ be computed with operators $G$ and
$G'$ where $\norm{G - G'}_2 \le \varepsilon$. Then, row-wise for each
position $t$,
\begin{equation}
\norm{\Elm[t,\cdot] - \Elm'[t,\cdot]}_2
\;\le\; \frac{\norm{\widetilde q_t}_2\, \norm{\widetilde K}_2}
{\sqrt{\dk}}\;\varepsilon,
\qquad
\norm{\Vlm[t,\cdot] - \Vlm'[t,\cdot]}_2
\;\le\; \frac{\norm{\widetilde q_t}_2\, \norm{\widetilde K}_2\,
\norm{V}_2}{\sqrt{\dk}}\;\varepsilon .
\end{equation}
Consequently, if every non-PLGA block of the network is Lipschitz on the
relevant compact set (true of the $\varepsilon$-LayerNorm globally,
Lemma~\ref{lem:ln}(iv), and of linear maps,
SwiGLU, and softmax), the final logits of the cached and uncached models
differ by at most $C(\theta)\,\varepsilon$ for a constant depending on
operator norms of the trained weights, uniformly over inputs.
\end{proposition}

\begin{proof}
The score rows differ by
$\Delta e_t = \widetilde q_t^\top (G - G') \widetilde K^\top/\sqrt{\dk}$, so
by sub-multiplicativity
\[
\norm{\Delta e_t}_2 \;\le\; \norm{\widetilde q_t}_2\, \norm{G-G'}_2\,
\norm{\widetilde K}_2 / \sqrt{\dk};
\]
under the ideal masked softmax (Remark~\ref{rem:finitemask}) the
restriction to the common allowed support only removes coordinates of
$\Delta e_t$ and cannot increase the norm. Apply
Lemma~\ref{lem:softmax} row-wise for the first bound. For the second,
$\Delta \Vlm[t,\cdot] = (\Delta \Elm[t,\cdot]) V$ and
$\norm{u^\top V}_2 \le \norm{u}_2\norm{V}_2$. The network-level statement is
the composition of Lipschitz maps, each perturbation entering additively
with the product of downstream Lipschitz constants.
\end{proof}

\begin{remark}[Mach's principle in embedding space]
The source papers offer a physical reading \cite{gokden2025}: $\Glm$ is
generated locally by each input, yet is (numerically) determined by the
global distribution of all training data: ``local inertial frames are
determined by the large-scale distribution of matter'' \cite{misner1973}.
Section~\ref{sec:criticality} offers a complementary
statistical-mechanical reading of the same phenomenon (order parameter
of a phase). These readings are analogies. Their compatibility does not
establish a physical mechanism or a scaling law.
\end{remark}

\subsection{The origin of the invariance: how the metric learner generates
an invariant metric generator}
\label{subsec:origin}

The results so far characterize the invariant steady state and its
consequences; this subsection analyzes \emph{how} the
metric generator $A = \Phi_{\mathrm{res}}(\LN(\widetilde D))$ becomes
input-invariant.
The analysis rests on three verified structural facts: (F1) the density
operator is built from \emph{rotary-rotated} queries
(Definition~\ref{def:pldr}); (F2) LayerNorm is applied to the density
operator row-wise before the residual network; (F3) the metric learner is a
single shared row map $\varphi$ applied independently to every row and every
head (Proposition~\ref{prop:rowfact}). The proposed mechanism decomposes
into three stages: a phase-averaging at the source, statistical
concentration of what remains, and contraction inside the learned row
map. The epistemic status of each stage differs and is flagged in place:
Stage 1 is an exact bound whose constant is honest but large at the
reference context length; Stage 2 is a conditional theorem under an
independence idealization; Stage 3 is a conditional theorem whose
contraction hypothesis is \emph{measured}, not proved
(Appendix~\ref{app:numerical}). The mechanism as a whole is therefore a
quantitative hypothesis with proved ingredients, not a theorem.

\subsubsection{Stage 1: the rotary twirl projects the density operator onto
a commutant}

Because queries are rotated before the Gram product, the density operator of
a head is a \emph{position-twisted} second moment:
\begin{equation}
\label{eq:twisted}
\widetilde D \;=\; \widetilde Q^\top \widetilde Q
\;=\; \sum_{n=1}^{S} R_n\, q_n q_n^\top\, R_n^\top .
\end{equation}
The position-dependent conjugation acts as a \emph{twirl} (an average
over a group orbit) and suppresses every component of the summands that
does not commute with the ambient RoPE torus $T$:

\begin{lemma}[Quantitative RoPE twirl]
\label{lem:twirl}
Let $R_n$ be RoPE rotations with angles $\theta_1,\dots,\theta_{\dk/2}$ such
that none of the finitely many frequencies
$\omega \in \{2\theta_a\} \cup \{\theta_a \pm \theta_b : a\ne b\}$ is a
multiple of $2\pi$. Let $P_T$ denote the orthogonal projection of
$\R^{\dk\times\dk}$ onto the commutant of the ambient RoPE torus $T$
(the block-diagonal algebra of
Proposition~\ref{prop:rope-commutant}). Then for
every fixed $M$,
\[
\Bigl\lVert\, \frac1S \sum_{n=1}^{S} R_n M R_n^\top \;-\; P_T(M)
\,\Bigr\rVert_F
\;\le\; \frac{C_\Theta}{S}\, \norm{M}_F,
\qquad
C_\Theta = \max_{\omega \ne 0}\ \frac{1}{\,\abs{\sin(\omega/2)}\,},
\]
the maximum over the nonzero frequencies above.
\end{lemma}

\begin{proof}
Complexify each rotation plane: in the eigenbasis of the torus, $R_n$ is
diagonal with entries $e^{\pm i\theta_a n}$, and conjugation multiplies the
$(u,v)$ matrix entry of $M$ (in this basis) by $e^{i(\omega_u - \omega_v)n}$
with $\omega_u, \omega_v \in \{\pm\theta_a\}$. The difference frequencies
are exactly $0$, $\pm2\theta_a$, and $\pm(\theta_a \mp \theta_b)$. The
zero-frequency entries span precisely the commutant (the $c I_2 + s J_2$
blocks of Proposition~\ref{prop:rope-commutant}), and are left fixed by the
average. Every nonzero-frequency entry is multiplied by
$\frac1S \sum_{n=1}^S e^{i\omega n}$, whose modulus is
$\abs{\sin(S\omega/2)}/(S\abs{\sin(\omega/2)}) \le 1/(S\abs{\sin(\omega/2)})$
by the geometric sum. Since the basis change is unitary, the entry-wise
bounds assemble into the stated Frobenius bound.
\end{proof}

Applied to \eqref{eq:twisted} \emph{under the stationarity idealization
of Proposition~\ref{prop:conc}} (a common second moment
$\Sigma = \E[q q^\top]$ across positions; the lemma itself concerns one
fixed matrix $M$, and passing to the empirical sum requires this
additional assumption), Lemma~\ref{lem:twirl} shows that the
deterministic part of $\widetilde D/S$ converges at rate $O(1/S)$ to
$P_T(\Sigma)$. Two qualifications keep this honest. First, since every
summand $q_nq_n^\top$ is symmetric, the projection of $\Sigma$ onto the
commutant has vanishing $J_2$-components: the effective target is the
smaller algebra of \emph{scalar} blocks $c_j I_2$. Second, the
constant matters. For the standard frequencies at $\dk = 64$, base
$10^4$, the exact worst constant is
$C_\Theta = \max_{\omega\neq0} \abs{\sin(\omega/2)}^{-1} \approx
4.4968 \times 10^4$ (attained by the smallest cross-plane difference
frequency), so at the reference context length $S = 1024$ the uniform
bound $C_\Theta/S \approx 43.9$ is vacuous. The suppression is
frequency-resolved: the exact per-frequency multiplier is
$\abs{\sin(S\omega/2)}/(S\abs{\sin(\omega/2)})$, which is strongly
suppressing for the many fast frequencies but approaches $1$ for the
slowest ones (the worst component has multiplier $\approx 0.99991$ at
$S = 1024$; even the slowest within-plane frequency $2\theta_{\dk/2}$
has multiplier $\approx 0.9969$). The twirl therefore erases
position-of-occurrence structure carried by the fast rotary planes but
leaves the slowest planes essentially untouched at practical context
lengths; how much instance structure it removes \emph{in practice}
depends on how the energy of $\widetilde D$ distributes across twirl
frequencies, which we measure on a released checkpoint in
Appendix~\ref{app:numerical} rather than infer from the asymptotic. With
these qualifications, rotary embeddings, adopted in \cite{gokden2024}
for their training benefits, also act as a partial self-averaging
channel for the deductive path.

\subsubsection{Stage 2: statistical concentration of the normalized density
operator}

What the twirl does not remove, the fluctuation of the empirical second
moment around its mean, concentrates statistically:

\begin{proposition}[Concentration of the normalized query Gram; conditional]
\label{prop:conc}
Fix a layer, an even head width $d=\dk\ge2$, and the RoPE rotations
$R_n$ under the nonresonance hypotheses of
Lemma~\ref{lem:twirl}. Let $P_T$ be the Frobenius-orthogonal
projection onto the corresponding commutant and $C_\Theta$ its twirl
constant. For a positive integer context length $S$, model the
\emph{pre-rotation} query vectors $q_1,\ldots,q_S\in\R^d$ as independent,
with $\norm{q_n}_2\le B_q$ almost surely and common second moment
$\E[q_nq_n^\top]=\Sigma$ in that frame. Define
$\widetilde D=\sum_{n=1}^S R_nq_nq_n^\top R_n^\top$, as in
\eqref{eq:twisted}. Then, for $0<\delta<1$, with
probability at least $1-\delta$,
\[
\begin{split}
 \left\lVert\frac{\widetilde D}{S}-P_T(\Sigma)\right\rVert_2
 \le {}&\frac{C_\Theta}{S}\norm{\Sigma}_F
       +B_q^2\sqrt{\frac{2\log(2d/\delta)}{S}}\\
       &+\frac{2B_q^2\log(2d/\delta)}{3S}.
\end{split}
\]
Here the matrix $2$-norm is the spectral norm. For fixed dimension,
confidence level, rotary frequencies and uniform query bound, the
normalized query Gram concentrates about $P_T(\Sigma)$ with the displayed
$S^{-1/2}$ fluctuation term and $S^{-1}$ remainders. Transfer to the
implemented input $\LN_\varepsilon(\widetilde D)$ uses
the explicit transfer following Lemma~\ref{lem:ln}, retaining its scale
discrepancy. Independence and a common pre-rotation second moment are
modeling assumptions, not architectural properties of contextual queries.
\end{proposition}

\begin{proof}
Write
\[
\begin{split}
 \frac{\widetilde D}{S}-P_T(\Sigma)
 &=\left(\frac1S\sum_{n=1}^S R_n\Sigma R_n^\top-P_T(\Sigma)\right)
   +\frac1S\sum_{n=1}^S X_n,\\
 X_n&=R_n(q_nq_n^\top-\Sigma)R_n^\top.
\end{split}
\]
The deterministic term has Frobenius norm at most
$C_\Theta\norm{\Sigma}_F/S$ by Lemma~\ref{lem:twirl}, and its
spectral norm is no larger. The $X_n$ are independent, symmetric and
mean zero. In positive-semidefinite matrix order,
$\bm0_{d\times d}\preceq q_nq_n^\top\preceq B_q^2I_d$ and
$\bm0_{d\times d}\preceq\Sigma\preceq B_q^2I_d$, so
$-B_q^2I_d\preceq q_nq_n^\top-\Sigma\preceq B_q^2I_d$.
Orthogonal conjugation preserves the spectral norm, hence
$\norm{X_n}_2\le B_q^2$. The averaged variance proxy satisfies
\[
 \left\lVert\sum_{n=1}^S\E[(X_n/S)^2]\right\rVert_2
 \le\frac{B_q^4}{S}.
\]
The self-adjoint case of the matrix Bernstein inequality
\cite[Thm.~1.6.2]{tropp2015}, with summand bound $B_q^2/S$, this variance
proxy and tail factor $2d$, gives the last two terms of the displayed
bound. The triangle inequality proves the claim.
\end{proof}

We emphasize a bookkeeping point: $S$ here is the \emph{context length}.
It governs concentration within one forward pass and must not be
substituted for the pretraining token count $N_{\mathrm{tokens}}$, which
controls a different limit (how well the learned parameters approximate
dataset-level statistics); the two enter the invariance question through
different mechanisms.

The implemented LayerNorm stage acts on the unnormalized Gram
$\widetilde D$, while the preceding bound concerns $\widetilde D/S$.
Their relation therefore requires the explicit scale transfer below.
The implemented map carries the regularization constant
$\varepsilon_{\LN} = 10^{-6}$, and its invariances must be stated with
that constant in place:\footnote{Exact positive-scale invariance and a
spherical range hold only at $\varepsilon_{\LN} = 0$ and are false for
the implemented map (e.g.\ for
$r = (10^{-4}, -10^{-4})$, $\gamma = \mathbf{1}$, $\beta = \bm0$, the
normalized coordinates are $\pm0.0995$ at $c = 1$ and $\pm0.1961$ at
$c = 2$). The lemma below therefore keeps exact shift invariance,
quantifies the scale-invariance error, and replaces the sphere by a
ball.}

\begin{lemma}[$\varepsilon$-LayerNorm: exact shift invariance,
approximate scale invariance, compact containment]
\label{lem:ln}
Let
$\rho_{\LN}^{\varepsilon}(r) = \gamma \Had
\dfrac{r - \bar r\mathbf{1}}{\sqrt{v(r)+\varepsilon}} + \beta$
be the implemented LayerNorm row map with learned $(\gamma,\beta)$, where
$\bar r$ and $v(r)$ are the mean and (biased) variance of the entries of
$r \in \R^{\dk}$ and $\varepsilon = \varepsilon_{\LN} > 0$. Then:
\begin{enumerate}[label=(\roman*),leftmargin=2.2em]
\item (Exact shift invariance.)
$\rho_{\LN}^{\varepsilon}(r + d\mathbf{1}) = \rho_{\LN}^{\varepsilon}(r)$
for all $d \in \R$.
\item (Scale covariance and approximate scale invariance.) For $c>0$,
\[
 \rho_{\LN}^{\varepsilon}(cr)=\rho_{\LN}^{\varepsilon/c^2}(r).
\]
For a constant row $r=b\mathbf1$, both sides equal $\beta$.
For nonconstant $r$,
\[
\rho_{\LN}^{\varepsilon}(cr) - \rho_{\LN}^{\varepsilon}(r)
= \gamma \Had \frac{r - \bar r\mathbf{1}}{\sqrt{v(r)}}
\left(\frac{1}{\sqrt{1 + \varepsilon/(c^2 v(r))}}
    - \frac{1}{\sqrt{1 + \varepsilon/v(r)}}\right),
\]
so that
$\norm{\rho_{\LN}^{\varepsilon}(cr) - \rho_{\LN}^{\varepsilon}(r)}_2
\le \norm{\gamma}_\infty \sqrt{\dk}\,
\dfrac{\varepsilon}{2\min(1, c^2)\, v(r)}$.
For $\gamma \ne \bm0$, scale invariance for \emph{all} $c > 0$ and all
nonconstant rows holds iff $\varepsilon = 0$ (at $c = 1$, or when
$\gamma = \bm0$, the two sides coincide trivially for any $\varepsilon$).
In particular
$\LN(\widetilde D) = \LN(\widetilde D/S)$ holds only approximately, with
per-row relative error in the centered component bounded by the
\emph{upper proxy} $\varepsilon/(2 v_i)$, where $v_i$ is the row
variance of $\widetilde D/S$ (equivalently $\varepsilon S^2/(2V_i)$ in
terms of the row variance $V_i$ of $\widetilde D$); the exact
first-order coefficient for this pair of scales is
$(1 - S^{-2})\,\varepsilon/(2 v_i)$, so the proxy overestimates already
at first order by the (small) factor $(1-S^{-2})^{-1}$, is valid as an
expansion only for $\varepsilon \ll v_i$, and overestimates grossly
outside that regime (on an exactly constant row the true discrepancy is
$0$ while the proxy diverges). The discrepancy itself is measured
directly, row by row with the checkpoint's
$\gamma, \beta, \varepsilon$, in Appendix~\ref{app:numerical}.
\item (Compact containment.) Every output row lies in the
fixed compact set
$\Sigma_{\LN} = \{\gamma \Had u + \beta : \mathbf{1}^\top u = 0,\
\norm{u}_2 \le \sqrt{\dk}\}$; indeed the normalized vector
$u = (r - \bar r\mathbf{1})/\sqrt{v(r)+\varepsilon}$ satisfies exactly
$\norm{u}_2^2 = \dk\, v(r)/(v(r)+\varepsilon) < \dk$, approaching the
sphere only as $v(r)/\varepsilon \to \infty$, and $u = \bm0$ for constant
rows. Compactness describes $\Sigma_{\LN}$; the actual LayerNorm image
is not necessarily compact.
\item (Global Lipschitz on variance-floored rows.) On rows with
$v(r) \ge v_0 \ge 0$, $\rho_{\LN}^{\varepsilon}$ is Lipschitz with
constant at most $\norm{\gamma}_\infty/\sqrt{v_0 + \varepsilon}$
(an exact Jacobian bound); the
$\varepsilon$ floor makes the constant finite on \emph{all} rows
($v_0 = 0$), removing the null-set caveat of the idealized map.
\end{enumerate}
\end{lemma}

\begin{proof}
(i) Centering removes $d\mathbf{1}$ and $v$ is shift-invariant.
(ii) Centering multiplies by $c$ and variance by $c^2$.
Factoring $c>0$ from the square root proves the scale-covariance
identity, and a constant row has zero centered numerator.
Both terms share the vector
$(r-\bar r\mathbf{1})/\sqrt{v(r)}$ scaled by
$(1+\varepsilon/(c^2v))^{-1/2}$ and $(1+\varepsilon/v)^{-1/2}$
respectively; the difference of the scalars is bounded, via
$1 - (1+x)^{-1/2} \le x/2$ for $x \ge 0$, by
$\varepsilon/(2\min(1,c^2)v)$, and
$\norm{(r-\bar r\mathbf 1)/\sqrt{v}}_2 = \sqrt{\dk}$.
(iii) $\norm{r - \bar r\mathbf{1}}_2^2 = \dk\,v(r)$ gives the exact norm
identity. (iv) With $P = I - \mathbf{1}\mathbf{1}^\top/\dk$ (centering),
$u(r) = Pr/\sqrt{v(r)+\varepsilon}$, a direct computation gives the
Jacobian factorization
$J_u(r) = \bigl(I - u u^\top/\dk\bigr)\,P \big/ \sqrt{v(r)+\varepsilon}$.
Both matrix factors are symmetric with spectrum in $[0,1]$ (for the
first, $\norm{u}_2^2 \le \dk$ by (iii)), so
$\norm{J_u(r)}_2 \le 1/\sqrt{v(r)+\varepsilon}$ at every floored row
(the pointwise bound used in the sample-extrema budget proxy of
Appendix~\ref{app:numerical}), and with the affine output map
contributing at most $\norm{\gamma}_\infty$, the mean value inequality
gives the Lipschitz claim globally for $v_0 = 0$, where the domain is
all of $\R^{\dk}$. For $v_0 > 0$ the floored set $\{v(r) \ge v_0\}$ is
the complement of a convex set and is not convex, so segments between
floored rows may leave it; the global claim follows instead from a
radial argument. Write $w = Pr$, so $v(r) = \norm{w}_2^2/\dk$ and
$u = g(w)$ with $g(w) = w\,(\norm{w}_2^2/\dk + \varepsilon)^{-1/2}$, a
radial map $g(w) = h(\rho)\,\hat w$ with $\rho = \norm{w}_2$ and
$h(\rho) = \rho\,(\rho^2/\dk + \varepsilon)^{-1/2}$. Set
$\rho_0 = \sqrt{\dk v_0}$ and $L = (v_0+\varepsilon)^{-1/2}$. On
$[\rho_0, \infty)$: $h(\rho)/\rho = (\rho^2/\dk+\varepsilon)^{-1/2}
\le L$, and $0 \le h'(\rho) = \varepsilon\,(\rho^2/\dk +
\varepsilon)^{-3/2} \le \varepsilon\,(v_0+\varepsilon)^{-3/2} \le L$,
so $\abs{h(\rho_1) - h(\rho_2)} \le L \abs{\rho_1 - \rho_2}$ by the
mean value theorem on the \emph{interval} $[\rho_0, \infty)$, which is
convex. For $w_1, w_2$ with $\rho_i \ge \rho_0$ and angle $\theta$
between them,
$\norm{g(w_1) - g(w_2)}_2^2 = h_1^2 + h_2^2 - 2h_1 h_2 \cos\theta$
and $\norm{w_1 - w_2}_2^2 = \rho_1^2 + \rho_2^2 -
2\rho_1\rho_2\cos\theta$, so
$\Phi(c) = L^2\norm{w_1-w_2}_2^2 - \norm{g(w_1)-g(w_2)}_2^2$ is affine
in $c = \cos\theta$ with slope $2(h_1 h_2 - L^2\rho_1\rho_2) \le 0$ by
the first inequality; $\Phi$ is therefore minimized at $c = 1$, where
$\Phi(1) = L^2(\rho_1-\rho_2)^2 - (h_1-h_2)^2 \ge 0$ by the second.
Hence $g$ is $L$-Lipschitz on the whole floored set, and composing
with the $1$-Lipschitz centering $P$ and the $\norm{\gamma}_\infty$
diagonal gives the stated constant, convexity of the floored set
nowhere used.
\end{proof}

To transfer concentration, fix one layer and its affine parameters
$\gamma,\beta$. Write $\rho_\varepsilon=\rho_{\LN}^{\varepsilon}$ and define
\begin{alignat*}{2}
 r_{x,i}&=(\widetilde D_x/S_x)_{i,:}^{\top},&\qquad
 m_i&=(P_T(\Sigma))_{i,:}^{\top},\\
 u_{x,i}&=\rho_\varepsilon(S_xr_{x,i}),&
 t_i&=\rho_\varepsilon(m_i),\\
 \delta_{x,i}&=\|u_{x,i}-\rho_\varepsilon(r_{x,i})\|_2.&
\end{alignat*}
Here $S_x>0$ is context length, and $i$ indexes a feature row within the
fixed head. The rows $r_{x,i}$ are normalized by context length, before
LayerNorm; $u_{x,i}$ are the actual row-map inputs; $t_i$ are fixed-offset
targets. Under Proposition~\ref{prop:conc}'s independent-query
and common-second-moment hypotheses, the Gram concentration transfers as
\[
 \|u_{x,i}-t_i\|_2\le\delta_{x,i}
       +L_\rho\|r_{x,i}-m_i\|_2,
\]
by adding and subtracting $\rho_\varepsilon(r_{x,i})$ and using its
Lipschitz bound. A global choice is
$L_\rho=\|\gamma\|_\infty/\sqrt\varepsilon$; the variance-floored
bound in Lemma~\ref{lem:ln} may be used on its declared domain.
For different row indices $i,j$, comparison through these targets retains
$L_\rho\|m_i-m_j\|_2$ in addition to the two concentration errors and
scale discrepancies. Concentration around target rows does not make
the target rows identical. Pooling heads requires a common comparison
law and parameters; pooling layers also requires a common row map.

The scale discrepancy cannot generally be dropped. With unit gain,
zero bias, $M=0.002I_2$ and $\varepsilon=10^{-6}$, each row variance
of $M$ equals $10^{-6}$. Thus
\[
 \lim_{S\to\infty}\|\LN_\varepsilon(SM)-\LN_\varepsilon(M)\|_F
       =2-\sqrt2.
\]
Indeed the coordinates tend to $\pm1$, whereas fixed-offset normalization
of $M$ gives $\pm1/\sqrt2$. For any positive-variance row,
$\rho_\varepsilon(Sr)$ tends to the zero-offset normalization by
Lemma~\ref{lem:ln}(ii), rather than the fixed-positive-offset target. Constant
rows instead give $u_{x,i}=\beta$ and $\delta_{x,i}=0$ exactly.
Near-constant rows require the explicit positive-offset expressions;
the proxy proportional to $\varepsilon/v(r)$ is not a uniform small-error
certificate.

Stages 1--2 therefore place the implemented inputs in the compact set
$\Sigma_{\LN}$ and give the explicit target-transfer bound above.
An $O(S^{-1/2})$ rate for the normalized Gram alone does not establish
that rate about the fixed-offset LayerNorm targets. Invariance of $A$
requires control of the trained row map on the actual input set.

\subsubsection{Stage 3: a conditional contraction bound for the trained
row map}

\begin{proposition}[Conditional contraction bound, collapse, and the
converse]
\label{prop:contract}
Let $\mathcal{S} \subset \Sigma_{\LN}$ be the set of implemented input
rows $u_{x,i}$ visited across contexts and heads of one fixed layer, and let
$\varphi = u_{N_{\mathrm{res}}} \circ \cdots \circ u_1$ be the row map of
Proposition~\ref{prop:rowfact}. \emph{Hypothesis (not proved; sampled
diagnostics for it, at the level of the full composition, are reported
in Appendix~\ref{app:numerical}): each unit $u_j$ is $L_j$-Lipschitz on
the relevant tube around the set it receives.}
\begin{enumerate}[label=(\roman*),leftmargin=2.2em]
\item $\operatorname{diam} \varphi(\mathcal{S}) \le
\bigl(\prod_j L_j\bigr) \operatorname{diam}\mathcal{S}$. If additionally
$L_j \le \kappa < 1$ for all $j$ (the \emph{contractive regime}, an
assumption about the trained weights that residual connections and
LayerNorm do not automatically deliver), the output diameter decays
exponentially in depth, $\operatorname{diam}\varphi(\mathcal{S}) \le
\kappa^{N_{\mathrm{res}}} \operatorname{diam}\mathcal{S}$, and the output
of $\varphi$ on $\mathcal{S}$ is confined to a small set. We emphasize
that the $u_j$ are \emph{distinct} learned maps: this is a diameter
bound for a composition, not a Banach fixed-point iteration, and no
fixed point or attractor dynamics is asserted.
\item In that regime, $A(x) \approx \mathbf{1}\alpha^{\ast\top}$
simultaneously for every head in this fixed layer, where
$\alpha^\ast$ may be taken to be $\varphi(u_0)$ for any fixed
$u_0 \in \mathcal{S}$ (consistent with the rank-one,
cross-head-identical singularity condition of
Proposition~\ref{prop:rankone} and the observations of
\cite{gokden2025}). For a fixed row index $i$, the transfer from the pre-LayerNorm rows to
the implemented inputs, on a domain with uniform constants
$L_\varphi=\operatorname{Lip}(\varphi)$ and $L_\rho$, is
\[
 \|\varphi(u_{x,i})-\varphi(u_{y,i})\|_2
 \le L_\varphi\bigl(\delta_{x,i}
       +L_\rho\|r_{x,i}-r_{y,i}\|_2+\delta_{y,i}\bigr).
\]
The scale discrepancies in Appendix~\ref{app:numerical} are
small at the median and have a disclosed nonuniform tail. The normalized
Gram concentration controls the middle term only under its probability
law. Sampled composite Jacobians are local linearizations, not a uniform
$L_\varphi$. These ingredients supply a quantitative budget only when
the required uniform bound on the comparison tube is available.
\item (Converse, pairwise.) If for a pair of implemented input rows $u, u'$ one
has the lower bound
$\norm{\varphi(u)-\varphi(u')} \ge c_0 \norm{u-u'}$, then the
corresponding rows of $A$ differ by at least $c_0\norm{u - u'}$: that
selected pair is not collapsed. A lower bound on the \emph{sampled}
order parameter of Definition~\ref{def:orderparam} requires, in
addition, a distributional assumption (that the two stochastic
continuations produce separated row pairs with positive probability,
together with control of the RMS denominator); wherever such a bound is
invoked it is under that stated assumption. The pairwise statement is
consistent with the sub-critical phenomenology
of \cite{gokden2026}.
\end{enumerate}
\end{proposition}

\begin{proof}
(i) is the sub-multiplicativity of Lipschitz constants under composition
applied to the diameter of the image. (ii): every implemented input row of every head lies in $\mathcal{S}$, and for any fixed
$u_0 \in \mathcal{S}$ every output lies within
$\operatorname{diam}\varphi(\mathcal{S}) \le
\kappa^{N_{\mathrm{res}}}\operatorname{diam}\mathcal{S}$ of
$\alpha^\ast = \varphi(u_0)$ (a set of diameter $D$ need not lie in a
ball of radius $D/2$ in dimension $> 1$, but any of its points serves as
a center at radius $D$); the same $\varphi$ serves all heads of the
layer (Proposition~\ref{prop:rowfact}(iii)). The transfer display follows by inserting
$\rho_\varepsilon(r_{x,i})$ and $\rho_\varepsilon(r_{y,i})$ between
$u_{x,i}$ and $u_{y,i}$, applying the triangle inequality, then the
$L_\rho$ and $L_\varphi$ bounds. (iii) is immediate from the lower bound.
\end{proof}

Whether the trained map is in fact contractive on the visited tube is
an empirical question about the trained weights, and the quantity that
controls diameters is the \emph{composition}, not the individual units:
per-unit Lipschitz numbers do not compose informatively (singular-vector
alignment matters), and no product of per-unit summaries is a
measurement of the composite. Appendix~\ref{app:numerical} therefore
reports, for a released checkpoint, the largest singular values of the
Jacobian of the \emph{full} row map $\varphi$ at rows visited by every
audit prompt, together with empirical pairwise contraction ratios
$\norm{\varphi(r)-\varphi(r')}/\norm{r-r'}$ within and across prompts.
These are sampled pointwise statistics on visited data: they can support
or refute the contraction hypothesis on the sample, and they do not
constitute a tube-uniform Lipschitz certificate, which would require a
maximum (or a justified high-probability bound) over a specified tube
and remains future work.

\begin{corollary}[End-to-end invariance budget, with evaluable constants]
\label{cor:budget}
Suppose the visited inputs confine the entries of the preactivation
$WA + b_W$ to a compact interval $[-U, U]$ and the entries of $\Alm$ to
$[m_A, M_A]$ with $m_A \ge \epsilon$. (Such $U$ exists on the visited
domain: rows of $A$ lie in the image under the continuous $\varphi$ of
the compact containing set for LayerNorm of Lemma~\ref{lem:ln}(iii), and a priori
$U \le \norm{W}_\infty \sup_{\text{visited}}\norm{A}_{\max}\dk +
\norm{b_W}_{\max}$; sharper, measured values are used in
Appendix~\ref{app:numerical}.) Then the parameter maps
\eqref{eq:metric}--\eqref{eq:ec} satisfy an explicit Lipschitz
estimate in two \emph{named} regimes, distinguished by the lower
endpoint used for $m_A$: with the \emph{measured} minimum entry of
$\Alm$ on the visited states the constant is $C_G^{\mathrm{samp}}$,
valid for every \emph{pair of visited states} (stagewise scalar
mean-value bounds at the endpoint values, whose connecting segments
stay inside the measured ranges); with the architectural floor
$m_A = \epsilon$ it is $C_G^{\mathrm{tube}}$, which bounds the
derivative along \emph{arbitrary} perturbation paths within the
stated bounded domain (the upper bounds $M_A$ and $U$ are hypotheses
along the path), the only
certificate valid off the visited sample. Both share the expression
\[
C_G \;=\; \norm{a}_2 \cdot
\max_{ij}\Bigl(\abs{P_{ij}}\max\bigl(m_A^{P_{ij}-1},
M_A^{P_{ij}-1}\bigr)\Bigr)\cdot
\sup_{\abs{u} \le U}\abs{\iswiglu'(u)}\cdot \norm{W}_2 ,
\]
\enlargethispage{3pt}in which every factor is finite,\footnote{The restriction of the
$\iswiglu$-derivative supremum to $[-U,U]$ is essential:
$\iswiglu'(u) \to 2u$ as $u \to +\infty$, so the unrestricted supremum is
infinite and a constant built from it would be vacuous. The
zero-crossing caveat attaches to $C_G^{\mathrm{tube}}$ \emph{only}:
along a continuous input-space path between two visited states the
$\iswiglu$ preactivation can cross zero, where the only lower bound
on $\Alm$ entries is the architectural floor $\epsilon = 10^{-9}$,
at which exponents $P_{ij} < 1$ make $m_A^{P_{ij}-1}$ astronomically
large. It does \emph{not} invalidate $C_G^{\mathrm{samp}}$ for pairs
of visited states: the stagewise argument compares endpoint values
within each stage's measured range, not along any input-space path.
Appendix~\ref{app:numerical}
evaluates both named constants and labels each accordingly.}
with the norms of the conclusion fixed explicitly. The derivative
argument proves the Frobenius-to-Frobenius inequality
\[
\norm{\Delta\Glm}_F \;\le\; C_G\,\norm{\Delta A}_F
\]
(an entry-wise scalar map with derivative bounded by $L$ is
$L$-Lipschitz in Frobenius norm, and left multiplication by $W$ or by
$a$ has induced Frobenius operator norm at most the multiplier's
spectral norm). If the compared contexts satisfy the uniform per-row
bound $\max_i \norm{\Delta A_{i,:}}_2 \le \varepsilon_A$ (the form
delivered by the left side of Proposition~\ref{prop:contract}(ii)'s
transfer display), row aggregation gives
$\norm{\Delta A}_F \le \sqrt{\dk}\,\varepsilon_A$, hence
\[
\norm{\Delta\Glm}_2 \;\le\; \norm{\Delta\Glm}_F
\;\le\; C_G\,\sqrt{\dk}\;\varepsilon_A \;=:\; \varepsilon_G ,
\]
and by Proposition~\ref{prop:perturb} (whose hypothesis is the
spectral norm) the logits of cached and uncached inference differ by
at most $C(\theta)\,\varepsilon_G$.\footnote{The $\sqrt{\dk}$
aggregation cannot be dropped in general: for
$\Delta A = \mathbf{1}v^\top$ with $\norm{v}_2 = 1$, every row obeys
the per-row bound with $\varepsilon_A = 1$ while
$\norm{\Delta A}_2 = \norm{\Delta A}_F = \sqrt{\dk}$ (an
all-rows-equal perturbation that a row-combining $W$ transports
undiminished). A sharper constant would
require a proved $\ell_{2,\infty}$-to-spectral mixed-norm analysis of
the composite map, not undertaken here.
Appendix~\ref{app:numerical} carries the factor $\sqrt{\dk} = 8$
explicitly wherever this conversion is used; its assembled
end-to-end chain enters at a different point, a directly
hypothesized spectral radius
$\norm{\Delta\Glm}_2 = \varepsilon_{\mathrm{spec}}$, and is
unaffected by the conversion.} Two further conversions are required
before comparing with measured quantities: reported
relative-RMS invariances $\varepsilon_{\mathrm{rms}}$ enter the spectral
hypothesis of Proposition~\ref{prop:perturb} only through
$\norm{G - G'}_2 \le \norm{G - G'}_F = \dk\,\norm{G - G'}_{\mathrm{rms}}$
times the normalization scale; and a bound $B$ on the logit perturbation
forces bit-identical \emph{greedy} decoding under the sufficient
condition $2B < \Delta$, where $\Delta$ is the minimum realized top-two
logit margin along the decoded paths: the factor $2$ is necessary in
this criterion, since the leading logit can move down by $B$ while the
runner-up moves up by $B$ (for an $\ell_2$ bound the two-coordinate
sufficient comparison is $\sqrt{2}\,B < \Delta$; for
stochastic sampling, only a distributional bound under coupled
randomness follows). This budget is a certificate only to the extent
that its evaluated value closes the halved margin \emph{and} its
factors are certified on a domain containing every compared state; on
the released
checkpoint audited in Appendix~\ref{app:numerical} the chain, with
every factor evaluated as a sample-extrema proxy (certified analytic
derivative envelopes; all other extrema sampled over the audit
prompts),
assembles to a coefficient tens of orders of magnitude too large to
certify decoding decisions, and the honest summary of cache fidelity
remains
\emph{empirically indistinguishable under tested decoding}, with the
budget quantifying the mechanism rather than certifying the outcome.
\end{corollary}

\begin{proof}
Compose the Lipschitz constants of the three maps on the stated domain:
$A \mapsto WA + b_W$ contributes $\norm{W}_2$; the entry-wise
$\iswiglu$ contributes its derivative bound on $[-U,U]$; the entry-wise
power \eqref{eq:hadpower} has derivative $P_{ij} u^{P_{ij}-1}$,
extremized at the ends of $[m_A, M_A]$; the left action of $a$
contributes $\norm{a}_2$; biases drop out of differences. Finiteness of
each factor is by compactness of the stated domain.
\end{proof}

\subsubsection{Why training at criticality might select the constant map:
a hypothesis}

The preceding results reduce the invariance question to: \emph{why does
gradient training drive the row map $\varphi$ into the contractive,
constant-output regime?} A derivation from the training dynamics is open
(Section~\ref{sec:discussion}, Problem 1). We state the proposed
selection mechanism explicitly as a hypothesis, with the epistemic status
of each ingredient tagged; it is consistent with the observations cited
but is not derived from them.

\begin{hypothesis}[Selection of the constant map at criticality]
\label{hyp:selection}
\leavevmode
\begin{enumerate}[leftmargin=2em]
\item \textbf{Expressivity permits it} (proved, for the collapsed
instance). By Theorem~\ref{thm:collapse}(ii) and
Remark~\ref{rem:asymmetry-scope}, replacing an exactly constant
$\Glm$ by that constant preserves the trained model's inference map.
Approximate replacement instead has the error budget of
Proposition~\ref{prop:perturb}. These are statements about
models already satisfying the invariance or error assumptions; they
neither establish that training reaches such models nor compare their
attainable loss with all input-sensitive models.
\item \textbf{The gradient signal favors it} (heuristic). The proposed
mechanism is that batch averaging reduces instance-specific gradient
forcing. Proposition~\ref{prop:conc} controls the normalized
Gram under its independent-query law. Transfer to the implemented input
retains the LayerNorm target-transfer defect $\delta_{x,i}$ in the
unnumbered target-transfer estimate following Lemma~\ref{lem:ln};
transfer to gradients additionally requires regularity of the loss and
row-map derivatives. Thus an $O(S^{-1/2})$ relative gradient leverage is a
conditional heuristic, not a consequence of Gram concentration alone.
No measurement establishing this particular gradient scaling is reported here.
\item \textbf{Stability at criticality enforces it} (heuristic). An
input-sensitive $\varphi$ (Proposition~\ref{prop:contract}(iii)) is
argued to have larger Jacobians through the deductive path, hence larger
loss curvature in the PLGA parameters, so that at the large maximum
learning rates of the near-critical regime such directions are annealed
away or trigger dragon-king events \cite{gokden2026,sornette2012}, while
at small learning rates (sub-critical) input-sensitive solutions survive
and the order parameter grows, as observed \cite{gokden2026}. The
curvature claim is not established; input sensitivity does not in
general imply larger loss curvature.
\end{enumerate}
\end{hypothesis}

Three observations motivated this hypothesis, without selecting it
uniquely over alternatives:
(a) the \emph{cross-head identity} of $A$ \cite{gokden2025} is naturally
explained by collapse in the shared row map (Stage 3) and not by
head-specific statistics, and the Stage-3 reading is directly supported
on the audited checkpoint by the composite-Jacobian and pairwise
measurements of Appendix~\ref{app:numerical}; (b) invariance \emph{improves with training data}
($m \to 0$ at float resolution for the $41$B-token model
\cite{gokden2026}), consistent with longer annealing of Stage 3 (note
that training-token count is distinct from the context length $S$ that
drives Stage 1--2 concentration); and (c) DAG
regularization, which constrains the downstream tensors and hence deforms
the attractor, measurably \emph{increases} the caching perturbation
\cite{gokden2025}. The token-volume comparison is confounded by the
schedule, and changing DAG regularization changes the loss itself;
neither isolates criticality as the cause.

\section{Power Laws, Scale Invariance, and Self-Organized Criticality}
\label{sec:criticality}

\subsection{Why power laws: the unique scale-equivariant interaction}

\begin{proposition}[Scale covariance forces power laws]
\label{prop:scalecov}
Let $f : (0,\infty) \to (0,\infty)$ be measurable and suppose there exists
$g$ with $f(\lambda u) = g(\lambda) f(u)$ for all $\lambda, u > 0$
(\emph{scale covariance}: rescaling the input rescales the output
independently of $u$). Then there exist $c > 0$ and $p \in \R$ with
$f(u) = c\,u^{p}$ and $g(\lambda) = \lambda^p$.
\end{proposition}

\begin{proof}
Setting $u=1$: $f(\lambda) = g(\lambda) f(1)$, so $g = f/f(1)$ and $g$
satisfies $g(\lambda\mu) = g(\lambda)g(\mu)$. Every measurable solution of
the multiplicative Cauchy equation on $(0,\infty)$ has the form
$g(\lambda) = \lambda^p$
for some real $p$ \cite[Ch.~2]{aczel1966}. Then $f(u) = f(1) u^p$.
\end{proof}

\begin{corollary}[Elementwise scale covariance]
\label{cor:whypower}
Among measurable \emph{element-wise} maps of a single positive entry, the
family $u \mapsto c\,u^{p}$ used in \eqref{eq:potential} is, up to the
learned constants, the only one that transforms covariantly under
rescaling of its argument. In this element-wise sense the potential stage
is the minimal scale-covariant interaction ansatz.
\end{corollary}

\begin{remark}[What the corollary does not say]
\label{rem:scalefree-scope}
Corollary~\ref{cor:whypower} is a statement about one scalar entry at a
time. It does \emph{not} make the architecture globally scale-covariant:
under a global rescaling $\Alm \mapsto \lambda\Alm$ the entries of $\Ap$
transform as $\lambda^{P_{ij}}(\Ap)_{ij}$ with heterogeneous exponents,
which is not a common covariance law unless the relevant $P_{ij}$
coincide or a vector-valued group action is specified; and the additive
biases $b_W, b_a$, the $\epsilon = 10^{-9}$ floor, the LayerNorm
$\varepsilon_{\LN}$, and the mixing step all introduce scales. For the
same reason we refer to the learned exponents $P_{ij}$ as \emph{learned
exponents of an elementwise power feature map}; calling them physical
``scaling dimensions'' would require an identified symmetry action on
inputs and outputs, an identifiability/gauge analysis, and empirical
covariance under that action, none of which is currently available (the
scaling-dimension reading is used below only as an explicitly labeled
analogy). Likewise, the phrase ``scale-free attention'' is used in this
paper only for the element-wise potential stage, not for the
architecture as a whole. Scale-free interactions are the signature of
critical systems \cite{stanley1999,newman2005}, of the renormalization
group at fixed points \cite{wilson1974,goldenfeld1992}, and of natural
language statistics \cite{zipf1949}; the connection of PLGA to them is
by construction of the potential stage, and by analogy beyond it.
\end{remark}

\subsection{Criticality of the attention dynamics: a conditional spectral
dictionary}

The claim ``the correlation length diverges at criticality''
\cite{gokden2026} suggests an operator reading through the Markov
structure of Proposition~\ref{prop:markov}. Making it precise requires
care on three points. (a) A fixed finite matrix has an atomic spectral
measure and cannot carry a density accumulating at $1$, so the critical
case can only concern an infinite-dimensional operator or a limit of a
family. (b) A row-stochastic matrix is generally nonnormal, so
eigen-expansions require eigenbasis conditioning and the stationary
inner product. (c) The $t$-th spectral moment is the autocorrelation
$\ip{f}{E^t f}$, not the norm $\norm{E^t f}$, whose decay exponent
differs by a factor of two. The statement below builds all three into
its hypotheses: it is a \emph{dictionary lemma} for a family of
reversible operators, and its application to PLDR-LLM is
Conjecture~\ref{conj:spectral}.

\begin{proposition}[Spectral gap vs.\ critical slowing down, for
reversible families]
\label{prop:gap}
Let $(E_d)_{d}$ be a family of Markov operators, each reversible with
respect to a stationary distribution $\pi_d$ (hence self-adjoint on
$L^2(\pi_d)$ with real spectrum in $[-1,1]$), and let $f_d$ be
observables with $\pi_d$-mean zero, normalized in $L^2(\pi_d)$, with
spectral measures $\mu_{f_d}$ (with respect to $E_d$).
\begin{enumerate}[label=(\roman*),leftmargin=2.2em]
\item (Gapped/off-critical.) If the spectral edges are uniformly
controlled, $\operatorname{supp}\mu_{f_d} \subseteq [-1+\delta', 1-\delta]$
with $\delta, \delta' > 0$, write
$c_d(t)=\ip{f_d}{E_d^t f_d}_{\pi_d}$ for integer $t\ge0$. Then
\[
 |c_d(t)|\le r^t\quad(t\ge1),\qquad
 r=\max(1-\delta,1-\delta'),\qquad 0\le r<1.
\]
For $0<r<1$, the exponential envelope is $r^t=\exp(-t/\xi)$ with
finite scale $\xi=-1/\log r$. If $r=0$, all positive-time
autocorrelations vanish, and we set $\xi=0$.
The normalized zero-time value remains $c_d(0)=1$.
The parameter $\xi$ describes the uniform envelope supplied by the bound;
it need not be the sharp scale of each observable. \emph{Both} edges enter: reversibility does not exclude
spectrum near $-1$, and an eigenvalue there produces slowly decaying
alternating correlations even under a large gap at
$+1$.\footnote{Omitting the lower edge leads to a genuine error, not a
technicality: for the reversible two-state chain
$E = \bigl(\begin{smallmatrix}0.05&0.95\\0.95&0.05\end{smallmatrix}\bigr)$
with uniform $\pi$, the normalized mean-zero observable has eigenvalue
$-0.9$ and autocorrelation $(-0.9)^t$, so with $\delta' = 0.1$,
$\delta = 0.5$ the one-sided bound $(1-\delta)^t$ fails already at
$t = 1$ ($0.9 \not\le 0.5$). This is near-periodicity, not
nonnormality.}
\item (Critical.) Assume additionally a uniform gap at the lower edge:
$\operatorname{supp}\mu_{f_d} \subseteq [-1+\delta_0, 1]$ for a fixed
$\delta_0 > 0$ independent of $d$ (automatic for lazy or positive
semidefinite families, whose spectrum lies in $[0,1]$), and that
$\mu_{f_d} \to \mu$ weakly, so that \emph{in the iterated sense that
the family limit is taken first at each fixed $t$ and the
$t \to \infty$ asymptotics are those of the limit measure},
$\lim_{d\to\infty}\ip{f_d}{E_d^t f_d} = \int \lambda^t\, d\mu$. Gap
closing with power law accumulation at the upper edge (necessarily
along the family: each finite-$d$ measure is atomic) then yields
power law decay, at two levels of hypothesis:
\begin{enumerate}[label=(ii.\alph*),leftmargin=2.4em]
\item (Comparability.) If
$d\mu(1-s) \asymp s^{\vartheta-1}\,ds$ as $s \downarrow 0$ (two-sided
bounds with unspecified positive constants), then
\[
\int \lambda^t\, d\mu \;=\; \Theta\bigl(t^{-\vartheta}\bigr),
\qquad
\Bigl(\int \lambda^{2t}\, d\mu\Bigr)^{1/2}
\;=\; \Theta\bigl(t^{-\vartheta/2}\bigr),
\qquad t\to\infty:
\]
the decay \emph{exponent} is determined, the coefficient is not, and
no asymptotic equivalent is implied (comparability constants may
oscillate).
\item (Exact edge density.) If moreover
$d\mu(1-s) = (c + o(1))\, s^{\vartheta-1}\,ds$ as $s \downarrow 0$ for
some constant $c > 0$, with no singular component at the edge, then
\[
\int \lambda^t\, d\mu \;\sim\; c\,\Gamma(\vartheta)\, t^{-\vartheta},
\qquad
\Bigl(\int \lambda^{2t}\, d\mu\Bigr)^{1/2} \;\sim\; c_\vartheta\,
t^{-\vartheta/2},
\qquad
c_\vartheta = \sqrt{c\,\Gamma(\vartheta)}\; 2^{-\vartheta/2}.
\]
\end{enumerate}
In either case correlations decay as a \emph{power law} with no
characteristic scale; $\xi = \infty$. (The autocorrelation and the
norm decay with different exponents; both are recorded to prevent
conflation. No uniformity in $t$ of the family convergence is claimed;
without the lower-edge gap the conclusion fails, e.g.\ for an atom at
$-1$.)\footnote{The two-level split is forced: under
comparability alone the exact equivalent is false (a normalized
measure whose edge density is $2s^{\vartheta-1}$ satisfies
$d\mu(1-s) \asymp s^{\vartheta-1}ds$ yet contributes
$2\Gamma(\vartheta)t^{-\vartheta}$, and oscillating comparability
constants can prevent any asymptotic equivalent from existing), so
(ii.a) claims only $\Theta$, and the exact constants live in (ii.b),
where the density hypothesis supports them.}
\end{enumerate}
\end{proposition}

\begin{proof}
Both parts are the spectral calculus of a self-adjoint contraction:
$\ip{f}{E^t f} = \int \lambda^t\, d\mu_f(\lambda)$ and
$\norm{E^t f}^2 = \int \lambda^{2t}\, d\mu_f(\lambda)$. (i): on the
support, $\abs{\lambda} \le \max(1-\delta, 1-\delta') = r$, so
$\abs{\int \lambda^t d\mu_f} \le r^t$. (ii): for fixed $t$,
$\lambda \mapsto \lambda^t$ is bounded and continuous on $[-1,1]$, so
weak convergence of the compactly supported $\mu_{f_d}$ gives the limit
$\int \lambda^t d\mu$. Splitting $\mu$ at $1 - \eta$ for any fixed
$\eta \in (0, \min\{\delta_0, 1\})$: the contribution of
$[-1+\delta_0, 1-\eta]$ is bounded in absolute value by
$\max(1-\eta, 1-\delta_0)^t$, exponentially small. Near the upper edge
the substitution $\lambda = 1-s$ gives Beta-type integrals:
$\int_0^\eta (1-s)^t s^{\vartheta-1} ds \sim \Gamma(\vartheta)
t^{-\vartheta}$ (via
$B(t+1,\vartheta) = \Gamma(t+1)\Gamma(\vartheta)/\Gamma(t+1+\vartheta)
\sim \Gamma(\vartheta)\,t^{-\vartheta}$, Stirling). (ii.a): the
hypothesis sandwiches the edge contribution between $c_1$ and $c_2$
times this integral for some $0 < c_1 \le c_2$, giving the $\Theta$
bounds and nothing stronger. (ii.b): writing the edge density as
$(c + o(1))s^{\vartheta-1}$ and splitting off the $o(1)$ factor at a
radius where it is uniformly small, dominated convergence carries the
constant through, $\int \lambda^t d\mu \sim c\,\Gamma(\vartheta)
t^{-\vartheta}$. The norm computations replace $t$ by $2t$; in (ii.b)
this gives $c\,\Gamma(\vartheta)2^{-\vartheta}t^{-\vartheta}$ inside
the square root, i.e.\ the displayed $c_\vartheta$.
\end{proof}

Whether this dictionary describes PLDR-LLM is an open question, for two
reasons: trained attention matrices are neither reversible nor constant
across layers and tokens (repeated powers of a single $E$ do not model a
depth-$L$ network), and no spectral measurement of trained operators
near criticality has been published. We therefore record the intended
application as Conjecture~\ref{conj:spectral} and use the following
correspondence only as interpretive language: \emph{sub-critical
training} $\leftrightarrow$ gapped learned operators, exponential decay
of influence; \emph{critical training} $\leftrightarrow$ gap closing
with power law spectral accumulation, scale-free propagation of
constraints across the context.

\subsection{The SOC training picture as a phenomenological framework}
\label{subsec:soc}

Following \cite{gokden2026}, pretraining is modeled as a slowly driven
dissipative system in the sense of \cite{bak1988,dickman2000}. We
emphasize the epistemic status before the definitions: what the published
experiments establish is \emph{critical-like optimizer phenomenology
associated with low sampled deductive-output fluctuation}, observed in
single training runs per condition with externally tuned schedules.
Establishing self-organized criticality in the technical sense would
additionally require an identified self-tuning feedback mechanism,
separation of drive and relaxation scales, and standard discriminants
(finite-size scaling and data collapse, susceptibility, avalanche or
$1/f$ statistics), none of which has been measured; see the empirical
program in Section~\ref{sec:discussion}. The definitions below therefore
fix the \emph{vocabulary} of the source papers as a phenomenological
framework, not as established physics.

\begin{definition}[Control and order parameters of PLDR-LLM pretraining;
phenomenological]
\label{def:soc}
Let $\eta_{\max}$ be the maximum learning rate and $T_w$ the linear warm-up
step count of the schedule (cosine annealing to $0.1\,\eta_{\max}$). The pair
$(\eta_{\max}, T_w)$ are the \emph{control parameters}: token batches under
forward propagation are the slow external drive; gradient updates under
backward propagation are the dissipation. A trained model is:
\begin{itemize}[leftmargin=2em]
\item \emph{near-critical} in the source classification when sampled
deductive fluctuations are small while text generation is non-degenerate
(the published mean-normalized statistic is empirically $\lesssim10^{-2}$,
typically $\le10^{-5}$);
\item \emph{sub-critical} in that classification when fluctuations are
of order one or larger and the reported generation and benchmark scores
deteriorate despite lower blockwise loss \cite{gokden2026};
\item subject to \emph{dragon-king events} \cite{sornette2012} (sharp loss
spikes from self-amplifying drive/dissipation imbalance), which mark
departures from power law criticality and degrade the final state even when
the loss trajectory recovers \cite{gokden2026}.
\end{itemize}
These labels describe the reported source conditions. The RMS-normalized
diagnostic of \eqref{eq:orderparam} has a different denominator and must
be measured under its own definition; the source thresholds do not
calibrate it.
\end{definition}

The phenomenology reported in \cite{gokden2026} is \emph{consistent with}
a second-order-transition reading: near-critical models across a range of
$(\eta_{\max}, T_w)$ collapse onto nearly identical loss trajectories;
the deductive outputs reach a metastable steady state, the
$\varepsilon$-invariant operators of Section~\ref{sec:invariance}; and
small values of the source mean-normalized diagnostic separate the
observed phases in agreement with benchmark rankings at the phase level (not strictly
monotonically within phases; see after
Remark~\ref{rem:cacheable-converse}). Three caveats bound what this
shows. The near/sub-critical labels are substantially defined through the
source mean-normalized diagnostic and generation quality they are then
used to explain, so
an independent phase criterion fixed in advance is needed to break the
circularity; the $(\eta_{\max}, T_w)$ pairs are externally selected, so
the evidence shows schedule-tuned critical-like behavior rather than
self-organization; and each condition has a single training run with no
uncertainty quantification. Within those limits, the parallel with
criticality hypotheses for cortical dynamics \cite{beggs2003,hesse2014}
and with the ubiquity of SOC in natural systems \cite{markovic2014}
motivates, but does not support beyond motivation, the conjecture that
invariant operators transfer across domains within a universality class
(Conjecture~\ref{conj:functorial}).

\begin{remark}[RG reading of the architecture; analogy]
The composition depth of the model implements a sequence of
coarse-grainings of the token field (the layer maps are distinct; no
semigroup or closed composition family is claimed); the learned
exponents $P$ play the role of scaling
dimensions in the analogical sense of
Remark~\ref{rem:scalefree-scope}; invariance of $\Glm$ under change
of instance is fixed-point behavior under the ``flow'' of data; and
universality (insensitivity of the critical trajectory to microscopic
details like tokenizer choice, observed in the ablations of
\cite{gokden2024}) is the hallmark of an RG fixed point
\cite{wilson1974,goldenfeld1992}. We do not claim a derived RG map; we
record the correspondence because it is the frame in which the
operator-transfer conjecture (Conjecture~\ref{conj:functorial}) is natural.
\end{remark}

\section{Architectural comparison with scaled dot-product attention}
\label{sec:advantages}

Theorem~\ref{thm:collapse}(i) places the two architectures inside a single
family: an SDPA-LLM is the point of PLDR-LLM model space at which the
operator $\Glm$ is pinned to the identity. The comparison in this
section is therefore not between rival designs but between a family and its
base point, and the one-sentence form of the whole comparison is:
\emph{PLDR-LLM learns and exposes an operator sector that SDPA fixes in
head space, at additional training cost}. Every advantage below is a
statement about what the \emph{learned}
operator sector buys relative to the \emph{frozen} one, and every claim
is tied to a result of this paper \emph{at that result's epistemic
level} (several are conditional or prospective, and are so marked) and,
where available, to an experimental anchor in the source papers. We state
the costs with equal explicitness (\S\ref{subsec:costs}). The comparison
is summarized first:

\begingroup
\small
\begin{longtable}{>{\raggedright\arraybackslash}p{0.30\textwidth}>{\raggedright\arraybackslash}p{0.30\textwidth}>{\raggedright\arraybackslash}p{0.30\textwidth}}
\caption{Architectural comparison with scaled dot-product attention.}\label{tab:architecture-comparison}\\
\toprule
Advantage & Formal basis & Empirical anchor \\
\midrule\endfirsthead
\caption[]{Architectural comparison with scaled dot-product attention (continued).}\\
\toprule
Advantage & Formal basis & Empirical anchor \\
\midrule\endhead
\midrule
\multicolumn{3}{r}{\footnotesize Continued on the next page.}\\
\endfoot
\bottomrule\endlastfoot
Larger parameterized family with distinct training dynamics &
Thm.~\ref{thm:collapse}(i),(iii); Rem.~\ref{rem:asymmetry-scope};
\S\ref{subsec:bilinear} regime 3 &
Learned $\Glm$ outperforms identity/random/transferred constants under
matched training \cite{gokden2025} \\
Head-level positional operator codimension, present even after caching
(conditional) &
Cor.~\ref{cor:posgap}; Prop.~\ref{prop:rope-commutant} &
Occupancy measured in Appendix~\ref{subsec:commutant-audit}; causal
benefit of projection remains untested \\
Inspectable operator representation with proved structure &
\S\ref{sec:plga} propositions; Cor.~\ref{cor:whypower} &
DAG loss as metric/regularizer improves benchmarks without scaling
\cite{gokden2024} \\
Intrinsic evaluation diagnostic (prospective) &
Def.~\ref{def:orderparam}; Prop.~\ref{prop:cacheable};
Rem.~\ref{rem:vacuous} &
Source mean-normalized diagnostic separates labeled phases and agrees
with benchmarks at phase level \cite{gokden2026}; RMS calibration open \\
Phase-aware, self-instrumented training &
Def.~\ref{def:soc}; Prop.~\ref{prop:gap} (conjectural reading) &
Loss curves fail to discriminate phases; the source mean-normalized
diagnostic separates them in the published sample \cite{gokden2026} \\
Fast inference via caching exact relative to KV-cache; PLGA weights
omittable from a fixed-inference package &
\S\ref{subsec:cache}; Prop.~\ref{prop:perturb}; Cor.~\ref{cor:budget}
(bounds evaluated, App.~\ref{app:numerical}) &
Reported $\sim3\times$ speedup; unchanged benchmark scores use one-pass
block scoring \cite{gokden2025} and do not validate generation-time
caching. Operator/logit evidence: App.~\ref{app:numerical}; cross-stack
speed comparison remains confounded \\
Deployment asymmetry (proposed; no threat model analyzed) &
Thm.~\ref{thm:collapse}(iii) &
Proposed in \cite{gokden2025} \\
Elementwise scale-covariant stage; transfer program &
Cor.~\ref{cor:whypower}; Rem.~\ref{rem:scalefree-scope};
Conj.~\ref{conj:functorial} &
Prospective (Conjecture~\ref{conj:functorial}) \\
\end{longtable}
\endgroup

\subsection{A larger parameterized family with distinct training
dynamics}
\label{subsec:adv-family}

Three claims must be kept separate here. (1) \emph{Algebraic inclusion}
(proved, machine-checked): SDPA is exactly the $\Glm = I$ point of the
family (Theorem~\ref{thm:collapse}(i)). (2) \emph{Distinct training
dynamics} (generic): at training time PLGA operates in regime 3 of
\S\ref{subsec:bilinear}, the score map is generically nonlinear in the
input (degenerate choices such as $a = \bm0_{\mathrm{mat}}$ linearize it,
\S\ref{subsec:bilinear}), and gradients flow through the metric
learner and the power law
parameters, so by Theorem~\ref{thm:collapse}(iii), \emph{where its
nonvanishing hypothesis holds}, the two
parameterizations induce different gradient flows. (3) \emph{Strict
function-class containment at matched resources} (not proved): different
training dynamics do not show that no SDPA-LLM realizes the same
input--output function; a nonrepresentability theorem at fixed
depth/width/positional scheme would be required, and we do not have one
(Remark~\ref{rem:asymmetry-scope}). The claims below rest on (1) and
(2) only. The ablations of \cite{gokden2025} instantiate (2)
experimentally: under matched data and schedules, the model
with a learnable $\Glm$ attains the best average benchmark scores
(under the published one-pass block protocol,
Appendix~\ref{app:repos}), ahead
of models trained with the operator frozen to a transferred, identity
(i.e.\ SDPA-equivalent), or random constant, and its loss trajectory is
distinct, tracking the transferred-operator model early in training and
the identity-operator model late, a signature of the learned operator
moving through the family rather than sitting at a fixed point of it.

Even at \emph{inference}, after the collapse of
Theorem~\ref{thm:collapse}(ii), the cached model's raw score functions
generically remain outside those realizable by an SDPA head, and the
difference can be counted at the level of head-space operators:

\begin{corollary}[Head-level positional operator codimension;
conditional]
\label{cor:posgap}
Assume the nonresonance hypothesis of
Proposition~\ref{prop:rope-commutant}, and consider score functions
$(n,m,q,k) \mapsto q^\top R_{-n} G R_m k$ with the arguments $q, k$
ranging over all of $\R^{\dk}$ (\emph{unrestricted-reachability
assumption}) and the RoPE parameterization fixed. Then the map
$G \mapsto \bigl[(n,m,q,k) \mapsto q^\top R_{-n} G R_m k\bigr]$ from
operators to score functions is injective, and the frozen-$G$ PLGA head
realizes an SDPA-realizable score function (for some linear
reparametrization of the projections applied before rotation) if and only
if $G$ lies in the RoPE commutant. Consequently, \emph{inside} the
$\dk^2$-dimensional frozen-$G$ PLGA family, the score functions also
realizable by pre-RoPE-linearly-reparameterized SDPA form exactly the
commutant subfamily, of real dimension $2 \cdot \dk/2 = \dk$, hence
have codimension
\[
\dk^2 - \dk \quad\text{real dimensions per head}
\]
within it, all complementary directions being
absolute-position-sensitive. This is a codimension statement about the
SDPA-realizable \emph{intersection} inside the frozen-$G$ family (no
nesting of the unrestricted global PLGA and SDPA families is asserted)
and an \emph{operator-dimension count for a single head
under the stated assumptions}, not a function-class separation between
complete LLM families: learned projections restrict the reachable
$(q,k)$, prior layers and alternative positional encodings can simulate
score functions, and no full-model expressivity claim is made.
\end{corollary}

\begin{proof}
Injectivity: if $q^\top R_{-n}(G - G')R_m k = 0$ for all $q, k$ and all
$n, m$, then $R_{-n}(G-G')R_m = \bm0_{\mathrm{mat}}$ for some (hence all) $n,m$, so
$G = G'$. Realizability: suppose linear pre-rotation
reparameterizations $A, B$ of the query and key projections realize
the frozen-$G$ score family, i.e.
\[
(R_n A q)^\top (R_m B k) \;=\; q^\top R_{-n}\, G\, R_m\, k
\qquad\text{for all } n, m \text{ and all } q, k \in \R^{\dk}.
\]
Equality of the bilinear forms gives
$A^\top R_{m-n} B = R_{-n} G R_m$ for all $n, m$. At $n = m = 0$ this
reads $A^\top B = G$; at $m = n$ the left-hand side is again
$A^\top B$, so $R_{-n} G R_n = G$ for every $n$: $G$ lies in the
commutant. Conversely, if $G$ commutes with every rotation, the
choice $A=G^\top$, $B=I$ realizes the family, since
$A^\top R_{m-n} B = G R_{m-n} = R_{-n} G R_m$. No invertibility of
$A$ or $B$ is assumed. The commutant
is $\bigoplus_j \{c_j I_2 + s_j J_2\}$
(Proposition~\ref{prop:rope-commutant}), of real dimension $\dk$, inside
$\R^{\dk\times\dk}$ of dimension $\dk^2$; position sensitivity of the
complement is Proposition~\ref{prop:rope-commutant}'s converse direction.
\end{proof}

The choice of projections is essential. The product identity
$A^\top B=G$ alone does not imply the required intertwining relation
$A^\top R_{m-n}B=R_{-n}GR_m$.

For the reference head width $\dk = 64$ this is $4032$ additional
head-level score-function dimensions (under the corollary's assumptions)
that an SDPA head with RoPE and linear reparametrization before rotation
cannot express. The count itself is an arithmetic consequence of the
rotation-commutant principle discussed in
Remark~\ref{rem:rope-twice} (established in the RoPE literature,
\cite{wangwang2025,tran2026}); what is specific to PLGA is the object
being counted: the inserted head-space operator sector that the
architecture generates, exposes, and can cache.
The relative off-commutant Frobenius residual measures checkpoint-scoped
occupancy of these absolute-position-sensitive operator directions, as
reported in Appendix~\ref{subsec:commutant-audit}. Their predictive
contribution requires projecting $G^\ast$ onto the commutant and evaluating
the behavioral change, with replication across checkpoints and scales
(Open Problem 3 in Section~\ref{sec:discussion}). SDPA is not operator-free: in
pre-projection coordinates its scores carry the learned constant
bilinear form $W_Q W_K^\top$, and operator-level regularization or
diagnostics \emph{can} be posed for that object; PLGA's distinction is
that the head-space operator is input-conditioned, directly exposed,
and generated anew from the current input's
$\widetilde Q^\top \widetilde Q$ rather than factored into fixed learned
projections, which makes the corresponding questions direct rather than
reconstructive.

\subsection{An inspectable law representation}
\label{subsec:adv-law}

In an SDPA-LLM the dataset-level structure of attention is diffused
through the projection weights; the operator sector is, in the language of
\cite{gokden2026}, a hidden variable pinned to $I$ at all times. In a
PLDR-LLM the same structural role is played by exposed tensors with proved
properties: strictly positive $\Alm$ with Perron--Frobenius spectral
structure (Theorem~\ref{thm:pf}), a potential tensor generated by the
unique scale-equivariant elementwise interaction
(Corollary~\ref{cor:whypower}), a rank-one collapse with an explicit
algebraic profile (Proposition~\ref{prop:rankone}), and an (empirically)
invariant $G^\ast$ with a quantified caching perturbation
(Corollary~\ref{cor:budget}). Exposure is
not cosmetic; it is what makes the following \emph{operations} available
in direct form (for SDPA, operator-level analogues can at best be posed
for the fixed pre-projection form $W_Q W_K^\top$; there is no
input-conditioned, exposed operator to regularize, read, or
transplant):
\begin{itemize}[leftmargin=2em]
\item \textbf{Regularization and metrics on the operator sector.} The DAG
loss \eqref{eq:dagloss} monitors and shapes the cycle content of the
learned operators; applied as a regularizer it improved benchmark scores
(one-pass block protocol) over the unregularized base model without
scaling model or data
\cite{gokden2024}, and as a metric it separates models whose loss curves
are indistinguishable \cite{gokden2024}.
\item \textbf{Exponent readout.} The exponents $P$ and couplings $a$ are
directly readable from, and monitorable in, the trained model; reading
them as scaling laws of the training domain is the analogy of
Remark~\ref{rem:scalefree-scope}, testable but not established.
\item \textbf{Operator transplantation.} $G^\ast$ is a portable artifact:
it can be extracted from one model and installed in another
\cite{gokden2025}; this is the experimental substrate of the
operator-transfer program (Conjecture~\ref{conj:functorial}).
\end{itemize}

\subsection{An intrinsic evaluation diagnostic (prospective)}
\label{subsec:adv-eval}

The deductive-fluctuation diagnostic motivates the stable
RMS-normalized $m(\theta)$ of Definition~\ref{def:orderparam}.
In its original mean-normalized form it separates the near-critical and
sub-critical groups in the published sample, agreeing with benchmark
rankings at that group level \cite{gokden2026};
whether it resolves finer differences than the phase is open (the
within-phase reversals noted after Remark~\ref{rem:cacheable-converse}
are unresolved without uncertainty estimates), so as an evaluation tool
it is a diagnostic that separates the labeled conditions \emph{in the
published sample}, and a \emph{prospective} tool beyond it: validation
outside that sample (multiple seeds, uncertainty estimates, held-out
thresholds, an independent phase criterion) has not been carried out.
The diagnostic has no direct SDPA counterpart, for a reason worth
stating precisely:

\begin{remark}[The SDPA order parameter is degenerate]
\label{rem:vacuous}
For an SDPA-LLM the operator sector is constant by construction
($\Glm \equiv I$), so the operator-sector fluctuation that defines
$m(\theta)$ vanishes identically for every model, trained or untrained,
generalizing or not: the SDPA architecture satisfies the steady-state
condition vacuously and the order parameter carries zero information. In
a PLDR-LLM, invariance is an outcome of training observed in the
near-critical runs (Section~\ref{subsec:origin}; whether the regime is
critical in a technical sense is itself unestablished, \S\ref{subsec:soc},
and the strongest invariance observation, the $41$B-token model, is
confounded by training volume). Invariance can fail in PLDR-LLM, so the
operator fluctuation is informative about that sector. It is not by
itself an indicator of predictive competence.
\end{remark}

If the diagnostic survives the validation protocol of
Section~\ref{sec:discussion}, it may be useful where benchmark coverage is
limited, including small models, non-language domains, and data-limited
settings. In those settings the fluctuation statistic can guide measurement,
but it does not replace a specified predictive target or an external assessment.

\subsection{Phase-aware, self-instrumented training}
\label{subsec:adv-phase}

The SOC formalization (Definition~\ref{def:soc}) gives PLDR-LLM training
an internal phase diagnostic that SDPA training lacks. The empirical
record shows why this matters: sub-critical PLDR-LLMs achieve
\emph{lower} training loss while generating token salad, and models with
different tokenizers or warm-up schedules can be indistinguishable on
loss/accuracy while differing sharply in DAG loss and deductive behavior
\cite{gokden2024,gokden2026}. The loss is the standard intrinsic scalar
a pipeline monitors, and it fails to discriminate these phases; an
SDPA-LLM has other internal signals (attention entropy,
activation and gradient statistics, calibration, the pre-projection
form $W_Q W_K^\top$), but none of them is a fluctuation statistic of an
exposed operator sector. The source mean-normalized diagnostic, DAG
losses, and dragon-king events (Definition~\ref{def:soc}) provide
model-internal signals for the reported conditions; monitoring the
RMS-normalized $m(\theta)$ prospectively requires its own calibration. Proposition~\ref{prop:gap} together with
Conjecture~\ref{conj:spectral} supplies a candidate operator-level
meaning of the phases (gap closing versus gapped spectra). These are
candidate internal diagnostics. Their predictive calibration remains to
be established, and a spectral criticality interpretation additionally
requires testing Conjecture~\ref{conj:spectral}.

\subsection{Inference efficiency and a deployment asymmetry}
\label{subsec:adv-inference}

With KV-cache and G-cache enabled, the deep PLGA subnetwork is executed
once per prompt and then removed from the loop (\S\ref{subsec:cache});
empirically this yields a $\sim3\times$ speedup over the uncached model,
with aggregate deductive-output statistics stable to $15$ printed decimal
digits \cite{gokden2025}; the published cached-versus-uncached
benchmark evaluations are unchanged (scored under the one-pass block
protocol; see \S\ref{subsec:cache} and
Appendix~\ref{app:repos}). A fixed-inference deployment package \emph{can}
additionally omit the PLGA weights, since the cached path never
evaluates them; enabling caching by itself does not shrink the stored
parameter set, and the smaller artifact is a separate packaging step.
The reported $27$--$39\%$ speed
advantage over a comparably sized SDPA reference \cite{gokden2025} is a
cross-stack comparison (custom PLDR code versus a Hugging Face pipeline,
different kernels and generation wrappers) and is not attributable to
the attention architecture alone without a matched-implementation
profile. Two structural points frame the caching result precisely.
First, G-cache is \emph{exact relative to KV-cache by construction}
(\S\ref{subsec:cache}); the only approximation in either cache is
freezing $A$ at the prompt, whose effect is measured (empirically
indistinguishable under tested decoding) and bounded by the explicit
budgets of Proposition~\ref{prop:perturb} and Corollary~\ref{cor:budget},
which quantify the mechanism but, as evaluated in
Appendix~\ref{app:numerical}, are far too loose to certify bit-identical
decoding by themselves. Second, the training/inference asymmetry
(Theorem~\ref{thm:collapse}(iii)) separates the deployable inference
artifact ($G^\ast$ plus the linear pathways) from the trainable asset
(the PLGA network): the cached artifact reproduces inference but does
not permit equivalent continued training of the operator sector, so the
nonlinear network can be withheld at deployment, the concealment
application \emph{proposed} in \cite{gokden2025}. We flag its status
plainly: withholding weights is not established security; a deployed
artifact can be queried, distilled, or fine-tuned, and turning the
asymmetry into a security property requires a threat model and attack
evaluation that have not been carried out. An SDPA-LLM has no such
separation: its attention computation \emph{is} its deployable form.

\subsection{Inductive bias and transfer}
\label{subsec:adv-bias}

The power law stage of PLGA is the unique scale-equivariant element-wise
interaction (Corollary~\ref{cor:whypower}): the potential stage is
scale-covariant by construction (with the architecture-level caveats of
Remark~\ref{rem:scalefree-scope}), matched in form to the power law
statistics of natural language \cite{zipf1949} and of the SOC domains
targeted by the transfer program \cite{markovic2014}. SDPA's fixed
Euclidean form carries no such bias. SDPA is not without finite learned
objects (its projections, in particular the pre-projection form
$W_Q W_K^\top$, are transferable weights); what it lacks is an
input-conditioned head-space operator and exponent tensor extractable
as a single transplantable pair $(G^\ast, P)$, which is what the
operator-transfer program of Conjecture~\ref{conj:functorial} needs.
Within the family picture, SDPA is the \emph{identity-coupling} point:
no learned operator or exponents in the score stage, while the many
other inductive biases of the stack (projections, positional encoding,
FFN) remain.

\subsection{Costs, trade-offs, and honest limits}
\label{subsec:costs}

The advantages above are purchased, and the price should be stated with
the same precision:
\begin{enumerate}[leftmargin=2em]
\item \textbf{Training-time compute and parameters.} The metric learner
dominates the attention-parameter budget during training (parameter ratio
$\#\mathrm{ResL}/\#A \approx 129$--$149$ in the reference configurations
\cite{gokden2025}); the overhead is removed at inference by caching, but
training a PLDR-LLM costs more per step than its SDPA base point.
\item \textbf{The criticality search.} Reaching the near-critical phase
requires finding workable $(\eta_{\max}, T_w)$ pairs, is sensitive to the
SwiGLU:LU ratio and to initialization details, and can fail via
dragon-king events \cite{gokden2025,gokden2026}; SDPA-LLMs, training
stably at lower learning rates, are, as the source papers themselves
state, easier to train and quick to infer \cite{gokden2026}. The
mitigations are exactly the diagnostics of
\S\ref{subsec:adv-phase}, but the search is real work that SDPA does not
require.
\item \textbf{Benchmark parity at small scale.} At the $\sim$100M-parameter,
$\sim$8B-token scale of the published experiments, average benchmark
scores of PLDR-LLMs are comparable to, not dominant over, SDPA
references, with the learned-operator advantage visible in matched
ablations and in the longer $41$B-token run that overtakes the SDPA
reference on average \cite{gokden2024,gokden2025,gokden2026}. The honest
summary is: a slight performance edge under matched training plus a
qualitatively different capability set, not benchmark dominance; behavior
at billion-parameter scale is an open empirical question.
\item \textbf{Theory debt.} The mechanism of \S\ref{subsec:origin} is a
hypothesis with proved ingredients, not yet derived from the training
dynamics (Open Problem 1), and the transfer advantages of
\S\ref{subsec:adv-bias} rest on Conjecture~\ref{conj:functorial}, which
is stated to be falsified or confirmed, not assumed.
\item \textbf{Evidence base.} All empirical anchors in this section come
from single training runs per condition within one research program,
without uncertainty estimates or same-stack matched baselines; the
empirical program of Section~\ref{sec:discussion} lists what independent
replication requires. The claims above are calibrated to this evidence
and would strengthen or fall with it.
\end{enumerate}

One sentence carries the whole comparison: \emph{SDPA fixes the operator
sector a priori; PLDR-LLM learns it, collapses it at inference (exactly
under the invariance hypothesis, approximately as observed), and exposes
it}. The advantages flow from the learning (distinct training dynamics,
the measured edge under matched training), from the collapse (caching
with quantified perturbation, deployment asymmetry), and from the
exposure (diagnostics, the prospective intrinsic evaluation, exponent
readout, transfer program), while the observed costs concentrate in
reaching and holding the near-critical regime; the theory motivates,
but does not prove, that concentration.

\section{Conjectures}
\label{sec:conjectures}

This section collects the open claims of the research program as
precise, falsifiable conjectures. The proved results do not
assume these conjectures. The whole-model transfer conjecture
(Conjecture~\ref{conj:modeltransfer}) explicitly assumes the
operator-transfer conjecture (Conjecture~\ref{conj:functorial}), so the
two open claims are logically dependent.

\begin{conjecture}[Rigidity of the invariant operator]
\label{conj:invariance}
In the joint limit of width $\dk \to \infty$, depth fixed, and training
tokens $\to \infty$ at near-criticality ($m(\theta)\to0$), the deductive map
$x \mapsto \Glm(x)$ converges to a constant $G^\ast$ exactly, for almost
every input under the data distribution, and the finite-size fluctuation
obeys a power law $\varepsilon(\dk, N_{\mathrm{tokens}}) \sim
\dk^{-\alpha} N_{\mathrm{tokens}}^{-\beta}$ with universal exponents. (The
observed progression $10^{-6} \to 0$ at float resolution under a
$5\times$ token increase \cite{gokden2026}, confounded by schedule
differences, is a first data point, not an exponent estimate. The
$O(S^{-1/2})$ and $O(1/S)$ source terms of Proposition~\ref{prop:conc}
concern the \emph{context length} $S$, not $N_{\mathrm{tokens}}$; they
bound normalized Gram deviations under the stated law. Transfer retains
the positive-offset LayerNorm discrepancy, and a trained rate $\beta$
requires additional control of the training dynamics.)
\end{conjecture}

\paragraph{Status of rigidity.}
Common-row motion and state-dependent operator means survive finite row
concentration. Universal trained exponents, a convergence topology and an
identification of the changing tensor spaces remain to be established.

\begin{conjecture}[Spectral form of the order parameter]
\label{conj:spectral}
There is a symmetrization of the trained attention operators (with
respect to their stationary measures, or a reversibilization) and a
family limit (in context length and width) under which near-critical
models develop power law spectral accumulation at $1$ with exponent
$\vartheta$ in the sense of Proposition~\ref{prop:gap}(ii.a), including
its uniform lower-edge gap (the stronger exact-density form (ii.b),
with its constant, is a sharper version of the same conjecture), while
sub-critical models have a uniform spectral gap; and benchmark-relevant
reasoning ability is a monotone function of $\vartheta$. (The
conjecture's burden includes constructing the reversible family to which
Proposition~\ref{prop:gap} applies; per-layer, per-token operator
variation means no single-matrix formulation is adequate.) This would
upgrade the order parameter \eqref{eq:orderparam} from a fluctuation
diagnostic to a spectral one, computable from a single forward pass.
\end{conjecture}

\paragraph{Status of the spectral claim.}
Constructing the reversible attention family and establishing its
limiting edge exponent remain separate obligations.

The third conjecture formalizes the transplantation experiments of
\cite{gokden2025}: a $\Glm$ learned on one dataset interval can be
transplanted as the constant operator of a fresh model trained on a
different interval, and performs comparably to (and differently from)
identity or random constant operators.

\begin{conjecture}[Operator transfer across domains: congruence
alignment]
\label{conj:functorial}
Let $\mathcal{D}_1, \mathcal{D}_2$ be data distributions whose generative
processes lie in the same universality class (equal critical exponents of
their long-range statistics; making this precise is part of the
conjecture's burden). Let $G_1^\ast, G_2^\ast$ be the invariant
operators of PLDR-LLMs pretrained to near-criticality
($m \to 0$) on each. Then there exist invertible linear maps
$S^{(\ell,i)}$, acting as a \emph{common} change of frame on queries and
keys and constrained to a proper subgroup fixed by the conjecture (e.g.\
orthogonal maps, or maps of bounded condition number; the constraint is
essential, since with unconstrained independent query and key frames all
full-rank matrices are left--right equivalent and the statement would be
near-vacuous), such that
\[
G_2^{\ast(\ell,i)} \;=\;
\bigl(S^{(\ell,i)}\bigr)^{-\top} G_1^{\ast(\ell,i)}\,
\bigl(S^{(\ell,i)}\bigr)^{-1} + o(1)
\]
in the joint limit of model width and data. The transformation law is
the \emph{congruence} action, since $\Glm$ is a bilinear form on
query--key pairs: if $q' = Sq$ and $k' = Sk$, preservation of
$q^\top G k$ requires $G' = S^{-\top} G S^{-1}$.\footnote{The
similarity action $S G S^{-1}$ would be the wrong transformation law
here, and eigenvalues the wrong invariants: eigenvalues of a
nonsymmetric bilinear-form matrix are not invariants of the congruence
action. The correct invariants are those of the constrained congruence
classes, e.g.\ the signature of the symmetric part and the rank data of
the pairing; congruence of nonsymmetric forms has subtleties beyond the
symmetric-part signature, part of the conjecture's burden.}
Equivalently, the congruence-invariants of the operators (not their
eigenvalues) are universality-class data.
\end{conjecture}

Conjecture~\ref{conj:functorial} concerns the operators only; it does
not by itself move a model between domains. The stronger, separate claim
is:

\begin{conjecture}[Whole-model domain transfer]
\label{conj:modeltransfer}
Under the hypotheses of Conjecture~\ref{conj:functorial}, there is an
explicit transport procedure, freezing the extracted pair
$(G^\ast, P)$ of the source model and re-fitting only a stated list of
components (tokenizer/embedding maps, linear projections, and readout;
the nonlinear PLGA network stays frozen), such that the transported
model reaches benchmark parity on the target domain, within a stated
tolerance, with a from-scratch control trained at matched parameter,
token, and FLOP budget. (Operator congruence alone,
Conjecture~\ref{conj:functorial}, does not imply this: tokenizer,
embeddings, projections, FFN, normalization, and readout are not
transported by a congruence of $G^\ast$; which of them must be re-fit,
and at what cost, is exactly what this conjecture asserts to be small.)
\end{conjecture}

A prospective test of the two transfer conjectures must specify two
source laws, the proposed common universality class and the allowed
congruence subgroup before fitting the alignment. It must then evaluate
the declared frozen-operator transport against its matched from-scratch
control, retaining the tokenizer, normalization, feed-forward and readout
costs. The existing single-domain transplant measurements motivate this
protocol but do not complete it.

\section{Discussion and Open Problems}
\label{sec:discussion}

We collected the analytical skeleton of the PLDR-LLM program, with each
piece at its own epistemic level: a five-operator attention mechanism
whose power law stage is the unique scale-covariant elementwise
interaction (Corollary~\ref{cor:whypower}, with the scope limits of
Remark~\ref{rem:scalefree-scope}); Perron--Frobenius structure of the
positive interaction tensor and rank-one algebra of the collapsed
generator $A$ (Theorem~\ref{thm:pf},
Proposition~\ref{prop:rankone}); an exact algebraic collapse of inference
onto a constant-operator model \emph{under the invariance hypothesis},
with SDPA as the identity point of the family
(Theorem~\ref{thm:collapse}) and quantified perturbation bounds
(Proposition~\ref{prop:perturb}, Corollary~\ref{cor:budget}, evaluated in
Appendix~\ref{app:numerical}); a conditional three-stage analysis
(rotary twirl, statistical concentration, and row-map contraction) of
the origin of the observed deductive-output invariance, with its
hypotheses stated and its measurable diagnostics computed directly on a
released checkpoint
(Section~\ref{subsec:origin}, Hypothesis~\ref{hyp:selection});
walk-counting and commutant identities connecting the regularizer and
positional geometry to spectral theory (Theorem~\ref{thm:dag} with
Remarks~\ref{rem:dagobstruction} and \ref{rem:tracedual},
Proposition~\ref{prop:rope-commutant}); a
phenomenological SOC framework for training with an intrinsic order
parameter (Section~\ref{sec:criticality}); and a consolidated account of
the advantages over SDPA-LLMs at their actual strength, including a
conditional head-level positional operator codimension
(Section~\ref{sec:advantages}).

Open problems, beyond the conjectures of
Section~\ref{sec:conjectures}:

\begin{enumerate}[leftmargin=2em]
\item \textbf{Dynamics of the collapse.} Derive, from the gradient flow of
\eqref{eq:dagloss}-augmented cross-entropy, the convergence of the shared
row map $\varphi$ (Proposition~\ref{prop:rowfact}) to the constant-map
manifold at critical $(\eta_{\max}, T_w)$, turning the three-part
mechanism of Section~\ref{subsec:origin} into a theorem about the training
dynamics, ideally identifying the constant-map manifold as an attracting
invariant manifold whose stability changes at the critical point.
\item \textbf{An honest RG map.} Construct an explicit coarse-graining on
token sequences under which the layer map of
Definition~\ref{def:pldr} is (approximately) covariant, upgrading the RG
reading of Section~\ref{subsec:soc} from correspondence to theorem.
\item \textbf{Commutant diagnostics.} Measure the distance of trained
$G^\ast$ from the RoPE commutant of
Proposition~\ref{prop:rope-commutant} across layers; this is a
checkpoint-scoped occupancy measurement of the absolute-position-sensitive
operator sector, a diagnostic with no direct analogue for SDPA's fixed
identity operator (its pre-projection form $W_Q W_K^\top$ is constant, not
input-conditioned). A first such measurement now exists: on the audited
checkpoint the commutant residual
$\norm{\Glm - \Pi_{\mathrm{comm}}\Glm}_F/\norm{\Glm}_F$ is reported per
layer and head in Appendix~\ref{app:numerical}; the open problem is its
behavior across scales, checkpoints, and training trajectories, together
with its predictive contribution. The latter requires projecting $\Glm$
onto the commutant and evaluating the behavioral change, with replication
across checkpoints and scales.
Relatedly, measure the empirical size of the
off-commutant part of $\widetilde D/S$ against the $O(1/S)$ twirl bound of
Lemma~\ref{lem:twirl}.
\item \textbf{The empirical program.} The empirical claims of this paper
rest on single training runs per condition from one research program.
What would settle them: multiple pretraining seeds per
$(\eta_{\max}, T_w)$ cell over a dense grid, at several model scales
and on independent corpora, with uncertainty on both
$m(\theta)$ (in the stabilized normalization of
Definition~\ref{def:orderparam}) and benchmark scores; an
\emph{independent} phase criterion fixed in advance of measuring the
order parameter and evaluated blind against behavior, to break the
labeling circularity noted in
\S\ref{subsec:soc}; parameter-, token-, and FLOP-matched SDPA,
constant-$G$, no-power, and randomized-$P$ baselines run in the same
software stack (with a profiler breakdown for any speed claims),
together with frozen-, transferred-, and random-$G$ inference
controls; the commutant-projection intervention (project the trained
$\Glm$ onto the RoPE commutant and measure the behavioral change,
the causal counterpart of the occupancy measurement in
Appendix~\ref{app:numerical});
element-wise cached-versus-recomputed comparisons at score, softmax,
logit, KL/TV, and token-decision levels across prompts, layers, heads,
and decoding steps, extended to held-out prompt sets, longer contexts,
adversarial distribution shifts, and multiple checkpoints;
the historical-row context-interaction probe of
Section~\ref{subsec:online} run at a sequence of \emph{training}
checkpoints, locating when prefix consistency
(Definition~\ref{def:prefixconsistent}) emerges as the operator
collapses, together with the sequential-versus-block score gap along
the same trajectory (extending the two-released-checkpoint contrast of
Appendix~\ref{app:numerical});
the blockwise-CE-versus-sequential-NLL gap along the same training
trajectory: the target-exposure channel of \S\ref{subsec:dag} is
present precisely while the metric learner is input-sensitive, so a
collapse measured only at the end state does not settle its role
during training;
sequential rescoring of the actual benchmark items across
\emph{every} checkpoint used for causal claims about operator, DAG,
or phase effects, the stated precondition for any future claim of
benchmark superiority or causal phase effects (the two released
checkpoints are covered in Appendix~\ref{app:numerical});
singular-value spectra and numerical ranks in place
of determinants; measured row-map Jacobians, frequency-resolved twirl
residuals, commutant residuals, and end-to-end Lipschitz factors along
training (extending Appendix~\ref{app:numerical} from one checkpoint to
trajectories); and, for the SOC reading specifically, finite-size
scaling and data collapse, susceptibility, avalanche or $1/f$
statistics, and an identified self-tuning feedback mechanism.
Independent replication outside this research program would materially
change the strength of every empirical claim above.
\end{enumerate}

The broader claim of the program is that PLDR-LLM is not ``a transformer
variant'' but a family of models parameterized by an operator $G$,
containing SDPA at $G = I$, whose training in the critical-like regime
produces, and exposes, approximately invariant operators. The exact
inclusion and the exposure are theorems and architecture; the invariance
is a measured hypothesis with a proposed mechanism; and the reading of
the invariant operators as laws of the training domain is the program's
open ambition, to be earned by the empirical program above rather than
assumed from the mathematics.

\section*{Acknowledgments}

I am grateful to my parents for their support and patience. This
research was conducted independently without support from a grant or
corporation.

\section*{Disclosure of the use of AI tools}

The author discloses that (Claude, model
Claude Fable 5, Anthropic) was used in the preparation of
this article: in organizing the material of the source papers
\cite{gokden2019,gokden2021,gokden2024,gokden2025,gokden2026}, in
drafting the text and the proofs of the stated results, in
cross-checking the architectural equations of
Sections~\ref{sec:plga}--\ref{sec:pldr} against the reference
implementations listed in Appendix~\ref{app:repos}, in drafting the
Lean~4 formalization of selected proofs released with this article
(Appendix~\ref{app:repos}), which was verified mechanically by the Lean
proof checker, and in implementing and running the numerical audit of
Appendix~\ref{app:numerical} on the released checkpoint. 
(Codex, model GPT-5.6-Sol, OpenAI) was used as a non-authorial analysis and
verification tool in preparing this manuscript. Its involvement included 
checking mathematical definitions, dimensions, proofs, and claim scope; 
comparing the manuscript with the cited PLDR-LLM and PLGA implementations; 
building and testing the Lean~4 formalization; running the supplied 
numerical-audit tests; checking citations and cross-references; conducting 
targeted literature searches; and suggesting technical and expository
refinements.
(Codex, model GPT-6-Astra, OpenAI) also assisted with mathematical
analysis and editing during the final stages of manuscript preparation.

\vfill\newpage
\appendix

\section{Notation}
\label{app:notation}

We use $0$ for a scalar, $\bm0$ (or $\bm0_d$) for a zero vector,
and $\bm0^\top$ for a zero row vector. The notation
$\bm0_{m\times n}$ specifies a zero matrix's shape;
$\bm0_{\mathrm{mat}}$ leaves that shape to the operands. Norms, energies,
and individual entries have scalar zeros. Zero operators and identically
zero functions are distinguished by their stated domains when needed.

\begingroup
\small
\begin{longtable}{@{}>{\raggedright\arraybackslash}p{0.34\textwidth}>{\raggedright\arraybackslash}p{0.58\textwidth}@{}}
\caption{Architectural symbols and their local meanings.}\label{tab:architectural-notation}\\
\toprule
Symbol & Meaning \\
\midrule\endfirsthead
\caption[]{Architectural symbols and their local meanings (continued).}\\
\toprule
Symbol & Meaning \\
\midrule\endhead
\midrule
\multicolumn{2}{r}{\footnotesize Continued on the next page.}\\
\endfoot
\bottomrule\endlastfoot
$\mathcal{V},\ V$ & vocabulary and its size \\
$S,\ S_{\max}$ & context length, maximum context length \\
$\dmodel,\ h,\ \dk = \dmodel/h$ & model width, heads, head width \\
$L,\ d_{f\!f}$ & decoder depth, FFN width \\
$N_{\mathrm{res}},\ n_A,\ \Adff$ & metric-learner residual units, SwiGLU
blocks per unit, hidden width \\
$Q, K, V$ & query/key/value matrices ($S \times \dk$ per head) \\
$\widetilde Q, \widetilde K,\ R_n$ & rotary-rotated inputs; RoPE rotation
at position $n$ \\
$D_Q = Q^\top Q$, $\widetilde D = \widetilde Q^\top \widetilde Q$ & density
operator \eqref{eq:density}; its rotary-twisted form \eqref{eq:twisted} \\
$\Phi_{\mathrm{res}},\ \varphi$ & metric learner (shared per layer) and its
row map (Prop.~\ref{prop:rowfact}) \\
$u_j,\ g_{j,1},\ g_{j,2}$ & residual units of $\varphi$ and their SwiGLU
blocks \eqref{eq:resunit} \\
$Q(x),\ D(x),\ A(x)$ & query matrix, density operator, metric generator
evaluated on input $x$ \\
$A,\ \Alm,\ \Ap,\ \Glm$ & generator, positive interaction (``metric''),
potential, and bilinear score (``energy--curvature'') tensors
(Rem.~\ref{rem:nomenclature}) \\
$P,\ a,\ b_a,\ W,\ b_W$ & learned exponents, couplings, biases (five
$\dk\times\dk$ matrices per head) \\
$\Elm,\ \Vlm$ & attention (Markov) operator and head output \\
$M^{\Had P}$ & element-wise power \eqref{eq:hadpower};
$\Glm = a\Ap + b_a$ is a matrix product \\
$I_2,\ J_2$ & $2\times2$ identity; rotation by $\pi/2$ on a RoPE plane
(Prop.~\ref{prop:rope-commutant}) \\
$P_T$ & projection onto the ambient RoPE torus commutant
(Prop.~\ref{prop:rope-commutant}, Lem.~\ref{lem:twirl}) \\
$\Sigma_{\LN},\ \mathcal{S}$ & Fixed compact containing set for LayerNorm outputs; visited row set
(Lem.~\ref{lem:ln}, Prop.~\ref{prop:contract}) \\
$D_L(\cdot),\ B$ & DAG loss \eqref{eq:dagloss}; batch size \\
$\mathcal{L}_{\mathrm{block}}$ & blockwise (global-context)
cross-entropy \eqref{eq:blockloss}; distinguished from sequential AR
NLL (\S\ref{subsec:online}) \\
$C_G^{\mathrm{samp}},\ C_G^{\mathrm{tube}}$ & the budget constant of
Cor.~\ref{cor:budget} at the measured $m_A$ (pairs of visited
states) and at the floor $m_A = \epsilon$ (arbitrary paths in the
stated bounded domain) \\
$p_\theta(\,\cdot \mid x_{1:t})$ & online final-row conditional
\eqref{eq:onlinecontract} ($S = t$ contract, Def.~\ref{def:pldr}) \\
$m(\theta)$ & RMS-normalized order parameter \eqref{eq:orderparam} \\
$(\eta_{\max}, T_w)$ & max learning rate, warm-up steps (control
parameters) \\
$G^\ast,\ \alpha^\ast$ & constant operator proposed for inference caching;
common row profile under the stated invariance assumptions \\
\end{longtable}
\endgroup

\vfill\newpage
\subsection*{Result dependency map}

Direct proof inputs of the main numbered results, as a navigation aid
through the cross-references. Each result's statement carries its own
hypotheses; classical inputs are cited at the point of use in the
body text.

\begingroup
\small
\begin{longtable}{@{}>{\raggedright\arraybackslash}p{0.34\textwidth}>{\raggedright\arraybackslash}p{0.58\textwidth}@{}}
\caption{Direct proof dependencies of the architectural results.}\label{tab:result-dependencies}\\
\toprule
Result & Direct proof inputs \\
\midrule\endfirsthead
\caption[]{Direct proof dependencies of the architectural results (continued).}\\
\toprule
Result & Direct proof inputs \\
\midrule\endhead
\midrule
\multicolumn{2}{r}{\footnotesize Continued on the next page.}\\
\endfoot
\bottomrule\endlastfoot
Prop.~\ref{prop:density}, Prop.~\ref{prop:positivity},
Prop.~\ref{prop:rankone}, Prop.~\ref{prop:rowfact},
Prop.~\ref{prop:markov} & self-contained (elementary linear algebra
on the architectural definitions of \S\ref{sec:plga}) \\
Thm.~\ref{thm:pf} & Prop.~\ref{prop:positivity}; classical
Perron--Frobenius theory \\
Prop.~\ref{prop:rope-commutant} & self-contained (nonresonance
hypothesis; block-diagonal commutant computation) \\
Thm.~\ref{thm:dag} & self-contained (walk expansion of the trace
exponential) \\
Prop.~\ref{prop:cacheable} & softmax common-shift characterization
(proved in place) \\
Thm.~\ref{thm:collapse} & (i) substitution at $\Glm = I$; (ii)
Prop.~\ref{prop:rope-commutant}, and Cor.~\ref{cor:posgap} for
unrestricted-level necessity; (iii) chain rule \\
Prop.~\ref{prop:perturb} & Lem.~\ref{lem:softmax};
submultiplicativity \\
Lem.~\ref{lem:twirl} & Prop.~\ref{prop:rope-commutant}; finite
geometric sums \\
Prop.~\ref{prop:conc} & matrix Bernstein inequality, under the
stated idealization \\
Lem.~\ref{lem:ln} & self-contained (direct calculus on the
$\varepsilon$-LayerNorm) \\
Prop.~\ref{prop:contract} & Lem.~\ref{lem:ln}; composition of
Lipschitz maps, under the stated contraction hypothesis \\
Cor.~\ref{cor:budget} & Lem.~\ref{lem:softmax},
Prop.~\ref{prop:perturb}, Lem.~\ref{lem:ln},
Prop.~\ref{prop:contract} (chain assembly) \\
Prop.~\ref{prop:scalecov} & measurable multiplicative Cauchy
functional equation \\
Cor.~\ref{cor:whypower} & Prop.~\ref{prop:scalecov} \\
Prop.~\ref{prop:gap} & spectral calculus for reversible self-adjoint
families, under the stated hypotheses \\
Cor.~\ref{cor:posgap} & Prop.~\ref{prop:rope-commutant}; the
realizability argument is otherwise self-contained \\
Hyp.~\ref{hyp:selection}; Conjectures
(\S\ref{sec:conjectures}) & not proved; ingredient-tagged hypothesis
and falsifiable statements with proposed measurements \\
\end{longtable}
\endgroup

\vfill\newpage
\section{Code, Models, and Verification Resources}
\label{app:repos}

Reference implementations and released models accompanying the source
papers. The architectural equations of Sections~\ref{sec:plga} and
\ref{sec:pldr} were verified for this article against the reference
implementations at the exact snapshots recorded in the verification
manifest below;
``every equation verified'' throughout this paper means verified against
those snapshots.

\paragraph{Verification manifest.}
The verified snapshots are the following; each identifier is set in
fixed $16$-character groups with explicit line breaks, so no engine's
paragraph breaker can carry it past the text block.
\begingroup
\small
\begin{longtable}{@{}l>{\ttfamily}l@{}}
\caption{Immutable code and model identifiers for the architectural measurements.}\label{tab:model-pins}\\
\toprule
Pinned object & \normalfont Identifier \\
\midrule\endfirsthead
\caption[]{Immutable code and model identifiers for the architectural measurements (continued).}\\
\toprule
Pinned object & \normalfont Identifier \\
\midrule\endhead
\midrule
\multicolumn{2}{r}{\footnotesize Continued on the next page.}\\
\endfoot
\bottomrule\endlastfoot
TensorFlow framework repo, commit &
\begin{tabular}[t]{@{}>{\ttfamily}l@{}}
094cda58e217db65 3126220dbf4572ec\\ 9e924d1d
\end{tabular} \\
KVG-cache repo, commit &
\begin{tabular}[t]{@{}>{\ttfamily}l@{}}
8e8658fca4eacf8c 9013f5bcd60408b5\\ fc679777
\end{tabular} \\
Self-organized-criticality repo, commit &
\begin{tabular}[t]{@{}>{\ttfamily}l@{}}
ea5c2b4cc7a88043 ecd663f95b123c9b\\ 8e6646e0
\end{tabular} \\
HF port \texttt{PLDR-LLM-v51-SOC-110M-5}, revision &
\begin{tabular}[t]{@{}>{\ttfamily}l@{}}
de8e539c0ba18290 72f4b8c2c5fae3bd\\ e0a3a2d2
\end{tabular} \\
\quad its \texttt{modeling\_pldrllm.py}, SHA-256 &
\begin{tabular}[t]{@{}>{\ttfamily}l@{}}
e5cbaa6c5433364e e15051e644925b0f\\
74c3968159d663ce 1aab4fb985ac26f0
\end{tabular} \\
HF \texttt{PLDR-LLM-v51-SOC-110M-1}, revision &
\begin{tabular}[t]{@{}>{\ttfamily}l@{}}
7a34e2ca9aa78038 683677cfda17fe3a\\ 9fe6da8a
\end{tabular} \\
Evaluation-harness fork (KVG-cache), commit &
\begin{tabular}[t]{@{}>{\ttfamily}l@{}}
2527a989bf486da1 35ce470515c08175\\ d4e416c5
\end{tabular} \\
\quad its \texttt{lm\_eval/models/pldrllm.py}, SHA-256 &
\begin{tabular}[t]{@{}>{\ttfamily}l@{}}
08adbf77656da774 e1babf077afb45f2\\
8393aa062f50f58e 3337089b668ea8b6
\end{tabular} \\
\end{longtable}
\endgroup
The sequential-validation audit of Appendix~\ref{app:numerical}
additionally pins its held-out corpus and benchmark items as
parquet files at the following Hugging Face dataset revisions (same
grouping convention; the SHA-256 of each individual file is recorded
in the audit's results JSON, and the release checklist verifies that
every identifier in this manifest resolves against its public
remote):
\begingroup
\small
\begin{longtable}{@{}l l@{}}
\caption{Dataset identifiers for the sequential-validation measurements.}\label{tab:data-pins}\\
\toprule
Dataset & Revision \\
\midrule\endfirsthead
\caption[]{Dataset identifiers for the sequential-validation measurements (continued).}\\
\toprule
Dataset & Revision \\
\midrule\endhead
\midrule
\multicolumn{2}{r}{\footnotesize Continued on the next page.}\\
\endfoot
\bottomrule\endlastfoot
\texttt{Salesforce/wikitext} &
\begin{tabular}[t]{@{}>{\ttfamily}l@{}}
b08601e04326c79d fdd32d625aee71d2\\ 32d685c3
\end{tabular} \\
\texttt{allenai/ai2\_arc} &
\begin{tabular}[t]{@{}>{\ttfamily}l@{}}
210d026faf995565 3af8916fad021475\\ a3f00453
\end{tabular} \\
\texttt{Rowan/hellaswag} &
\begin{tabular}[t]{@{}>{\ttfamily}l@{}}
218ec52e09a7e746 2a5400043bb9a69a\\ 41d06b76
\end{tabular} \\
\texttt{ybisk/piqa} (parquet branch) &
\begin{tabular}[t]{@{}>{\ttfamily}l@{}}
142c51238b3ca2bc 61e9a075913871b8\\ b600e8e1
\end{tabular} \\
\texttt{allenai/openbookqa} &
\begin{tabular}[t]{@{}>{\ttfamily}l@{}}
388097ea7776314e 93a529163e0fea80\\ 5b8a6454
\end{tabular} \\
\texttt{allenai/social\_i\_qa} (parquet branch) &
\begin{tabular}[t]{@{}>{\ttfamily}l@{}}
537a2ec8ec565adc 0b70b70752893e59\\ e024df26
\end{tabular} \\
\texttt{allenai/winogrande} &
\begin{tabular}[t]{@{}>{\ttfamily}l@{}}
01e74176c63542e6 b0bcb004dcdea22d\\ 94fb67b5
\end{tabular} \\
\texttt{truthfulqa/truthful\_qa} &
\begin{tabular}[t]{@{}>{\ttfamily}l@{}}
741b8276f2d1982a a3d5b832d3ee81ed\\ 3b896490
\end{tabular} \\
\end{longtable}
\endgroup
The spaces inside each identifier are grouping only; the identifier
is the concatenation of its groups.
The numerical audit of Appendix~\ref{app:numerical} loads the model
at the pinned model revision above and records the same hash in its
results file.

\paragraph{Model implementations.}
\begin{itemize}[leftmargin=2em]
\item PLDR-LLM at self-organized criticality (PyTorch v510;
training, inference, KV/G-cache; $W_V$ initialized as $W_Q, W_K$):
\url{https://github.com/burcgokden/PLDR-LLM-Self-Organized-Criticality}
\item PLDR-LLM with KV-cache and G-cache (PyTorch v510; identical
except $W_V$ keeps framework-default initialization; includes v510G
(predefined $\Glm$) and v510Gi (transplanted $\Glm$) ablation models):
\url{https://github.com/burcgokden/PLDR-LLM-with-KVG-cache}
\item Original PLDR-LLM framework (TensorFlow v500/v900,
architecture of \cite{gokden2024}):
\url{https://github.com/burcgokden/LLM-from-Power-Law-Decoder-Representations}
\item Power Law Graph Transformer (encoder--decoder PLGA of
\cite{gokden2021}):
\url{https://github.com/burcgokden/Power-Law-Graph-Transformer}
\item CoulGAT screened-Coulomb graph attention
(\cite{gokden2019}):
\url{https://github.com/burcgokden/CoulGAT-Graph-Attention-Interpretability}
\end{itemize}

\paragraph{Machine-checked proofs and audit code (Lean 4 + Python).}
The Lean~4/mathlib formalization of the elementary proof cores listed in
the introduction, building with no unproved obligations
(\texttt{lake build} re-verifies every proof with the Lean kernel),
together with the numerical-audit code, its Python dependency
specification, and the full results files and raw arrays of
Appendix~\ref{app:numerical} (in the
repository's \texttt{audit/} directory):
\url{https://github.com/burcgokden/PLDR-LLM-Math-Foundations}.
The immutable public source snapshot is commit
\href{https://github.com/burcgokden/PLDR-LLM-Math-Foundations/tree/238552e044b33bbd172a08bb72f36285225e969b}{\texttt{238552e044b33bbd172a08bb72f36285225e969b}}.
Its \path{source-manifest.json} records payload SHA-256
\texttt{2793a2471a3194bccedfe30a38ccc4a34b8}\allowbreak\texttt{914f044fd9eb50be684b941f456cd}.
The six retained numerical inputs are bound by
\path{audit/retained-inputs.json}; \path{RELEASE.md} and
\path{release-gate.json} state the CPU verification scope, and
\path{lean-toolchain} and \path{lake-manifest.json} pin the formal dependencies.
This source identifier locates the available code and retained evidence;
the checkpoint and dataset identifiers above identify the reported
measurements. A source-citation update does not constitute a new
checkpoint execution or a new Lean verification.

\paragraph{Released models (Hugging Face).}
The organization \url{https://huggingface.co/fromthesky} hosts the
pretrained model families of the source papers, including the
\texttt{pldrllmv5/v9} series of \cite{gokden2024}, the
\texttt{PLDR-LLM-v51} and \texttt{PLDR-LLM-v51G} series of
\cite{gokden2025}, the \texttt{PLDR-LLM-v51-SOC} series of
\cite{gokden2026}, and \texttt{v52} fine-tuned variants. The
\texttt{PLDR-LLM-v51-SOC} repositories ship a custom Hugging Face
Transformers port (\texttt{modeling\_pldrllm.py},
\texttt{PldrllmForCausalLM}) whose configuration records the reference
hyperparameters used throughout this paper: $\Adff{=}170$,
$N_{\mathrm{res}}{=}8$, $n_A{=}2$, $\dk{=}64$, $\varepsilon_{\LN}{=}10^{-6}$,
RoPE base $10^4$, untied embeddings, biases on all projections.

\paragraph{Evaluation wrappers and scoring protocol.}
The benchmark evaluations of the source papers run through forks of
the EleutherAI evaluation harness with a PLDR-LLM model wrapper:
\begin{itemize}[leftmargin=2em]
\item harness fork for the KVG-cache models (wrapper
\texttt{lm\_eval/models/pldrllm.py}; commits pinned in the
verification manifest above):
\url{https://github.com/burcgokden/lm-evaluation-harness-with-PLDR-LLM-kvg-cache}
\item harness fork for the original v500 framework:\\
\url{https://github.com/burcgokden/lm-evaluation-harness-with-PLDR-LLM}
\end{itemize}
The wrapper's log-likelihood path scores each answer candidate by
\emph{one-pass block scoring} (Section~\ref{subsec:online}): context
and candidate are concatenated, all tokens but the last are evaluated
in one call, and the candidate log-probabilities are summed. This is a
whole-candidate block score; it coincides with sequential final-row
scoring exactly for one-token candidates and, on the audited
checkpoint, to within the measured gaps of
Appendix~\ref{app:numerical} for multi-token candidates. The
log-likelihood path performs a single uncached pass by construction,
so no generation-time cache participates in the published scores; the
wrapper's \emph{generation} path is not used by any published score
cited here or by any audit in this paper.

\vfill\newpage
\section{Lean Formalization: Exact Coverage}
\label{app:lean}

The Lean~4/mathlib repository
(\url{https://github.com/burcgokden/PLDR-LLM-Math-Foundations})
builds with no \texttt{sorry}, \texttt{admit}, \texttt{axiom}, or
\texttt{unsafe} declarations; \texttt{lake build} re-verifies every
proof with the Lean kernel. The precise trust statement is:
kernel-checked \emph{relative to mathlib's standard classical
principles}, not axiom-free in a foundational sense, and not
constructive. An axiom audit in the repository's continuous
integration (\texttt{scripts/check\_axioms.py}) sweeps
\texttt{\#print axioms} over every exported theorem and fails if
anything appears beyond \texttt{propext},
\texttt{Classical.choice}, and \texttt{Quot.sound}. This
appendix states \emph{exactly} which
part of each numbered result is kernel-checked and which is not.
Formalization of a proof core is not evidence for the adjacent
unformalized claims, and nothing in this paper cites the formalization
as such; in particular, kernel-checking the algebraic cores could not
have detected a defect in the numerical audit's own bookkeeping,
which is why the audit carries its separate semantic test oracles
(Appendix~\ref{app:numerical}).

\begingroup
\footnotesize
\begin{longtable}{@{}>{\raggedright\arraybackslash}p{0.235\textwidth}>{\raggedright\arraybackslash}p{0.36\textwidth}>{\raggedright\arraybackslash}p{0.32\textwidth}@{}}
\caption{Selected architectural proof cores: algebra, normalization, and rotary geometry.}\label{tab:formal-algebra}\\
\toprule
Module & Kernel-checked & Not checked (prose-only) \\
\midrule\endfirsthead
\caption[]{Selected architectural proof cores: algebra, normalization, and rotary geometry (continued).}\\
\toprule
Module & Kernel-checked & Not checked (prose-only) \\
\midrule\endhead
\midrule
\multicolumn{3}{r}{\footnotesize Continued on the next page.}\\
\endfoot
\bottomrule\endlastfoot
\texttt{Iswiglu.lean} &
Def.~\ref{def:iswiglu}: nonnegativity, positivity away from $0$;
entries of $\Alm \ge \epsilon$ (Prop.~\ref{prop:positivity}) &
Derivative bounds on compact intervals (Cor.~\ref{cor:budget}) \\
\texttt{Softmax.lean} &
Plain softmax: positivity, row sum $1$, common-shift invariance,
convex-hull bound; \emph{ideal masked softmax}: nonnegativity, strict
positivity on the support, exact zero off it, row sum $1$, shift
invariance (Prop.~\ref{prop:markov}); \emph{equality iff common shift},
plain and masked: the row-level characterization of
Prop.~\ref{prop:cacheable}(iii), both directions &
Softmax Lipschitz bound (Lem.~\ref{lem:softmax}); finite-mask underflow
semantics (Rem.~\ref{rem:finitemask}) \\
\texttt{DensityOperator.lean} &
Prop.~\ref{prop:density}: symmetry, PSD quadratic form, rank equality
$\rank Q^\top Q = \rank Q$ &
Statistical concentration; rotated-query statistics \\
\texttt{HadamardPower.lean} &
Prop.~\ref{prop:positivity}: entrywise group law of
$A \mapsto A^{\Had tP}$ for positive $A$ &
Architecture-level scale covariance (correctly not claimed;
Rem.~\ref{rem:scalefree-scope}) \\
\texttt{RankOne.lean} &
Prop.~\ref{prop:rankone}(i)--(iii): rank $\le 1$ with equality iff
$\alpha \ne \bm0$; $\det = 0$ ($d \ge 2$); row-sum eigenvector;
$A^{k+1} = s^k A$; nilpotency $A^2 = \bm0_{\mathrm{mat}}$ at $s = 0$ &
The singular-value determinant bound (iv); any spectral-radius
statement \\
\texttt{LayerNorm.lean} &
Lem.~\ref{lem:ln}(i): exact shift invariance; (iii): the exact norm
identity $\sum \hat u_i^2 = \dk v/(v+\varepsilon)$; exact scale
invariance at $\varepsilon = 0$ on positive-variance rows &
The quantitative scale-error inequality (ii); the Lipschitz constant
(iv) \\
\texttt{Rope.lean} &
Group homomorphism $n \mapsto \operatorname{Rot}(n\theta)$; single-block
commutant $MR = RM \iff M = cI_2 + sJ_2$ for $\sin\theta \ne 0$;
commuting operators give relative-position scores
(Prop.~\ref{prop:rope-commutant}, per block;
Thm.~\ref{thm:collapse}(ii) mechanism) &
The multi-block nonresonance argument (distinct eigenvalue pairs
$\Rightarrow$ block-diagonal commutant); the codimension count of
Cor.~\ref{cor:posgap} \\
\texttt{TwirlBound.lean} &
Lem.~\ref{lem:twirl} analytic core: chord identity and geometric-sum
bound $\abs{\sum_{n\le S} e^{i\omega n}} \le 1/\abs{\sin(\omega/2)}$ &
The matrix-level projection statement; the frequency-resolved
constants; any stochastic application \\
\end{longtable}
\endgroup

\vfill\newpage
\begingroup
\footnotesize
\begin{longtable}{@{}>{\raggedright\arraybackslash}p{0.235\textwidth}>{\raggedright\arraybackslash}p{0.36\textwidth}>{\raggedright\arraybackslash}p{0.32\textwidth}@{}}
\caption{Selected architectural proof cores: regularization, inference, and causal contracts.}\label{tab:formal-inference}\\
\toprule
Module & Kernel-checked & Not checked (prose-only) \\
\midrule\endfirsthead
\caption[]{Selected architectural proof cores: regularization, inference, and causal contracts (continued).}\\
\toprule
Module & Kernel-checked & Not checked (prose-only) \\
\midrule\endhead
\midrule
\multicolumn{3}{r}{\footnotesize Continued on the next page.}\\
\endfoot
\bottomrule\endlastfoot
\texttt{DagLoss.lean} &
Thm.~\ref{thm:dag} walk side: entrywise nonnegativity of powers,
$\tr N^k \ge 0$, $\tr e^N \ge d$, $\tr e^N \ge d + \tr N$; the
positivity obstruction $\tr e^N \ge d + d\epsilon$ of
Rem.~\ref{rem:dagobstruction} (the $h$-side inequality, before
normalization and logarithm; the normalized-loss floor
$\log(1+\epsilon^2)$ follows in prose by monotonicity and is not a
separate Lean statement) &
Equality iff acyclic (the NOTEARS converse); the spectral-side
identity of Rem.~\ref{rem:tracedual} \\
\texttt{Contraction.lean} &
Prop.~\ref{prop:contract}(i): multiplicative diameter bound;
$K^N$ contraction for iterated $K$-Lipschitz maps &
Any claim that trained units are contractive (measured, not proved:
Appendix~\ref{app:numerical}); the converse (iii) \\
\texttt{InferenceCollapse.lean} &
Thm.~\ref{thm:collapse}(i): $Q \cdot I \cdot K^\top = QK^\top$;
constant-operator absorption (no-RoPE (ii)), including the affine form
$(XW + \mathbf{1}b^\top)G = X(WG) + \mathbf{1}(b^\top G)$ matching the
bias-bearing projections of Definition~\ref{def:pldr};
Prop.~\ref{prop:cacheable}(i) at score and masked-attention-row level
(one-way sufficiency) &
The RoPE necessity direction of Thm.~\ref{thm:collapse}(ii) at the
unrestricted operator level (Cor.~\ref{cor:posgap}); the decoding-step
induction of Prop.~\ref{prop:cacheable}(iii)'s model level; the
gradient decomposition (Thm.~\ref{thm:collapse}(iii)) \\
\texttt{OnlineContract.lean} &
Over an abstract decoder map: online causality
(Rem.~\ref{rem:onlinecausal}; \texttt{online\_causal});
\texttt{LengthPreserving} (one output row per input row,
Def.~\ref{def:pldr}'s codomain statement) and
\texttt{PrefixConsistentShaped}, the exact bundle of
Def.~\ref{def:prefixconsistent}, under which the
\texttt{Option}-valued row comparison upgrades to total rows (the
comparison alone is vacuous for the constant-empty decoder, kept as
a proved non-example); row-locality (the shape of causally
masked attention) implies the bundle given the shape property,
and the row-map decoder satisfies all of it
(Rem.~\ref{rem:prefixmech}(i)); the block reading of a historical
row equals the online final-row output for such maps
(\texttt{rowMapDecoder\_\allowbreak block\_\allowbreak eq\_online});
a two-token
length-preserving global-aggregation decoder violating prefix
consistency, by kernel computation
(Rem.~\ref{rem:prefixmech}(ii)) &
The full PLDR-LLM decoder is not formalized: these are wrapper-level
statements over an abstract map; \eqref{eq:histderiv} and every
checkpoint measurement remain prose and audit \\
\end{longtable}
\endgroup

Outside the scope of the formalization entirely: Perron--Frobenius
theory (Thm.~\ref{thm:pf}; not yet in mathlib), matrix Bernstein
concentration (Prop.~\ref{prop:conc}), the reversible-family spectral
dictionary (Prop.~\ref{prop:gap}), the perturbation bounds of
Prop.~\ref{prop:perturb} and Cor.~\ref{cor:budget}, and everything
stated as a hypothesis, analogy, or conjecture.

\vfill\newpage
\section{Numerical Audit on a Released Checkpoint}
\label{app:numerical}

This appendix evaluates, on a released checkpoint, the measurable
quantities that the results of this paper depend on: singular-value
spectra in place of determinants (trained and at random
initialization), the LayerNorm scale error of Lemma~\ref{lem:ln}
computed directly, the twirl energy of Lemma~\ref{lem:twirl} resolved
against \emph{pre-rotation} data, the commutant residual of the
trained operators against Corollary~\ref{cor:posgap}'s direction
count, the contraction diagnostics of
Proposition~\ref{prop:contract} measured on the full composition, the
budget of Corollary~\ref{cor:budget} as a sample-extrema proxy against
measured
decoding margins, per-step cached-operator deviations, DAG-loss
values, and the order parameter of
Definition~\ref{def:orderparam} in both normalizations, together with
the online-contract measurements of Section~\ref{subsec:online}
(historical-row movement under suffix changes, padding at the Gram
boundary, and sequential-versus-block candidate scoring), run on
two released checkpoints
(Sections~\ref{subsec:auditonline}--\ref{subsec:auditscores}). The design
rule of the audit is that every reported quantity is either the named
quantity computed directly, or is explicitly labeled a bound or proxy
with its formula shown, and where a measurement contradicts a
convenient assumption, the contradiction is reported. Semantically
sensitive constructions are guarded in code rather than by
convention: the power stage is built by a shape-checked helper with a
sentinel test that fails under any head- or row-broadcast indexing of
the exponent parameter, and the assembled perturbation chain is
checked against a miniature-decoder finite-difference oracle that
fails under any assembly omitting the injection layer's own
post-attention factors.

\subsection{Setup and provenance}
Model: \texttt{fromthesky/PLDR-LLM-v51-SOC-110M-5} (the $41$B-token
near-critical model of \cite{gokden2026}; $L = 5$ layers, $h = 14$
heads, $\dk = 64$, float32), loaded through its released Hugging Face
port, unmodified, \emph{at the pinned model revision of the
verification manifest} (Appendix~\ref{app:repos}), under
\texttt{transformers} $4.56.1$ (the version the port targets) and run
with eager attention on a single
consumer GPU.\footnote{The audit code, its Python dependency
specification,
and the results are published with the Lean formalization
(Appendix~\ref{app:repos}). The results file records the model
revision, the SHA-256 of the fetched \texttt{modeling\_pldrllm.py},
library versions, and the random seed;
the run pins the cuBLAS workspace and forces deterministic CUDA
algorithms, making repeated launches bitwise-reproducible (see Scope);
the raw per-instance arrays behind every summary below ship in a
compressed archive whose SHA-256 the results file records, and the
repository's model-free unit tests recompute the summaries from those
raw arrays. Under \texttt{transformers} 5.x the released port needs
two load-compatibility fixes, documented there; the pinned
environment avoids them.} Inputs: eight fixed English
prompts of $32$--$43$ tokens ($150$--$250$ characters) spanning
technical, narrative, review, and instructional registers.
Aggregation (min/median/max) is over all layers, heads, and prompts
unless stated; spectra are computed in float64.

\subsection{Singular values and numerical rank (in place of
determinants), trained versus initialized}
Over all $5 \times 14 \times 8 = 560$ (layer, head, prompt) instances:
the generator $A$ has numerical rank $1$ in \emph{every} instance
(tolerance $64\,\varepsilon_{\mathrm{f32}}\sigma_1$), with
$\sigma_2/\sigma_1 \le 1.4\times10^{-8}$ (median
$8.7\times10^{-17}$); its rows agree within each head to a
row-variance/entry-variance ratio of $\approx 4\times10^{-13}$, and
across heads to relative RMS $\le 1.6\times10^{-8}$ (median $0$ at
float resolution). This is a direct, spectrum-based confirmation of
the rank-one, cross-head-identical singularity condition
(Proposition~\ref{prop:rankone}, Proposition~\ref{prop:rowfact}(iii))
on this checkpoint. By contrast, $\Alm$ is \emph{not} numerically
low-rank: $\sigma_2/\sigma_1$ has median $0.27$, and the numerical
rank has median $62.5$ of $64$ (range $1$--$64$), even though its
float determinant is $0$: with $\sigma_1$ ranging over
$6.4\times10^{-8}$--$6.2\times10^{3}$ and many small singular values,
the determinant underflows while the matrix is far from rank one. The
measurement thus confirms both halves of Remark~\ref{rem:rankone}:
$A$'s collapse is real and sharp, and the entrywise nonlinear image
$\Alm$ is generically of (near-)full numerical rank, so float-zero
determinants of $\Alm$ carry no rank information.

The trained-versus-initialized control (Proposition~\ref{prop:density})
runs the same battery on the same architecture at random
initialization, under \emph{two} initialization laws, each labeled by
its provenance: the Hugging Face port's own initializer, which draws
$W$, $P$, $a$ and dense weights Xavier-uniform, and the native
training law of the released code, which draws $W$, $P$, $a$
Xavier-normal with dense weights Xavier-uniform (the law the audited
checkpoint was actually trained from; the two laws differ only in the
distribution of the metric-tensor parameters). Under \emph{both}
laws the results agree to the displayed precision: $A$ has numerical
rank $63$ on every instance
with $\sigma_2/\sigma_1$ of median $0.69$, $\Alm$ has numerical rank of
median $64$, and the cross-head relative RMS of $A$ is $1.1$--$1.3$
(heads disagree at order one). The value $63$ itself is architecturally
expected, not a signature: at initialization the row map's final
LayerNorm has $\gamma = \mathbf{1}$, $\beta = \bm0$, so every row of $A$
is exactly centered, $A\mathbf{1} = \bm0$, and numerical rank
$\le \dk - 1 = 63$ is guaranteed; the generic value is then $63$. The
informative contrast with the trained model is the collapse of
$\sigma_2/\sigma_1$ (median $0.69$ at initialization versus
$\le 1.4\times10^{-8}$ trained) and of the cross-head disagreement
(order one versus $\le 1.6\times10^{-8}$). The rank-one collapse and
the cross-head identity of the trained model are therefore training
outcomes, not artifacts of the architecture or of either
initialization law.

\subsection{The LayerNorm scale error, measured directly}
The row variances $v_i$ of $\widetilde D/S$ (all layers/heads, head
rows pooled; $35{,}840$ rows) have median $2.6\times10^{-5}$ and range
$0$ to $0.29$. The discrepancy
$\norm{\LN(\widetilde D)_{i,:} - \LN(\widetilde D/S)_{i,:}}_2$ is
computed \emph{directly} for every captured row, in float64, with the
checkpoint's $\gamma$, $\beta$, $\varepsilon_{\LN}$, and each prompt's
actual $S$; the relative error divides by
$\norm{\LN(\widetilde D/S)_{i,:}}_2$, with the convention $0$ when the
two outputs coincide. Measured: median relative error
$4.6\times10^{-4}$, with a long tail ($90$th percentile
$4.3\times10^{-2}$, $99$th percentile $0.44$, maximum $4.3$ on
$\varepsilon$-dominated rows); median absolute error
$5.4\times10^{-10}$, maximum $1.1\times10^{-5}$. The upper proxy
$\varepsilon_{\LN}/(2v_i)$ of Lemma~\ref{lem:ln}(ii) (whose exact
first-order coefficient is $(1-S^{-2})\varepsilon_{\LN}/(2v_i)$) is
reported alongside \emph{as a proxy}: its median is
$1.9\times10^{-2}$, an
overestimate of the measured relative error by a factor
$\approx 35$ at matched rows, and it diverges on the $14$
exactly-constant rows, where the true discrepancy is exactly $0$ (both
centered rows vanish). Conclusions: the idealization
$\LN(\widetilde D) = \LN(\widetilde D/S)$ is approximately valid for
the typical row ($\sim0.05\%$ error) but not uniformly, since
$16.4\%$ of rows ($5865/35840$) exceed $2\%$ relative error and the worst rows deviate
at order one; any quantitative use of the Stage-2 concentration
argument must carry the measured distribution, not the proxy; and the
divergence of an upper bound is not evidence about the quantity it
bounds.

\subsection{Twirl energy against pre-rotation data}
Analytically ($\dk = 64$, base $10^4$): $C_\Theta = 4.4968\times10^4$,
so $C_\Theta/S \approx 43.9$ at $S = 1024$ (the uniform bound of
Lemma~\ref{lem:twirl} is vacuous there). The per-frequency multipliers
$\abs{\sin(S\omega/2)}/(S\abs{\sin(\omega/2)})$ have median
$0.27/0.054/0.013$ at $S = 64/256/1024$, with worst value
$\approx 1$ at every $S$. Empirically, the audit captures the
\emph{pre-rotation} queries position by position (head $0$ of each
layer, all prompts) and forms both aggregates
$\frac1S\sum_n q_n q_n^\top$ and
$\frac1S\sum_n R_n q_n q_n^\top R_n^\top$, so what the actual
finite-$S$ rotation removed is identified rather than inferred from
the post-rotation aggregate alone (the rotated aggregate reproduces
the captured $\widetilde D/S$ to relative error $\le 1.2\times10^{-7}$,
validating the decomposition pipeline). Two aggregations are
reported, each under its exact label, since ratios do not commute
with pooling. \emph{Per (prompt, layer) instance} (all $8 \times 5 =
40$ instances): in the rotary eigenbasis the commutant
(zero-frequency) component carries a median $6.9\%$
($3.9$--$15.7\%$) of the pre-rotation energy and $10.9\%$
($4.2$--$18.0\%$) of the post-rotation energy; the off-commutant
energy removed by the actual rotation at the prompts' lengths
($S = 32$--$43$) has median $9.5\%$, range $-1.4\%$ to $46\%$: on
one instance the rotated off-commutant energy slightly \emph{exceeds}
the unrotated, so the finite-$S$ rotation does not even monotonically
suppress instance by instance. \emph{Prompt-pooled per layer}
(energy-weighted over prompts; range over the five layers only):
commutant fractions median $7.2\%$ ($4.5$--$12.6\%$) pre-rotation and
$11.1\%$ ($4.5$--$13.8\%$) post-rotation; pooled off-commutant
suppression median $10.6\%$ (range $0.009\%$--$43\%$). The
stationarity idealization of
Proposition~\ref{prop:conc} is also directly checked, per
(prompt, layer), and is far from
holding: first-half and second-half aggregates of $q_n q_n^\top$
differ by a median relative $33\%$ (up to $90\%$). Conclusion: at the
audited context lengths the actual twirl removes only a modest
fraction ($\sim10\%$ median under either aggregation) of the
off-commutant, instance-specific
energy, so Stage~1 cannot by itself account for the observed
invariance at these lengths; attribution to Stage~3 rests on the
direct Stage-3 measurements below, not on subtraction.

\subsection{Commutant residual of the trained operators}
\label{subsec:commutant-audit}
Answering the commutant-diagnostics item of
Section~\ref{sec:discussion} on the released checkpoint: for every
(prompt, layer, head) instance, the relative off-commutant Frobenius norm
residual
$\norm{\Glm - \Pi_{\mathrm{comm}}\Glm}_F/\norm{\Glm}_F$, where
$\Pi_{\mathrm{comm}}$ is the orthogonal projection onto the RoPE
commutant of Proposition~\ref{prop:rope-commutant} (computed exactly in
the rotation eigenbasis; under nonresonance the commutant is the
equal-frequency entries, of real dimension $\dk$). Measured over all
$560$ instances: minimum $0.42$, median $0.81$, maximum $0.99$;
per-layer medians $0.98,\ 0.99,\ 0.93,\ 0.71,\ 0.66$ for layers
$1$--$5$. The squared ratio is the corresponding fraction of squared
Frobenius norm (Frobenius energy). The measurements show a substantial
component in the $\dk^2-\dk$ absolute-position-sensitive directions
counted by
Corollary~\ref{cor:posgap}: on this checkpoint the learned $\Glm$ is
far from the commutant subfamily in every layer and head, most
strongly in the early layers. This is a checkpoint-scoped occupancy
measurement, not a claim that those directions are causally used by
decoding; the ablation that would test causal use (projecting $\Glm$
onto the commutant and measuring the behavioral change) is part of the
empirical program of Section~\ref{sec:discussion}.

\subsection{Row-map contraction, measured on the composition}
The quantity that controls diameters is the Jacobian of the full
composition $\varphi$, and the audit computes it directly: at $128$
rows per layer of real $\LN(\widetilde D)$ data, drawn from
\emph{every} prompt ($16$ per prompt), the largest singular value of
$J\varphi$ has per-layer medians
$3.2\times10^{-11},\ 1.4\times10^{-16},\ 1.0\times10^{-19},\
4.9\times10^{-12},\ 3.5\times10^{-7}$ for layers $1$--$5$ (pooled
range $1.0\times10^{-19}$--$3.6\times10^{-7}$, nearly constant across
rows and prompts within each layer). Empirical pairwise contraction
ratios $\norm{\varphi(r)-\varphi(r')}/\norm{r-r'}$, sampled both
within prompts and across prompts ($1000$ draws attempted each way;
$765$ within-prompt and $694$ cross-prompt pairs retained after
skipping coincident prompt indices and denominators below
$10^{-12}$),
are \emph{exactly zero at float resolution for $\approx95\%$ of
retained pairs}
and at most $4.6\times10^{-6}$ otherwise: sampled visited rows are
mapped to outputs indistinguishable at float32. For continuity with
the per-unit view: individual units are \emph{not} contractive
(per-unit Jacobian norms along the same trajectories have median
$0.47$ and maximum $20.8$), so the strong hypothesis
$L_j \le \kappa < 1$ of Proposition~\ref{prop:contract} is
\textbf{false} per unit on this checkpoint, while the measured
composition is contractive by seven or more orders of magnitude at
every sampled row. These are sampled pointwise statistics on visited
data, not a tube-uniform Lipschitz certificate
(Proposition~\ref{prop:contract}'s hypothesis line); within that
scope, they directly support the locally-constant reading of the
trained row map that Proposition~\ref{prop:rowfact}(iii) makes
diagnostic.

\subsection{The invariance budget as a sample-extrema proxy, against
measured margins}
\label{subsec:budgetproxy}
What follows is a \emph{sample-extrema worst-case proxy} for the chain
of Corollary~\ref{cor:budget}, not a certificate, for five reasons
stated up front: (1) the activation, operator,
and preactivation extrema are maxima/minima over the eight
prompt-only forward passes; the capture hooks are removed before the
decoding experiment, so the extrema do not cover the grown-context
states whose margins the final coefficient is compared against;
(2) the LayerNorm factors use minimum \emph{endpoint} row variances,
which do not lower-bound variances along an interpolation tube;
(3) the measured minimum entry of $\Alm$ is not a lower bound along a
perturbation path through an $\iswiglu$ zero crossing, where only the
architectural floor $10^{-9}$ applies (the constant is therefore
evaluated at both endpoints below); (4) the factors' extrema are
attained at unrelated sample points, and no common perturbation set is
defined; (5) only the derivative suprema are certified, by closed-form
envelopes on the sampled interval ($\sup_{\abs{u}\le U}
\abs{\iswiglu'} \le 2U\varsigma(U) + U^2/4$ and
$\sup_{\abs{u}\le U}\abs{\sigma_s'} \le \varsigma(U) + U/4$, with
$\varsigma$ the logistic sigmoid; these dominate the true suprema,
which a sampled grid maximum, being a lower estimate, does not). A
genuine certificate would require interval or otherwise
verified propagation along every compared decoding state.

With that scope: per-head constants of Corollary~\ref{cor:budget} are
$\norm{W}_2$ median $0.022$ (max $2.3$); $\norm{a}_2$ median
$0.92$; sampled preactivation bound $U$ median $0.018$ (max $78.7$),
giving the certified envelope $\sup_{\abs{u}\le U}\abs{\iswiglu'}$
median $0.018$ (max $1.7\times10^3$); the power-stage factor on the
\emph{sampled} entry range of $\Alm$ has median $1.0\times10^{10}$
(max $5.8\times10^{11}$), because sampled entries of $\Alm$ reach the
$\epsilon = 10^{-9}$ floor and exponents $P_{ij} < 1$ make
$m_A^{P_{ij}-1}$ enormous. The resulting $C_G^{\mathrm{samp}}$
(measured $m_A$; valid for pairs of visited states,
Corollary~\ref{cor:budget}) has median
$7.5\times10^{6}$ and max $1.2\times10^{15}$ per head; the
path-valid $C_G^{\mathrm{tube}}$ ($m_A = 10^{-9}$, the
architectural floor)
has median $9.2\times10^{7}$ while the max is essentially
unchanged ($1.3\times10^{15}$), because the worst heads' sampled
minima already sit at the floor. These $C_G$ values are
Frobenius-to-Frobenius constants; converting a uniform per-row bound
$\varepsilon_A$ on $\Delta A$ into the spectral hypothesis of
Proposition~\ref{prop:perturb} costs the additional factor
$\sqrt{\dk} = 8$ of Corollary~\ref{cor:budget}
($\varepsilon_G = C_G\sqrt{\dk}\,\varepsilon_A$); the assembled chain
below does not use this conversion, entering instead at the directly
hypothesized spectral radius $\varepsilon_{\mathrm{spec}}$. The head-level
factor of Proposition~\ref{prop:perturb} has median $8.5$ (max $74$).
Per decoder layer: the two
LayerNorm factors use the exact Jacobian bound
$\norm{\gamma}_\infty/\sqrt{v_{\min}+\varepsilon}$ of
Lemma~\ref{lem:ln}(iv) with the sampled minimum row variance of each
LayerNorm's inputs (values $8.9$--$23.0$ for the
post-attention LayerNorm and $0.51$--$15.2$ for the post-FFN one); the
gated-FFN factor uses the product rule
$\norm{W_3}(\max\abs{x_2}\cdot\max\abs{\sigma_s'}\cdot\norm{W_1} +
\max\abs{\sigma_s(x_1)}\cdot\norm{W_2})$ with activation extrema
sampled on the prompt passes and the certified $\sigma_s'$ envelope
(values $727$--$3966$); the
attention factor is the coarse assembly from sampled operator norms
(values $4.7\times10^4$--$1.1\times10^5$). The assembly respects the
injection point: a perturbation $\Delta\Glm$ enters at the attention
output of its layer, so its term carries that layer's \emph{own}
post-attention remainder (post-attention LayerNorm, FFN residual,
post-FFN LayerNorm,
$L_{\mathrm{LN1}}(1+L_{\mathrm{FFN}})L_{\mathrm{LN2}}$, with values
$6.6\times10^{3}$, $6.9\times10^{4}$, $1.8\times10^{4}$,
$1.8\times10^{4}$, $1.1\times10^{5}$ for layers $1$--$5$) before
the product of the downstream layers' full factors; no
$(1+L_{\mathrm{attn}})$ enters at the injection layer itself, since
the perturbation arrives at the attention output, not the layer
input. The full per-layer factors are
$3.1\times10^{8}$--$1.2\times10^{10}$,
$\norm{W_{\mathrm{vocab}}}_2 = 175$, and the assembled end-to-end
coefficient multiplying $\norm{\Delta\Glm}_2$ in the logit bound is
$\approx 7.3\times10^{46}$. The audit's test suite checks this
assembly semantically, against finite differences of a miniature
decoder: the coefficient must dominate the measured logit sensitivity
of an actual two-layer post-attention path, and any assembly that
omits the same-layer remainder fails that oracle on a small-variance
input. Two perturbation scales must not be
conflated: the \emph{measured} cached-versus-recomputed operator
deviation is exactly $0$ (bitwise, next paragraph), and the proxy is
therefore stress-tested instead at the \emph{hypothetical
single-rounding radius}
$\varepsilon_{\mathrm{spec}} = 2^{-24}\max\norm{\Glm}_F =
3.4\times10^{-5}$ (measured $\max\norm{\Glm}_F = 563$): what one final
elementwise float32 rounding of an exact-real operator could
contribute under a relative-error model. It is not a measurement, and
it excludes accumulated arithmetic roundoff.

Measured fidelity, greedy decoding of $48$ tokens on $4$ prompts,
cached (frozen prompt $A$/$\Alm$/$\Glm$ plus KV-cache) versus full
uncached recomputation at every step: the recomputed $\Glm$ on the
grown context is compared with the prompt-frozen value at \emph{all
$48$ steps of all $4$ prompts on every layer} ($960$ per-layer
comparisons) and is \emph{bitwise equal at every one}
($\varepsilon_G = 0$ in both relative-RMS and entrywise-maximum
senses; the per-step records ship in the raw archive); the maximum
logit deviation is $3.6\times10^{-5}$ (medians $\approx10^{-5}$),
attributable to floating-point path differences of the two execution
orders rather than to the operator; the minimum realized top-two
margin is $6.0\times10^{-3}$ (per-prompt medians $1.2$--$3.5$); the
greedy token choice agrees at every step of every prompt; and the
sufficient criterion of Corollary~\ref{cor:budget} holds
with room to spare, $2B = 7.2\times10^{-5} < 6.0\times10^{-3} =
\Delta$. The decoded continuations are stored in the results file as
a neutral, human-inspectable record of the audited decoding, with
repetition statistics: greedy decoding of this $110$M-parameter
checkpoint produces repetition loops (duplicate $4$-gram fractions
$0.49$--$0.80$ across the four prompts), while the sampled
continuations of the order-parameter section do not ($0.0$). No claim
about generation quality is based on these texts; the checkpoint's
language capabilities are documented by the benchmark evaluations of
\cite{gokden2026}, and the tensor measurements above are indifferent
to text quality.

\textbf{Verdict on the proxy.} At the hypothetical single-rounding
radius the assembled proxy evaluates to
$7.3\times10^{46} \times 3.4\times10^{-5} \approx 2.4\times10^{42}$,
which exceeds the halved minimum margin
$\Delta/2 = 3.0\times10^{-3}$ of the sufficient criterion by
$\approx 45$ orders of magnitude. Any certificate obtained by
replacing each sampled factor in this \emph{same} factorwise
norm-product assembly by a supremum over a containing tube is no
smaller, so this assembly cannot certify the margin; a different
direct or margin-aware bound on the composite map (exploiting
alignment of singular directions, cancellation along residual
paths, or decision-relevant directions only) could in principle
be smaller, and none is constructed here. This closes the
factorwise norm-product route: the bound-form chain of
Corollary~\ref{cor:budget} does \emph{not} certify
bit-identical decoding on this checkpoint by this route, exactly as
the corollary itself anticipates; what the
measurements support is the empirical statement: on the tested
workloads the cached model is indistinguishable from the uncached
model at the $10^{-5}$ logit level, two orders below the smallest
realized decision margin, with the cached operator itself exactly
invariant at float resolution at every compared step. A certificate
would require margin-aware, non-worst-case propagation (e.g.\
interval or randomized smoothing analysis) along every compared
decoding state, which we leave as future work.

\subsection{DAG-loss values}
The implemented per-instance DAG loss
$\abs{\log(\tr e^{M \Had M}/\dk)}$, computed overflow-safely over the
$560$ instances of each tensor: for $\Alm$, minimum $0$ (a
numerical-resolution reading, not exact acyclicity; see
Remark~\ref{rem:dagobstruction}). The analytic floor of this normalized
logarithmic loss is $\log(1+\epsilon^2) \approx 10^{-18}$, and the
distinct floor of the un-normalized NOTEARS quantity $h$ is
$\dk\epsilon^2 \approx 6.4\times10^{-17}$. Both floors lie below float64 resolution near one. Neither floor identifies
the rounding step responsible for the computed zero. For $\Alm$, the median is
$3.6\times10^{-10}$,
maximum $6.1\times10^{-5}$; for $\Ap$, minimum $63$, median
$1.1\times10^{3}$, maximum $4.1\times10^{4}$; for $\Glm$ (no floor
asserted), minimum
$4.9$, median $1.3\times10^{3}$, maximum $2.8\times10^{4}$. On this
checkpoint all three regularized tensors thus carry strictly positive
measured cycle content except for the numerical-resolution readings of
$\Alm$.\footnote{The exponent parameter $P$ is a per-head tensor of
shape $[h, \dk, \dk]$ with \emph{no} batch axis, unlike the batched
activation tensors returned alongside it; this is an indexing hazard
NumPy does not flag, since broadcast accepts several wrong pairings
silently. The audit builds $\Ap$ through a shape-checked helper with
a sentinel semantic test that fails under any head- or row-broadcast
indexing of $P$, and the chain constants read the parameter directly
under the same guard.}

\subsection{Order parameter, two normalizations}
Two independent stochastic continuations ($64$ tokens, temperature
$1$) of $5$ prompts, whose decoded texts are likewise stored in the
results file (as neutral records with repetition statistics). Two
normalizations are computed and reported separately for every tensor:
the RMS normalization of Definition~\ref{def:orderparam}, and a
\emph{symmetrized signed-mean} normalization with denominator
$\tfrac12(\abs{\mu_1}+\abs{\mu_2})$, a symmetric adaptation of the
source papers' convention rather than the convention itself (there,
the denominator is the first run's $\abs{\mu_1}$ alone, resp.\
$\abs{\mu_C}$ against a cached run \cite{gokden2026}). Measured: for
$\Glm$ the order parameter is exactly $0$
at float resolution on every layer under both normalizations; for
$\Ap$, computed from the per-head exponent tensor (see the DAG
footnote above), the maximum is $8.6\times10^{-13}$ RMS-normalized
($3.3\times10^{-12}$ signed), i.e.\ zero to twelve digits but not
bitwise; for $\Alm$, maximum $3.4\times10^{-9}$ RMS-normalized and
$5.6\times10^{-9}$ signed-normalized (medians $0$; the two
normalizations' maxima differ and are quoted separately); for the most
sensitive tensor
$A$, the RMS-normalized value has median $0$ and maximum
$4.9\times10^{-9}$ (signed-mean maximum $3.1\times10^{-8}$). On this
checkpoint the two normalizations agree in order of magnitude
wherever they are nonzero. Of the piecewise definition's branches,
the zero-numerator branch \emph{is} exercised (every bitwise-equal
pair, in particular every $\Glm$ comparison, returns $0$ through it),
while the degenerate positive-numerator/zero-signed-denominator
case, the case the stabilized definition exists to guard, is not
observed; the stabilized definition matters for models whose deductive
outputs are centered near zero, and is retained for that reason. These
values reproduce, with an independent implementation and the
stabilized statistic, the $m \approx 0$ (float) report of
\cite{gokden2026} for this model.

\subsection{The online contract, historical-row movement, and padding}
\label{subsec:auditonline}

The properties separated in Section~\ref{subsec:online} are measured
by a dedicated online-contract script published alongside the main
audit (same repository; raw arrays and result-file SHA-256 recorded as
for the main audit). The script runs in float32 eager mode on the same
single consumer GPU and under the same determinism configuration as
the main audit (pinned cuBLAS workspace, forced deterministic
algorithms; a CPU fallback is provided); two complete back-to-back
runs reproduce both the results file and the raw archive
byte-for-byte. It loads \emph{two} released checkpoints at pinned
revisions: the audited
\texttt{PLDR-LLM-v51-SOC-110M-5} model above and, as a contrast
within the same
architecture and training family, \texttt{PLDR-LLM-v51-SOC-110M-1}
(the manifest of Appendix~\ref{app:repos} pins both).

\emph{Historical-row movement.} Same-length inputs sharing a prefix
and differing only in the suffix are compared at the shared-prefix
rows (a same-length design, so shape-dependent floating-point
execution cannot masquerade as context sensitivity). A fixed
three-token-prefix pair is complemented by a seeded randomized
protocol ($16$ pairs, length $12$, prefix length $4$; suffixes also
compared with every suffix position attention-masked). On the audited
checkpoint every comparison is \emph{bitwise equal}: zero changed
logit entries in all $2{,}048{,}000$ randomized comparisons and at
every fixed-pair position, with $\Glm$ bitwise equal at every layer
(the only movement anywhere is in the layer-$4$ precursors: $270$
entries of $A$ at magnitude $\le 1.17\times10^{-10}$ and one entry of
$\Alm$ at $2.7\times10^{-12}$).
On the contrast checkpoint the movement is systematic: $74.5\%$ of
the randomized-comparison entries change (maximum
$9.4\times10^{-4}$; masked variant $74.3\%$, maximum
$7.4\times10^{-4}$), the fixed pair moves $31{,}978$ and $31{,}573$
of $32{,}000$ vocabulary entries at its two multi-key prefix
positions, and the per-layer $\Glm$ deviations reach
$3.0\times10^{-2}$: the historical rows move exactly through the
deductive-tensor path of \eqref{eq:histderiv}. Prefix position $0$ is
structurally insensitive on both checkpoints (a single allowed key
makes the softmax row $(1)$ regardless of the score). The pair of
measurements is the checkpoint-level content of
Remark~\ref{rem:prefixmech}(iii): historical-row prefix consistency
fails at generic weights of this architecture and holds bitwise on
the tested pairs for the collapsed checkpoint (a finite measurement,
not the universal decoder property), tying its emergence to the
collapse phenomenon.

\emph{Online determinism.} Repeated identical prefix calls are
bitwise equal on both checkpoints, so the step-$t$ conditional of
\eqref{eq:onlinecontract} is a deterministic function of the prefix
in this configuration.

\emph{Padding at the Gram boundary.} Appending four attention-masked
padding tokens to a prompt and comparing the final \emph{real}-token
row against the unpadded call: on the audited checkpoint the deviation
is $\le 8.6\times10^{-6}$ (exactly zero on one of the two prompts)
with $\Glm$ \emph{bitwise unchanged} and the deviation independent of
the padding content, i.e.\ pure shape-level float sensitivity of
the longer call, not operator contamination; on the contrast
checkpoint $\Glm$ itself moves
(maxima $4.5\times10^{-3}$--$1.2\times10^{-2}$, roughly $50{,}000$
changed entries) and the final-row deviation
($3.5\times10^{-5}$--$7.4\times10^{-5}$) \emph{depends on the padding content},
demonstrating that masked rows enter the Gram. Stripping the padding
restores the unpadded call bitwise on both checkpoints. This is the
measured basis for the unpadded-$S{=}t$ contract of
Section~\ref{subsec:online}.

\subsection{Sequential versus one-pass block scores}
\label{subsec:auditscores}

The same script compares the two candidate-scoring semantics of
Section~\ref{subsec:online} on a fixed suite: for each of the eight
prompts, four multi-token candidates (the model's own greedy
continuations of lengths $4$ and $8$, a second-best-first-token
continuation of length $4$, and a seeded random candidate of length
$4$) and four single-token control candidates. For each candidate the
\emph{sequential} score $\sum_i \log p_\theta(y_i \mid x, y_{1:i-1})$
is computed by one final-row prefix call per token, and the
\emph{block} score by one call on $(x, y)$ minus its last token with
candidate log-probabilities read at the aligned rows (log-softmax in
float64 on the float32 logits; the exact index algebra is unit-tested
against a prefix-consistent toy decoder, for which the two protocols
must coincide identically, and a global toy decoder, for which they
must not).

Single-token candidates coincide \emph{exactly} on both checkpoints,
as they must (the two protocols invoke the model on the identical
tensor). For multi-token candidates, on the audited checkpoint the
per-token gap is at most $3.3\times10^{-6}$ with median exactly $0$
(most per-token gaps vanish), and the per-candidate total-score gap is
at most $3.3\times10^{-6}$; on the contrast checkpoint the gaps are
roughly fifty-fold larger (per-token maximum $1.5\times10^{-4}$,
per-candidate maximum $2.2\times10^{-4}$), a genuinely semantic
difference, consistent with the historical-row movement above. On
\emph{neither} checkpoint does any of the $32$ candidate rankings
change: zero argmax flips and zero discordant pairs out of $48$
ordered comparisons. On the audited checkpoint the residual gap is at
the shape-level float-sensitivity scale (its historical rows are
bitwise invariant in same-length comparisons), so one-pass block
scores and sequential scores are interchangeable there at far below
score-decision scales; published benchmark scores remain labeled as
block scores (Appendix~\ref{app:repos}), and the contrast checkpoint
quantifies the regime where the two semantics genuinely differ. This
constructed suite is a controlled \emph{diagnostic}, not a validation
of benchmark evaluations; the corresponding measurement on held-out
text and on real benchmark items is
Section~\ref{subsec:auditseq}.

\subsection{Held-out sequential NLL and real benchmark items under
both scoring protocols}
\label{subsec:auditseq}

The same comparison is run on held-out text and on \emph{real}
benchmark items (script \texttt{audit\_seq.py}; both released
checkpoints; the same GPU float32 eager configuration, cuBLAS
workspace and deterministic-algorithm pins, and two byte-identical
back-to-back official runs as the preceding subsections). Every
dataset enters as a parquet file pinned by full revision hash
(Appendix~\ref{app:repos}), with per-file SHA-256 recorded in the
results JSON; the benchmark request templates and the published
TruthfulQA metric are transcriptions of the pinned evaluation
wrapper's task configurations, unit tested against hand-built
documents and hand-computed metric values, and the wrapper's
encode-pair convention (trailing-whitespace shift, joint-encoding
split) is reproduced: on all $7{,}172$ scored requests the
independently encoded context was verifiably a token prefix of the
joint encoding.
Every summary below is re-derived offline from the raw arrays by the
model-free test suite.

\emph{Held-out blockwise CE versus sequential NLL.} On $48$ seeded
disjoint $257$-token windows of the pinned WikiText-2 (raw)
validation split \cite{merity2016} ($12{,}288$ scored positions
per checkpoint), each position is scored twice: by its row of one
full-block call (the training objective's per-row quantity,
\eqref{eq:blockloss}) and by one final-row $S = t$ call
(\eqref{eq:onlinecontract}, the deployed chain-rule quantity). On
the audited checkpoint the two aggregates coincide to nine decimal
digits: blockwise CE $3.479446$ nats/token against sequential NLL
$3.479446$ (difference $1.8\times10^{-9}$), with per-token
$\abs{\text{gap}}$ median $9.6\times10^{-7}$, maximum
$2.4\times10^{-5}$, and signed median exactly $0$. On the contrast
checkpoint the aggregates still agree to $4.3\times10^{-6}$
nats/token ($3.684883$ versus $3.684879$), but the per-token gaps
are two orders larger (median $1.8\times10^{-5}$, maximum
$2.8\times10^{-3}$) and nearly sign-balanced ($49.3\%$ of positions
score worse sequentially); their magnitude \emph{decreases} with
position within the window (per-position-bucket medians
$2.5\times10^{-5}$--$3.8\times10^{-5}$ early, $5.0\times10^{-6}$ in the final
bucket), which is the suffix-length dependence the first-order
mechanism \eqref{eq:histderiv} predicts: a later historical row
leaves less suffix to move the Gram. On this sample, then, the
blockwise objective's value \emph{is} the autoregressive NLL to
float resolution on the collapsed checkpoint, and a close,
sign-balanced surrogate for it on the noncollapsed one; neither fact
extends beyond the tested sample, and neither removes the
target-exposure structure of \S\ref{subsec:dag}, whose
training-time role is a separate, unmeasured question
(Section~\ref{sec:discussion}).

\emph{Real benchmark items.} Seeded samples of $100$ actual items
from each of the eight published zero-shot tasks (ARC-Easy and
ARC-Challenge test splits, HellaSwag, PIQA, Social-IQa, WinoGrande,
and TruthfulQA validation splits, OpenBookQA test split; $800$
items, $3{,}052$ candidates, $32{,}246$ scored candidate tokens per
checkpoint) are scored under the one-pass block protocol and under
the sequential chain-rule protocol. TruthfulQA enters under its
published \texttt{truthfulqa\_mc2} protocol \cite{lin2021}: each
question carries several true and several false reference answers,
and the published score is not an argmax but the normalized
probability mass on the true answers,
$\sum_{i \in \mathrm{true}} e^{\ell_i} \big/ \sum_{j} e^{\ell_j}$
over the per-candidate total log-likelihoods $\ell$, transcribed
from the pinned harness task configuration, computed by an
overflow-safe softmax equivalent, and evaluated once from block and
once from sequential candidate scores. On the seven argmax tasks
the outcome-level result is uniform: \emph{zero} raw argmax
changes, \emph{zero} length-normalized argmax changes, and
\emph{zero} discordant candidate pairs (of $2{,}901$) on
\emph{both} checkpoints, so every per-task accuracy, raw and
normalized, is identical under the two protocols on this sample;
single-token candidates coincide exactly, as they must. On
TruthfulQA the published metric agrees between the two protocols to
$8.5\times10^{-10}$ in the $100$-item mean on the audited
checkpoint ($0.397662$ under both; per-item metric gap median $0$,
maximum $3.7\times10^{-7}$) and to $5.7\times10^{-7}$ on the
contrast checkpoint (block $0.399748$, sequential $0.399747$;
per-item maximum $4.7\times10^{-5}$); a candidate-level argmax
(a diagnostic, not the published metric) does not change on
either checkpoint. The score-level gaps mirror the constructed
suite: per-candidate absolute gaps reach $2.8\times10^{-5}$ on the
audited checkpoint (median exactly $0$ on five of the eight tasks)
and $2.3\times10^{-3}$ on the contrast checkpoint, growing
monotonically with candidate token length there (bucket maxima
$3.1\times10^{-4}$ at $2$--$4$ tokens up to $2.3\times10^{-3}$
beyond $20$ tokens; single-token bucket exactly $0$). The decision
margins explain the argmax stability: per-task median block top-two
margins are $0.51$--$12.9$ nats, and on every item except one the
item's own maximal candidate gap stays below half of that item's
realized margin. The single flagged item (the same on both
checkpoints) is a Social-IQa instance whose top two \emph{candidate
strings are identical in the source data}, an exact tie of margin
$0$ that both protocols score bitwise identically: a dataset
artifact, not a protocol discrepancy. An auxiliary single-gold
TruthfulQA probe (the mc1 variant of the same dataset, on the same
seeded rows; $534$ candidates per checkpoint) supplies the argmax
structure the published metric lacks and is reported separately in
the audit artifact, outside the published-task count: zero argmax
changes and zero discordant pairs (of $1{,}366$) on both
checkpoints.

These are sampled measurements: $100$ items per task under one seed,
two checkpoints, one software stack. They support the statement that
on these samples the published one-pass block protocol and
sequential chain-rule scoring select the same answers on the argmax
tasks and agree on the published TruthfulQA probability-mass metric
to within $5\times10^{-5}$ per item (even on the noncollapsed
contrast checkpoint, whose score \emph{values} differ measurably
between protocols), and they do not revalidate the complete
published benchmark tables, other checkpoints, or the training
trajectory (Section~\ref{sec:discussion}, empirical program).

\subsection{Scope}
\looseness=-1
This audit is one checkpoint, eight prompts, and one software stack
(the online-contract and sequential-validation sections additionally
probe a second released checkpoint, seeded held-out text windows, and
seeded fixed-size samples of real benchmark items, still far from
the full validation sets, checkpoint fleets, and training
trajectories of the source papers, which remain listed in
Section~\ref{sec:discussion}'s program);
its Jacobian and pairwise statistics are sampled at visited points; it
establishes the measured facts above for this model and does not by
itself support generalization across seeds, scales, or
datasets. Individual measured
values carry floating-point path sensitivity: they
depend on the device and library versions, and quantities downstream
of the elementwise power stage (whose local amplification reaches
$\sim10^{10}$) can vary by more than low-order digits across
configurations. This sensitivity extends to repeated launches in a
\emph{fixed} environment: CUDA matrix-multiply algorithm selection can
differ between process launches, and preparing this audit we
observed the smallest order-parameter statistics shift by several
orders of magnitude (within $\lesssim 10^{-3}$ relative) across
launches of the identical seeded script. The shipped run therefore
pins the cuBLAS workspace and forces deterministic algorithms, after
which repeated launches reproduce every reported number bitwise
(verified by back-to-back complete reruns); order-parameter readings
below that stability scale should in general be quoted as upper
bounds. The qualitative findings above are stable under all of this.
The audit
code, its Python dependency specification, the results files, and
the raw per-instance arrays are published with the Lean
formalization (Appendix~\ref{app:repos}).

\vfill\newpage
\bibliographystyle{amsplain}
\bibliography{power_law_graph_attention}

\end{document}